\pdfoutput=1
\documentclass[10pt,journal,compsoc]{IEEEtran}
\usepackage[T1]{fontenc}
\usepackage[utf8]{inputenc}
\usepackage{amsmath,amssymb,booktabs,array,cite,url,balance}
\usepackage{tikz}
\usetikzlibrary{arrows.meta,calc}
\newcommand{\tr}{\operatorname{tr}}
\newtheorem{proposition}{Proposition}
\newif\ifshowbios
\showbiosfalse
\begin{document}
\title{WiMoT: Replay-Free Continual Wi-Fi Sensing via Moment Transport}
\author{Chunchun~Yang,
Yang~Cao,
and~Zhitang~Chen%
\IEEEcompsocitemizethanks{\IEEEcompsocthanksitem C. Yang and Y. Cao are with the School of Information Science and Technology, University of Science and Technology of China, Hefei 230026, China.\protect\\
E-mail: ycc1992@mail.ustc.edu.cn, forrest@ustc.edu.cn
\IEEEcompsocthanksitem Z. Chen is an Independent Researcher.\protect\\
E-mail: ztangchen@gmail.com
\IEEEcompsocthanksitem Corresponding author: Yang Cao.}}
\markboth{}%
{Yang \MakeLowercase{\textit{et al.}}: WiMoT: Replay-Free Continual Wi-Fi Sensing via Moment Transport}
\IEEEtitleabstractindextext{%
\begin{abstract}
Continual Wi-Fi sensing must learn new rooms or users (plasticity) while keeping earlier ones (stability), yet sensing nodes often cannot retain channel state information (CSI) for privacy and storage reasons. We propose WiMoT (Wi-Fi Moment Transport), which stores the per-class feature mean and covariance of every earlier domain instead of samples. Stored statistics become stale when the representation changes, so WiMoT restricts the change to an affine map on a frozen feature basis and applies the same map to the statistics, which then stay valid without old samples. Supervised contrastive learning learns the new domain, a closed-form loss on the stored statistics keeps earlier classes separable, and a Gaussian discriminant classifier refitted from the transported statistics predicts without domain identity. On 3DO, WiMANS, and Widar3.0, WiMoT reaches average final accuracies of 94.02\%, 89.61\%, and 95.16\% with at most 0.58 percentage points (pp) of forgetting. On 3DO and Widar3.0 it is 3.77 and 13.11 pp ahead of the strongest of 21 baselines, which uses replay; on WiMANS it matches the strongest replay baseline in accuracy while reducing forgetting from 2.78 to 0.00 pp. Rank-32 covariance factors shrink the moment memory to 6.4\% at a cost of at most 0.49 pp.
\end{abstract}
}
\maketitle
\IEEEdisplaynontitleabstractindextext
\IEEEpeerreviewmaketitle

\IEEEraisesectionheading{\section{Introduction}\label{sec:intro}}
\IEEEPARstart{W}{i-Fi} sensing recognizes human activity from channel state information (CSI) without wearable devices, which makes it attractive for smart-home and other ubiquitous computing services. The CSI of the same activity, however, changes with the room, the measurement day, and the user\cite{widar2019,mobileda2020}. A deployed sensing node therefore has to absorb each new acquisition condition while keeping its recognition of earlier ones, a domain-incremental setting with a fixed activity vocabulary\cite{three2022}. Existing continual Wi-Fi sensing methods retain CSI exemplars\cite{csiaril2023,hewifi2025,edgecl2026,samoec2026} or add task-specific parameters\cite{consense2025}, but retaining recordings of people in their homes is often restricted by privacy rules, and edge nodes have little storage.

A continual learner has to learn each new domain (plasticity) and still recognize the earlier ones (stability). Without stored samples, each of three common designs falls short of this balance. A frozen representation with stored class statistics\cite{fetril2023,fecam2023,ranpac2023} is stable but cannot learn the new domain. An unconstrained update is plastic but forgets: on the nine-user stream evaluated here, fine-tuning forgets 28.78 pp, parameter regularization (SI)\cite{si2017} 20.17 pp, and parameter expansion (ConSense)\cite{consense2025} 27.50 pp. Updating the representation while keeping class means and covariances aims at both, but the stored statistics describe the old representation, and estimating their drift from current data\cite{sdc2020,ldc2024,adagauss2024,dpcr2025} is only approximate for a nonlinear update.

We build WiMoT (Wi-Fi Moment Transport) from this requirement, one step at a time. Stability comes first. Our earlier Optimizing SPCA-based Continual Learning (OSCL) classifier\cite{oscl2023} builds its decision rule only from class means and covariances and optimizes the combining labels of each task; under high-dimensional asymptotics, per-task optimal labels ensure that the optimal error of an earlier task does not increase when new tasks arrive. OSCL is linear, however, and needs a fixed feature space in which the classes are linearly separable, which raw CSI does not provide. WiMoT therefore places OSCL on a frozen time-series foundation model, MantisV2\cite{mantis2025,mantisv2}, and trains a supervised contrastive (SupCon)\cite{supcon2020} projection on the first domain to make the activities linearly separable. A representation frozen in this way inherits the stability of OSCL but learns later domains poorly. To restore plasticity, WiMoT adapts an affine map on the frozen basis to each new domain with SupCon. Adaptation in turn creates two risks: the stored statistics now describe the old representation, and the update can discard directions that separate earlier classes. WiMoT removes both with the same stored moments. Since the mean and covariance of affinely transformed features follow from those of the original features alone, applying the adapter to the stored moments keeps them exact in the new representation without any old sample (moment transport, Proposition~\ref{prop:exact}(i)), and a historical loss computed in closed form from these moments penalizes updates that make earlier labels less linearly readable (Proposition~\ref{prop:exact}(ii)). After every update, OSCL is refitted from the transported moments of all observed domains with equal weight, so it needs no domain identity and does not favor the newest domain. Each step resolves the weakness exposed by the previous one, and the result is a replay-free learner whose memory, representation, and classifier remain consistent with one another. The main contributions of this paper are as follows.
\begin{enumerate}
\item \emph{A memory that does not expire.} Restricting later updates to an affine adapter on a frozen basis lets the model's own map transport the stored moments exactly (Proposition~\ref{prop:exact}), so the memory stays valid without drift estimation. Transport is what makes adaptation useful: without it, AFA drops by 12.1, 3.4, and 31.0 pp (Section~\ref{sec:rq2}), whereas with it the adapter raises AFA by 1.07, 3.33, and 5.93 pp over a frozen representation (Section~\ref{sec:rq2}).
\item \emph{Plasticity and stability from one set of moments.} SupCon supplies plasticity and a historical loss, evaluated in closed form from the stored moments, supplies stability; removing SupCon lowers AFA by 2.54--5.12 pp and removing the historical loss raises forgetting on all three streams (Section~\ref{sec:rq2}). Together they keep forgetting at or below 0.58 pp and give the highest accuracy on earlier domains on all three streams; the balance of new- and earlier-domain accuracy is the highest on 3DO and Widar3.0, 3.03 and 8.51 pp above the best baseline, and within seed variation of the best on WiMANS (Section~\ref{sec:rq2}).
\item \emph{A domain-agnostic OSCL classifier that follows the representation.} We extend OSCL to domain-incremental prediction without domain identity by marginalizing domain-conditioned OSCL posteriors, refitted after every update from the transported moments of all domains. On the same adapters and moments it exceeds the strongest alternative classifier (LDA or QDA) by 1.93, 0.74, and 3.84 pp (Section~\ref{sec:rq2}).
\end{enumerate}

\section{Related Work}
\subsection{Continual Wi-Fi Sensing}
CSI-ARIL retains representative samples with a distillation loss\cite{csiaril2023}, and He \emph{et al.} combine an EWC loss with a memory set\cite{hewifi2025}. ConSense adds per-session prefixes to a frozen attention layer and retrains only neurons whose activations change\cite{consense2025}. EdgeCL replays a core-set selected by clustering and herding, with temperature-softened labels and a sharpness-aware update\cite{edgecl2026}. These methods protect earlier knowledge with stored samples or added parameters; WiMoT stores class moments instead.

\subsection{Continual-Learning Mechanisms}
Regularization methods penalize changes of important parameters, estimated from the Fisher information (EWC)\cite{ewc2017}, the optimization path (SI)\cite{si2017}, or output sensitivity (MAS)\cite{mas2018}; LwF distills the previous model\cite{lwf2018}. Replay methods train on retained observations (ER)\cite{er2019}, preserve their logits (DER++)\cite{der2020}, or constrain gradients with them (A-GEM)\cite{agem2019}. Recent domain-incremental methods use expert mixtures (CN-DPM)\cite{cndpm2020}, task-vector and classifier consolidation (DUCT)\cite{duct2025}, or concept prototypes (DualCP)\cite{dualcp2025}. The historical loss of WiMoT is the squared classification loss of earlier domains, evaluated in closed form from moments through the least-squares identity of analytic continual learning, which obtains a closed-form classifier from correlation matrices of a frozen representation\cite{acil2022,dsal2024}; here it constrains a gradient-trained adapter.

\subsection{Drift Compensation of Stored Statistics}\label{sec:drift}
Exemplar-free class-incremental methods that update the backbone must correct stored class statistics for feature drift. SDC and ADC shift stored means by displacements measured on current or perturbed samples\cite{sdc2020,adc2024}; LDC, AdaGauss, and GATF learn an old-to-new feature map and pass the stored means or Gaussians through it\cite{ldc2024,adagauss2024,gatf2026}; DPCR calibrates stored covariances with a ridge-estimated projection\cite{dpcr2025}; and CASTER applies the affine moment rule with a class-shared map estimated from unlabeled target moments\cite{caster2026}. In all of them the map is estimated from data, so the transported statistics are approximate. WiMoT instead constrains the representation change to an affine adapter on a frozen basis, so that the same map transports the stored moments by an identity (Proposition~\ref{prop:exact}).

\subsection{Statistical Continual Classification}
OSCL optimizes combining labels for spectral classification across tasks from class moments\cite{oscl2023}, with covariance whitening\cite{whitening2018} and shrinkage\cite{ledoit2004} for stability. It needs class moments valid in the current representation, which transport provides.

\section{Problem Setting}\label{sec:retention}
Let $\mathcal D_t=\{(x_i,y_i)\}_{i=1}^{N_t}$ be the training arrivals of domain $t$, with $y_i\in\{1,\ldots,C\}$ and the same label set in every domain. During update $t$, WiMoT accesses $\mathcal D_t$, the current model state, and stored aggregate statistics; it keeps no historical individual and cannot revisit earlier recordings. Domain boundaries and current labels are known during training. At inference, the predictor receives only $x$; the true domain identity is used solely by the evaluator. The replay baselines of Section~\ref{sec:experiments} keep a common class-balanced cache of fixed capacity (240, 36, and 42 individuals on 3DO, WiMANS, and Widar3.0), described in Supplementary Section~S2.

\section{Proposed Method}
\subsection{Design Rationale}\label{sec:rationale}
Continual learning without stored samples has to learn each new domain (plasticity) and keep recognizing the earlier ones (stability). WiMoT keeps one compact memory in place of the samples, the class means and covariances of every domain in a frozen feature basis. This memory raises three problems, and each has one mechanism. (i) \emph{Stored statistics go stale.} When the representation changes, they describe the old one. WiMoT restricts the change after the first domain to an affine map and applies the same map to the memory, so that it stays valid without drift estimation (Section~\ref{sec:transport}); without transport, the AFA falls by 12.1, 3.4, and 31.0 pp relative to transport (Section~\ref{sec:rq2}). (ii) \emph{The representation must learn and also remember.} An update can fit the new domain and discard directions that separate earlier classes, and no classifier can recover information that the representation no longer carries. SupCon supplies plasticity, and a historical loss computed from the memory keeps earlier classes separable (Section~\ref{sec:balance}). (iii) \emph{The classifier must follow the representation and treat all domains evenly.} A classifier fitted on current samples favors the newest domain, and a stored decision rule lags behind an updated representation. The OSCL classifier is refitted after every update from the transported memory of all domains with equal weight, without a domain identity (Section~\ref{sec:inference}).

\subsection{Feature Basis and Model}
\begin{figure*}[t]
\centering
\begin{tikzpicture}[x=1mm,y=1mm,>=Latex,font=\fontsize{8.5}{10}\selectfont,
 lab/.style={align=center,inner sep=1pt},
 small/.style={lab,font=\fontsize{7.5}{9}\selectfont},
 title/.style={anchor=west,font=\fontsize{9}{11}\selectfont\bfseries},
 arr/.style={->,draw=black!70,line width=.75pt,shorten >=.4mm,shorten <=.4mm},
 main/.style={->,draw=black!85,line width=1.2pt,shorten >=.4mm,shorten <=.4mm},
 tarr/.style={->,draw=fborange,line width=.75pt,shorten >=.4mm,shorten <=.4mm},
 sarr/.style={->,draw=fbgreen,line width=.85pt,shorten >=.4mm,shorten <=.4mm},
 frozen/.style={lab,draw=fbblue,fill=white,rounded corners=1.5pt,line width=.65pt},
 train/.style={lab,draw=fborange,fill=fborange!5,rounded corners=1.5pt,line width=.7pt},
 hero/.style={lab,draw=fborange,fill=fborange!9,rounded corners=2pt,line width=1.15pt},
 stat/.style={lab,draw=fbgreen,fill=fbgreen!5,rounded corners=1.5pt,line width=.7pt},
 stathero/.style={lab,draw=fbgreen,fill=fbgreen!10,rounded corners=2pt,line width=1.15pt},
 port/.style={lab,draw=black!45,fill=white,rounded corners=2.5mm,minimum height=5mm,minimum width=6mm,line width=.55pt},
 dim/.style={small,text=black!60,fill=white,inner sep=.6pt},
 note/.style={small,text=black!60,font=\fontsize{7.5}{9}\selectfont\itshape}]
\definecolor{fbblue}{HTML}{0072B2}
\definecolor{fborange}{HTML}{B85B16}
\definecolor{fbgreen}{HTML}{007F73}
\path[use as bounding box] (0,0) rectangle (178,100);
\fill[fbblue!4,rounded corners=2pt] (0,62) rectangle (99.5,100);
\fill[fbgreen!4,rounded corners=2pt] (78,7) rectangle (178,31);
\node[title,text=fbblue] at (2,96.6) {(a) Frozen feature basis};
\node[small,anchor=east,text=fbblue] at (98,96.6) {fixed after domain 1};
\node[title,text=black!80] at (103,96.6) {(b) Adaptation and prediction};
\node[title,text=fbgreen] at (80,10.4) {(c) Moment memory};
\foreach \i in {0,...,4}{\foreach \j in {0,...,4}{
 \pgfmathtruncatemacro{\shade}{15+mod(13*\i+19*\j,40)}
 \fill[fbblue!\shade] (2+1.9*\i,76.3+1.9*\j) rectangle ++(1.7,1.7);
}}
\draw[black!45,line width=.5pt] (1.6,75.9) rectangle (11.6,86.1);
\node[small] at (6.6,89) {CSI $x$};
\node[frozen,small,minimum width=20mm,minimum height=14mm] (enc) at (28,81) {Pretrained TS\\encoder\\(frozen)};
\draw[main] (11.6,81)--(enc.west);
\node[frozen,small,minimum width=32mm,minimum height=7.5mm] (sem) at (65,89.5) {Semantic stem\\$9216{\to}512{\to}128$};
\node[frozen,small,minimum width=32mm,minimum height=7.5mm] (pca) at (65,81) {First-domain PCA\\$9216{\to}256$};
\node[frozen,small,minimum width=32mm,minimum height=7.5mm] (amp) at (65,72.5) {Random amplitude stem\\$512{\to}1024{\to}128$};
\coordinate (mj) at (43,81);
\draw[main,shorten >=0mm] (enc.east)--(mj);
\fill[black!85] (mj) circle(.65);
\draw[main] (mj)|-(sem.west);
\draw[main] (mj)--(pca.west);
\draw[main] (6.6,75.9)--(6.6,66)-|(45,72.5)--(amp.west);
\node[dim,fill=fbblue!4,text=fbblue] at (26,66) {amplitude mean + SD $\to512$};
\node[lab,draw=fbblue,fill=white,rounded corners=2.5mm,line width=.85pt,minimum width=7mm,minimum height=25mm] (phi) at (92,81) {$\phi$\\[6pt]{\fontsize{6.5}{7}\selectfont$\mathbb R^{512}$}};
\foreach \n in {sem,pca,amp}{\draw[main] (\n.east)--(\n.east-|phi.west);}
\node[hero,minimum width=21mm,minimum height=14mm] (adapter) at (113.5,81)
 {\textbf{Affine adapter}\\$h=A_t\phi+b_t$\\$512\to128$};
\draw[main] (phi.east|-adapter.west)--(adapter.west);
\node[port,draw=fborange,line width=.85pt] (h) at (134,81) {$h$};
\draw[main] (adapter.east)--(h.west);
\node[stathero,minimum width=24mm,minimum height=14mm] (oscl) at (155.5,81)
 {\textbf{OSCL classifier}\\whitening + spectral\\classifier $\sum_s g_s\,p_s$};
\draw[main] (h.east)--(oscl.west);
\node[lab,font=\fontsize{10}{10}\selectfont] (yhat) at (174.5,81) {$\hat y$};
\draw[main] (oscl.east)--(yhat.west);
\node[note,text=black!70,anchor=east] at (177,92) {$\hat y=\arg\max_c p_{\mathrm{OSCL}}(c\mid x)$; no domain identity};
\draw[fborange!70,dashed,rounded corners=2pt,line width=.6pt] (128,37) rectangle (175,65);
\node[small,text=fborange,anchor=north east] at (174.5,64.6) {training only};
\node[train,small,minimum width=14mm,minimum height=7.5mm] (proj) at (148,54) {Projector\\$\to\ell_2$-norm};
\node[train,small,minimum width=12mm,minimum height=7.5mm] (sc) at (166,54) {SupCon\\$\mathcal L_{\mathrm{SC}}$};
\node[train,small,minimum width=14mm,minimum height=7.5mm] (wmap) at (148,43.5) {Map $W$\\$128{\to}C$};
\node[train,small,minimum width=12mm,minimum height=7.5mm,line width=1pt] (lh) at (166,43.5) {$\mathcal L_{\mathrm{hist}}$};
\coordinate (hj) at (134,54);
\draw[tarr,shorten >=0mm] (h.south)--(hj);
\fill[fborange] (hj) circle(.6);
\draw[tarr] (hj)--(proj.west);
\draw[tarr] (hj)|-(wmap.west);
\draw[tarr] (proj.east)--(sc.west);
\draw[tarr] (wmap.east)--(lh.west);
\draw[->,draw=fborange,dashed,line width=.75pt,shorten >=.4mm] (128,59.5)-|(adapter.south);
\node[small,text=fborange,anchor=south] at (121,59.7) {$\nabla$};
\node[small,text=black!75] at (112,48.5) {gradients of\\$0.5\,\mathcal L_{\mathrm{SC}}+\mathcal L_{\mathrm{hist}}$\\update adapter,\\projector and $W$\\(domain 1: $\mathcal L_{\mathrm{SC}}$ only)};
\node[stat,minimum width=26mm,minimum height=15mm] (mom) at (95,21) {};
\node[lab,font=\fontsize{8}{9.5}\selectfont\bfseries] at (95,25.5) {Moment memory};
\foreach \i/\sym in {0/n,1/\mu,2/\Sigma}{
 \draw[fbgreen!70,fill=white,line width=.45pt] (85.4+6.8*\i,18) rectangle ++(5.9,4.2);
 \node[small] at (88.35+6.8*\i,20.1) {$\sym$};
}
\node[small] at (95,15.6) {per domain and class};
\draw[sarr] (phi.south)--(phi.south|-mom.north);
\node[small,anchor=east,text=fbgreen] at (91,58.5) {store moments of $(\phi,y)$,\\then discard samples};
\node[stat,minimum width=26mm,minimum height=15mm] (transport) at (128,21)
 {\textbf{Transport}\\$\mu^h=A_t\mu^\phi+b_t$\\$\Sigma^h=A_t\Sigma^\phi A_t^{\top}$};
\draw[sarr] (mom.east)--(transport.west);
\draw[sarr] (104,28.5)--(104,33)-|(lh.south);
\node[dim,fill=white,text=fbgreen] at (150,33) {$B_s,K_s$ to $\mathcal L_{\mathrm{hist}}$};
\draw[sarr] (transport.east)-|(176.5,69)-|(oscl.south);
\node[dim,fill=fbgreen!4,text=fbgreen,anchor=south] at (158,21.3) {class moments in $h$};
\node[small,anchor=west,font=\fontsize{8}{9.5}\selectfont\bfseries] at (4,52) {Legend};
\draw[fbblue,fill=white,rounded corners=1pt,line width=.65pt] (4,44.5) rectangle ++(5,3);
\node[small,anchor=west] at (10.5,46) {frozen after domain 1};
\draw[fborange,fill=fborange!7,rounded corners=1pt,line width=.75pt] (4,38.5) rectangle ++(5,3);
\node[small,anchor=west] at (10.5,40) {trained by gradient};
\draw[fbgreen,fill=fbgreen!7,rounded corners=1pt,line width=.75pt] (4,32.5) rectangle ++(5,3);
\node[small,anchor=west] at (10.5,34) {aggregate statistics};
\draw[main] (42,46)--(50,46);
\node[small,anchor=west] at (51,46) {inference path};
\draw[tarr] (42,40)--(50,40);
\node[small,anchor=west] at (51,40) {training-only path};
\draw[sarr] (42,34)--(50,34);
\node[small,anchor=west] at (51,34) {moment path};
\node[note,anchor=west,align=left] at (4,21) {No historical sample is stored or replayed;\\$\mathcal L_{\mathrm{hist}}$ and OSCL read only\\aggregate moments.};
\end{tikzpicture}

\caption{WiMoT framework. (a) A frozen encoder, the semantic stem, the first-domain PCA shortcut, and a fixed random amplitude stem form the basis $\phi$. (b) The affine adapter maps $\phi$ to $h$, from which the OSCL classifier predicts $\hat y$ (bold path); the projector with SupCon and the linear map $W$ with $\mathcal L_{\mathrm{hist}}$ are used only in training. (c) Class moments of each domain are stored before its samples are deleted; they supply $B_s,K_s$ to $\mathcal L_{\mathrm{hist}}$ and, transported through the current adapter, the moments used by OSCL.}
\label{fig:framework}
\end{figure*}
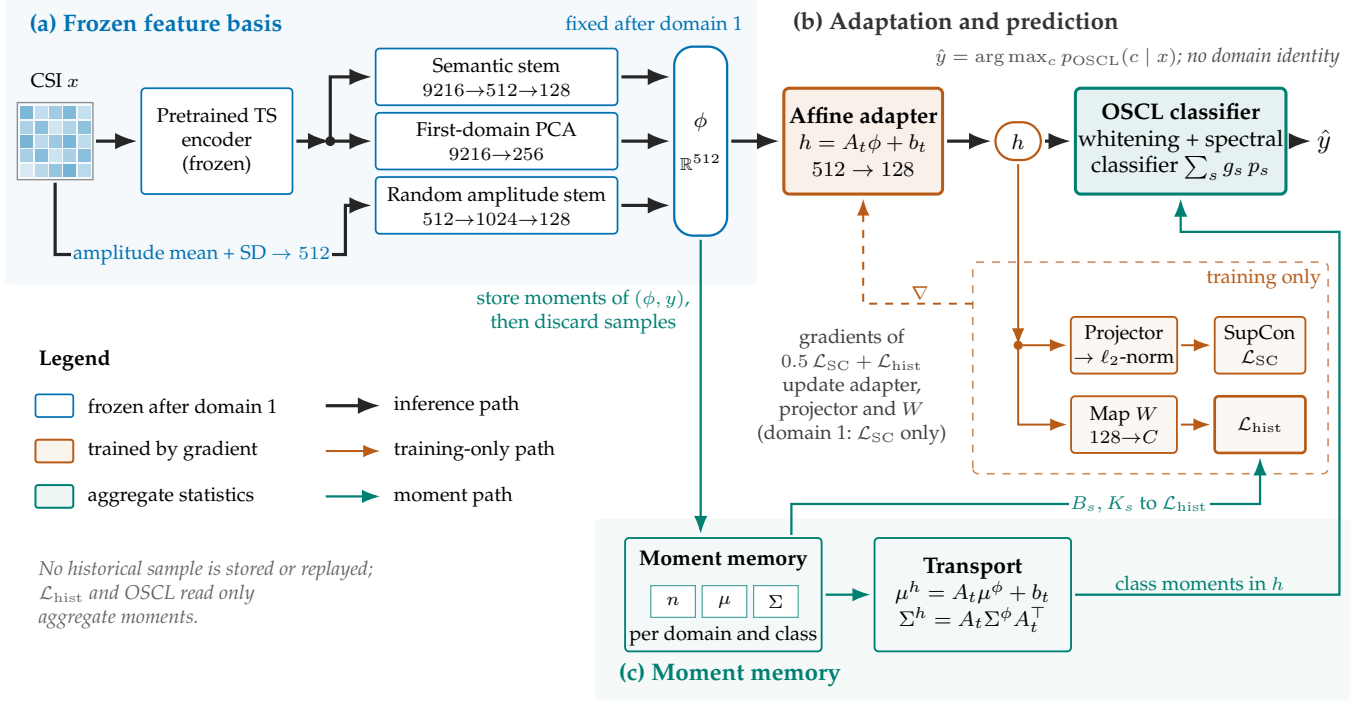

Fig.~\ref{fig:framework} shows the model. All datasets use the same frozen MantisV2 checkpoint\cite{mantis2025,mantisv2} with channelwise temporal processing. Each CSI channel is interpolated to 512 samples, and the classification token and token mean of the third of six pretrained attention layers give 512 coordinates per channel. Nine-group aggregation of channel embedding means and standard deviations yields 9216 semantic coordinates. Temporal channel means and standard deviations, each interpolated along the channel axis to 256 coordinates, form a 512-dimensional amplitude descriptor, so the common input has 9728 dimensions. Normalization follows Supplementary Section~S2. The evaluated tokenizer and first three attention layers hold 2.15 of the 4.19 million parameters of the checkpoint.

A semantic stem with widths $9216\to512\to128$, batch normalization, and ReLU, and a fixed random amplitude stem with widths $512\to1024\to128$ feed an affine $256\to128$ adapter and an affine $128\to128$ contrastive projector. Both stems are frozen after the first domain. A 256-dimensional PCA shortcut fitted to first-domain semantic features is then appended to the stem outputs, forming the frozen basis $\phi(x)\in\mathbb R^{512}$, and the adapter is expanded to $512\to128$ with zero-initialized new columns, which preserves the first-domain function. Subsequently,
\begin{equation}
 h_t=A_t\phi(x)+b_t.
\end{equation}
The statistical classifier operates on $h_t$. The contrastive projector, with weight $P_t$ and bias $p_t$, maps $h_t$ to the unit vector $u_t=(P_th_t+p_t)/\|P_th_t+p_t\|_2$, on which SupCon compares features by cosine similarity (training only). An affine map $h\mapsto W_th+d_t$ to $\mathbb R^{C}$ appears only in the historical loss of Section~\ref{sec:balance} and is not used for prediction.

\subsection{Keeping the Memory Valid: Moment Transport}\label{sec:transport}
A class memory is useful only while it describes the current representation, and after a nonlinear update it does not. WiMoT therefore adapts only an affine map after the first domain. For each observed domain $s$ and class $c$, it stores the count $n_{sc}$, the mean $\mu^{\phi}_{sc}$, and the sample covariance $\Sigma^{\phi}_{sc}$ in the frozen basis. If the deleted samples $\phi_i$ of this cell were mapped to $A_t\phi_i+b_t$, their mean would be $A_t\mu^{\phi}_{sc}+b_t$ and their covariance $A_t\Sigma^{\phi}_{sc}A_t^\top$, by linearity of the mean and bilinearity of the covariance. After an adapter update the memory is therefore transported by
\begin{equation}
 \mu^h_{sc}=A_t\mu^{\phi}_{sc}+b_t,\qquad
 \Sigma^h_{sc}=A_t\Sigma^{\phi}_{sc}A_t^\top,
\end{equation}
with neither a sample nor an estimate. This is the standard affine moment rule, which CASTER and DPCR apply with estimated maps\cite{caster2026,dpcr2025}; here $A_t$ and $b_t$ are the model's own map, so the transported moments coincide with the sample moments of the transformed deleted samples (Proposition~\ref{prop:exact}(i)). The identity concerns only the stored statistics: it holds after any affine update, whatever loss trained the adapter, and it does not say that earlier classes remain separable, which Section~\ref{sec:balance} addresses. The price is that nonlinear adaptation is limited to the first domain; Section~\ref{sec:rq2} shows that the remaining affine adaptation still matters and that it requires transport.

\subsection{Balancing Plasticity and Stability: SupCon and the Historical Loss}\label{sec:balance}
Transport keeps the stored statistics valid, but an update that fits the current domain can still discard directions that separate earlier classes. \emph{Plasticity.} SupCon\cite{supcon2020} pulls same-class features of the current arrivals together and pushes different classes apart, which lets the adapter fit the new domain. \emph{Stability.} Without stored samples nothing opposes this fit, so the historical loss does, using the transported moments. For each earlier domain it is the squared error with which a linear score map of the representation predicts the one-hot labels of the deleted samples. This error depends on the samples only through their second moment and label-cross moment, which WiMoT stores, so it equals the sample-level loss and is computed in closed form (Proposition~\ref{prop:exact}(ii)). An update that gains on the current domain at the cost of the linear readability of earlier labels is penalized. Linear readability is a proxy for the accuracy of the OSCL classifier, not a guarantee.

Formally, let $\bar\phi=[\phi^\top,1]^\top$, let $e_c$ be the one-hot vector of class $c$, let $q_c=(e_c-\mathbf 1/C)/\sqrt{1-1/C}$, and let
\begin{equation}
 \widetilde F_t=\frac1{a_1}\Bigl(I_C-\frac{\mathbf 1\mathbf 1^\top}{C}\Bigr)\bigl[W_tA_t,\;W_tb_t+d_t\bigr].
\end{equation}
Here $q_c$ is a fixed unit-norm target vector of class $c$ (the one-hot vector with its class mean removed), not a per-sample label. The matrix $\widetilde F_t\in\mathbb R^{C\times513}$ maps $\bar\phi$ to class scores through the adapter and the linear map; the centering removes the part common to all classes, and the constant $a_1$, fixed after the first domain, puts the scores on the scale of the targets (Supplementary Section~S2). For an earlier domain $s$, the class-balanced second and label-cross moments are
\begin{equation}
 B_s=\frac1C\sum_c\frac1{n_{sc}}\sum_{i:y_i=c}\bar\phi_i\bar\phi_i^\top,
 \quad K_s=\frac1C\sum_cq_c[\mu_{sc}^{\phi\top},1],
\end{equation}
and the historical loss is
\begin{equation}
 \mathcal L_{\mathrm{hist}}=\frac1{t-1}\sum_{s<t}\ell_s,\quad
 \ell_s=\tr\bigl(\widetilde F_tB_s\widetilde F_t^\top\bigr)-2\bigl\langle\widetilde F_t,K_s\bigr\rangle_F+1.
\end{equation}
The following statement makes the two uses of the stored moments precise.

\begin{proposition}\label{prop:exact}
Let $\phi$ be fixed and $h=A\phi+b$. (i) For every stored domain/class cell, the sample mean and sample covariance of $\{A\phi_i+b\}$ over the deleted samples of the cell are $A\mu^{\phi}_{sc}+b$ and $A\Sigma^{\phi}_{sc}A^\top$. (ii) For every $F\in\mathbb R^{C\times513}$ and $\widetilde F=F/a_1$,
\begin{equation}
 \ell_s=\frac1C\sum_{c}\frac1{n_{sc}}\sum_{i:y_i=c}\bigl\|\widetilde F\bar\phi_i-q_c\bigr\|_2^2 .
\end{equation}
\end{proposition}
\begin{IEEEproof}
Part (i) follows from the linearity of the mean and the bilinearity of the covariance. For (ii), expanding the square gives $\tr(\widetilde F\bar\phi_i\bar\phi_i^\top\widetilde F^\top)-2q_c^\top\widetilde F\bar\phi_i+\|q_c\|_2^2$. The class-balanced average of the first term is $\tr(\widetilde FB_s\widetilde F^\top)$, that of the second is $-2\langle\widetilde F,K_s\rangle_F$ because the class average of $\bar\phi_i$ is $[\mu_{sc}^{\phi\top},1]^\top$, and $\|q_c\|_2=1$.
\end{IEEEproof}

Part (i) is the transport identity of Section~\ref{sec:transport}; it holds for any affine adapter, independently of $\mathcal L_{\mathrm{hist}}$, and lets OSCL use the historical moments in the current representation (Fig.~\ref{fig:transport}(c)). Part (ii) is separate: it makes $\mathcal L_{\mathrm{hist}}$ the true-label loss of the deleted samples (Fig.~\ref{fig:transport}(b)). $W_t$ is a probe of historical label information, not a classifier. Both statements hold only for the fixed basis and the affine adapter.

\begin{figure*}[t]
\centering
\input{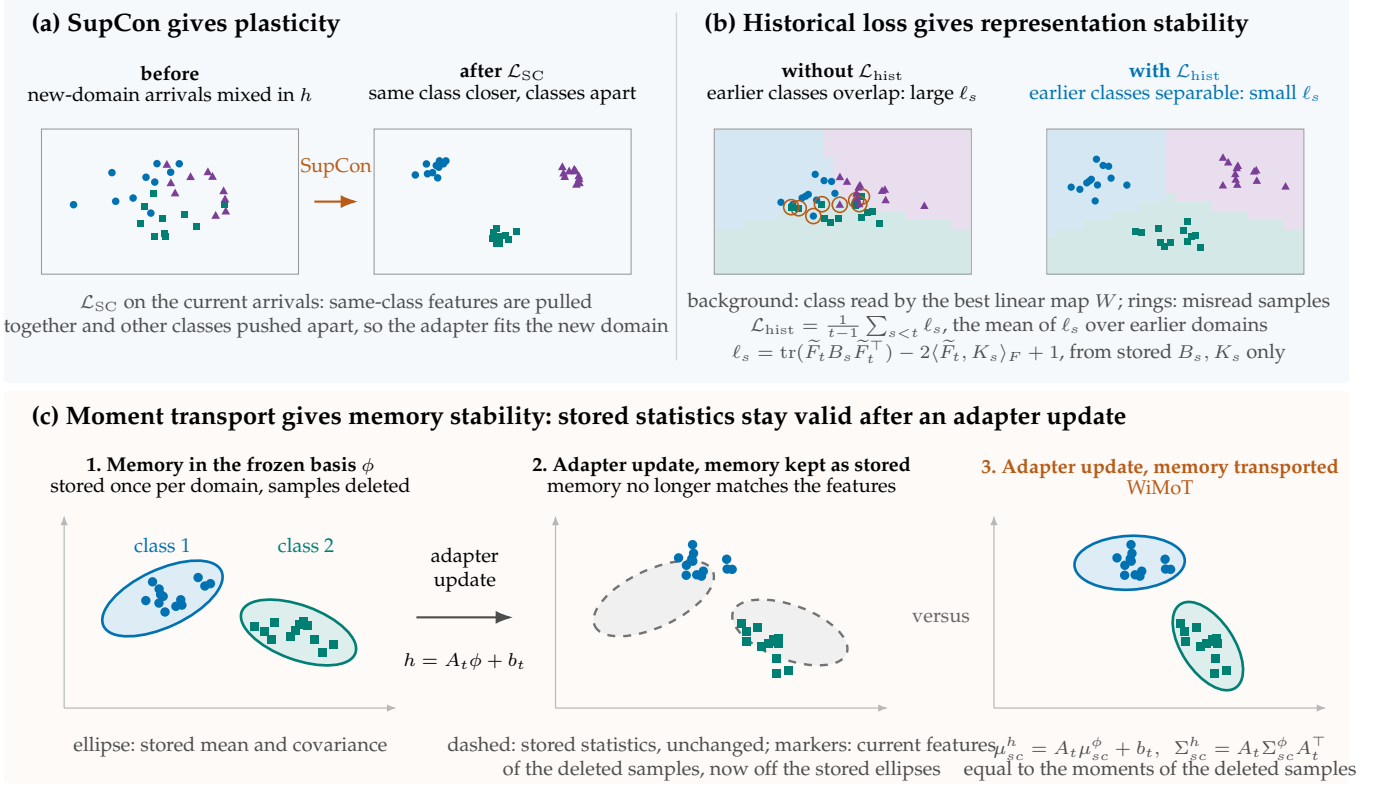}
\caption{The three components of WiMoT, shown on synthetic two-dimensional points. (a) SupCon gives plasticity. (b) The historical loss gives representation stability: earlier classes remain linearly separable (rings: samples misread by the best linear map $W$). (c) Moment transport gives memory stability: the stored statistics remain valid after an adapter update.}
\label{fig:transport}
\end{figure*}

The two losses are trained jointly, with $\mathcal L_{\mathrm{SC}}$ the SupCon loss with temperature $\tau=0.07$ on the normalized projections $u_i$ (Supplementary Section~S2). The training objective is
\begin{equation}
 \mathcal L_t=\begin{cases}0.5\,\mathcal L_{\mathrm{SC}}, & t=1,\\[2pt] 0.5\,\mathcal L_{\mathrm{SC}}+\mathcal L_{\mathrm{hist}}, & t>1.\end{cases}
\end{equation}
The weights 0.5 and 1 are used for all datasets; Section~\ref{sec:robustness} shows that the results are insensitive to them.
In the first domain, $\mathcal L_{\mathrm{SC}}$ trains the semantic stem, the adapter, and the projector. In later domains, $\mathcal L_{\mathrm{SC}}$ updates the adapter and projector on current arrivals, and $\mathcal L_{\mathrm{hist}}$ updates the adapter and $W$ from stored moments. The statistical classifier is not part of the gradient objective; it is fitted from the transported moments after each domain. Training uses AdamW with batch size 128, learning rate $10^{-3}$, weight decay $10^{-4}$, and cosine decay, for 100 epochs in the first domain and 50 in each later domain.

\subsection{A Classifier That Follows the Representation: OSCL}\label{sec:inference}
A classifier fitted on current samples sees the newest domain only, favors it, and becomes outdated when the representation changes. WiMoT instead rebuilds its classifier after every update from the transported moments of all observed domains, with equal weight per domain, so that the classifier agrees with the current representation and earlier domains enter prediction with the same weight as the current one. The classifier is OSCL\cite{oscl2023}, which builds its decision rule from class moments alone. It whitens the features with the pooled class covariance and scores a query along a direction $w=Z^\top\eta$, a combination of the class means weighted by combining labels $\eta$. Because the score of each class then has a mean and variance given by the class moments, the classification error can be estimated in closed form and minimized over $\eta$, instead of fixing $\eta$ to one-hot labels as in supervised PCA. This is the source of the no-forgetting property of OSCL: when the labels are optimized separately for each task, the earlier optimal labels, padded with zeros, remain feasible after new tasks are added, so under high-dimensional asymptotics the optimal error of an earlier task cannot increase\cite{oscl2023}. The property presupposes a fixed feature space and class moments that are valid in it, which the frozen basis and moment transport supply.

WiMoT applies OSCL without domain identity. A shared 128-dimensional Gaussian density model of $h$, built from the transported moments with uniform domain and class priors, gives by Bayes' rule the probability $g_s(x)$ that the query comes from domain $s$, and the OSCL classifier of domain $s$ gives the class posterior $p_s(c\mid x)$. Summing over the observed domains removes the unknown domain:
\begin{equation}
 \hat y(x)=\arg\max_{c}\sum_{s\le t}g_s(x)\,p_s(c\mid x).
\end{equation}
A query is thus decided mainly by the OSCL rule of the domains it resembles, and no domain identifier is supplied; OSCL is the only classifier of WiMoT. In WiMoT the domains share labels optimized with equal weight and the risk is a finite-sample estimate, so the asymptotic guarantee motivates the classifier rather than bounding its forgetting; the ablation in Section~\ref{sec:rq2} tests the choice. The whitening, combining-label optimization, and spectral weighting of the domain-conditioned densities are given in Supplementary Section~S3; the score family (Gaussian or Student-$t$) and the domain gate are chosen before testing by two-fold cross-fitting of the statistical head on training statistics.

\subsection{Update Procedure and Complexity}
Each domain after the first is learned in five steps: encode the current arrivals once in the fixed basis $\phi$; compute their class moments $(n_{tc},\mu^{\phi}_{tc},\Sigma^{\phi}_{tc})$ and keep all earlier moments in basis coordinates; optimize $0.5\mathcal L_{\mathrm{SC}}+\mathcal L_{\mathrm{hist}}$ for 50 epochs, where $\mathcal L_{\mathrm{hist}}$ reads $B_s,K_s$ of all earlier domains; transport all stored class moments through the updated adapter and refit the OSCL classifier; add $B_t,K_t$ to the historical constraints and delete the arrivals. With basis width $d=512$ and adapter width $k=128$, evaluating $\mathcal L_{\mathrm{hist}}$ costs $O((t-1)C(d+1)^2)$ per step, independent of the number of historical samples, and the refit costs $O(tC(kd^2+k^2d))$ once per update. Covariance storage grows as $O(TCd^2)$, not with the number of samples (Section~\ref{sec:deploy}).

\section{Experiments}\label{sec:experiments}

The experiments address four questions that follow the design rationale of Section~\ref{sec:rationale}: how WiMoT compares with replay and regularization baselines without storing samples (RQ1), whether it balances plasticity and stability and what each mechanism contributes (RQ2), whether the lead holds across domain orders, rooms, label budgets, frozen encoders, and loss weights (RQ3), and what it costs on an edge node (RQ4).

\subsection{Setup}
\emph{Data streams.} Table~\ref{tab:data} lists the three streams. The 3DO stream uses three acquisition days and three activities\cite{dataset3do2024}. The WiMANS stream uses the classroom, meeting-room, and empty-room recordings and nine activity labels\cite{wimans2024}, a recording-level subset of the multi-user benchmark. The Widar3.0 stream uses the nine users recorded in Room~1 on 2018-11-30 as nine domains and six gestures shared by their archives\cite{widar2019}; repetitions 1--3, 4, and 5 of each user, gesture, location, and orientation form the training, validation, and test partitions (Supplementary Section~S6).

\emph{Baselines.} We compare twenty-one baseline configurations: fine-tuning, ER\cite{er2019}, EWC\cite{ewc2017}, SI\cite{si2017}, MAS\cite{mas2018}, LwF\cite{lwf2018}, DER++\cite{der2020}, and A-GEM\cite{agem2019}, with separate native and ER-augmented versions of EWC, SI, MAS, and LwF; the recent domain-incremental methods CN-DPM\cite{cndpm2020}, DUCT\cite{duct2025}, and DualCP with its VanillaCP variant\cite{dualcp2025}, each without and with replay; and the Wi-Fi continual-learning methods ConSense\cite{consense2025} and EdgeCL\cite{edgecl2026}. Every method, WiMoT and each baseline, receives a 9728-dimensional input consisting of the 9216-dimensional frozen Mantis features and the 512-dimensional amplitude descriptor; the baselines process it through independently trained networks (ConSense and EdgeCL read it as 19 tokens), with the recent and Wi-Fi methods adapted from their published method cores to this interface. Each method coefficient is selected from three prespecified candidates by the seed-42 validation AFA of 3DO and WiMANS and then fixed; DualCP needed an additional candidate (Supplementary Section~S1). Replay baselines use the cache of Section~\ref{sec:retention}.

\emph{Protocol and metrics.} Every configuration is trained with seeds 42--44, with 100 epochs in the first domain and 50 in each later domain. After each update, all observed domains are evaluated without supplying their identity. With $A_{s,t}$ the test accuracy on domain $t$ after update $s$, average final accuracy (AFA) and average forgetting (AFR) are
\begin{equation}
 \begin{aligned}
 \mathrm{AFA}&=\frac1T\sum_{t=1}^T A_{T,t},\\
 \mathrm{AFR}&=\frac1{T-1}\sum_{t<T}\left(\max_{t\le s<T}A_{s,t}-A_{T,t}\right).
 \end{aligned}
\end{equation}
AFA is given in percent and AFR in percentage points (pp). Tables report the mean and sample standard deviation over seeds, which describe training variability rather than variation over people or environments; differences are paired by seed.

\begin{table}[t]
\centering\scriptsize
\setlength{\tabcolsep}{1.7pt}
\caption{Data streams (counts per domain; Cache: capacity of the replay cache)}
\label{tab:data}
\begin{tabular}{lllcccc}
\toprule
Stream & Domains (order) & Labels & Train/val/test & Channels & Cache\\
\midrule
3DO & 3 days & 3 activities & 2400/600/1200 & 52 & 240\\
WiMANS & 3 rooms & 9 activities & 378/108/108 & 270 & 36\\
Widar3.0 & 9 users, Room 1 & 6 gestures & 446--450/149--150/150 & 540 & 42\\
\bottomrule
\end{tabular}
\end{table}

\subsection{RQ1: Overall Effect Without Sample Replay}\label{sec:rq1}
\begin{table*}[t]
\centering\footnotesize
\setlength{\tabcolsep}{4pt}
\caption{Continual Wi-Fi sensing performance (mean $\pm$ sample SD, three runs). Bold: best of each column; WiMoT is also bold when it ranks second}
\label{tab:main}
\begin{tabular}{lcccccc}
\toprule
Method & \multicolumn{2}{c}{3DO} & \multicolumn{2}{c}{WiMANS} & \multicolumn{2}{c}{Widar3.0}\\
\cmidrule(lr){2-3}\cmidrule(lr){4-5}\cmidrule(lr){6-7}
& AFA (\%) $\uparrow$ & AFR (pp) $\downarrow$ & AFA (\%) $\uparrow$ & AFR (pp) $\downarrow$ & AFA (\%) $\uparrow$ & AFR (pp) $\downarrow$\\
\midrule
\multicolumn{7}{l}{\textit{Methods without replay} (no historical individual is stored)}\\[2pt]
\textbf{WiMoT} & $\boldsymbol{94.02\pm0.07}$ & $0.32\pm0.33$ & $\boldsymbol{89.61\pm0.47}$ & $\boldsymbol{0.00\pm0.46}$ & $\boldsymbol{95.16\pm1.23}$ & $\boldsymbol{0.58\pm0.58}$\\
Fine-tuning & $84.19\pm1.13$ & $17.62\pm1.70$ & $86.93\pm1.17$ & $4.48\pm0.96$ & $70.02\pm0.38$ & $28.78\pm0.34$\\
EWC (2017)\cite{ewc2017} & $84.19\pm1.10$ & $17.62\pm1.66$ & $86.93\pm1.17$ & $4.48\pm0.96$ & $70.62\pm0.37$ & $28.19\pm0.38$\\
SI (2017)\cite{si2017} & $89.38\pm0.21$ & $9.42\pm0.29$ & $87.76\pm1.52$ & $4.01\pm0.96$ & $75.51\pm0.28$ & $20.17\pm0.36$\\
MAS (2018)\cite{mas2018} & $87.72\pm1.13$ & $5.65\pm0.79$ & $88.17\pm1.28$ & $3.55\pm1.87$ & $68.44\pm0.58$ & $11.25\pm1.64$\\
LwF (2016)\cite{lwf2018} & $85.17\pm0.92$ & $4.78\pm0.32$ & $88.58\pm0.62$ & $0.93\pm0.46$ & $79.06\pm0.60$ & $14.75\pm0.43$\\
DUCT (2025)\cite{duct2025} & $84.80\pm1.64$ & $1.51\pm0.68$ & $74.49\pm7.86$ & $7.87\pm7.65$ & $37.41\pm7.26$ & $21.58\pm4.64$\\
DualCP (2025)\cite{dualcp2025} & $77.14\pm0.56$ & $10.46\pm0.40$ & $43.93\pm0.18$ & $6.17\pm0.53$ & $65.95\pm0.34$ & $3.36\pm0.25$\\
VanillaCP (2025)\cite{dualcp2025} & $77.81\pm0.04$ & $7.85\pm0.05$ & $67.70\pm0.64$ & $1.54\pm0.27$ & $73.14\pm0.31$ & $5.72\pm0.13$\\
ConSense (2025)\cite{consense2025} & $76.31\pm0.39$ & $17.49\pm0.73$ & $83.85\pm1.08$ & $4.94\pm3.01$ & $64.62\pm0.93$ & $27.50\pm1.28$\\
\midrule
\multicolumn{7}{l}{\textit{Methods with replay}; common class-priority cache (3DO / WiMANS / Widar3.0): 240 (9.09\%) / 36 (8.70\%) / 42 (8.54--8.61\%)}\\[2pt]
ER (2019)\cite{er2019} & $89.71\pm0.72$ & $7.03\pm1.12$ & $86.73\pm0.31$ & $5.25\pm0.27$ & $76.94\pm0.58$ & $21.19\pm1.13$\\
EWC+ER (2017)\cite{ewc2017} & $89.64\pm0.76$ & $7.11\pm1.09$ & $87.04\pm0.62$ & $5.25\pm0.27$ & $78.91\pm1.06$ & $19.25\pm1.45$\\
SI+ER (2017)\cite{si2017} & $90.25\pm0.67$ & $5.64\pm0.06$ & $87.86\pm0.64$ & $4.01\pm0.27$ & $81.38\pm0.95$ & $13.67\pm0.75$\\
MAS+ER (2018)\cite{mas2018} & $89.28\pm0.96$ & $5.21\pm0.40$ & $89.20\pm1.07$ & $2.78\pm1.22$ & $76.17\pm0.68$ & $10.47\pm0.88$\\
LwF+ER (2016)\cite{lwf2018} & $88.36\pm0.29$ & $\boldsymbol{-1.17\pm0.41}$ & $89.09\pm0.47$ & $1.23\pm0.27$ & $82.05\pm1.07$ & $10.25\pm1.63$\\
DER++ (2020)\cite{der2020} & $89.81\pm0.89$ & $5.29\pm0.51$ & $88.27\pm0.53$ & $2.78\pm1.22$ & $80.25\pm1.01$ & $16.72\pm1.39$\\
A-GEM (2019)\cite{agem2019} & $86.22\pm0.75$ & $14.64\pm0.95$ & $86.63\pm0.71$ & $4.48\pm1.41$ & $72.91\pm0.50$ & $25.64\pm0.84$\\
CN-DPM (2020)\cite{cndpm2020} & $88.48\pm0.50$ & $8.18\pm0.32$ & $86.32\pm1.25$ & $1.54\pm2.09$ & $74.84\pm2.22$ & $18.97\pm2.63$\\
DUCT+ER (2025)\cite{duct2025} & $84.44\pm1.04$ & $-0.87\pm1.38$ & $61.11\pm31.64$ & $19.75\pm31.04$ & $40.74\pm0.32$ & $16.39\pm4.88$\\
DualCP+ER (2025)\cite{dualcp2025} & $81.92\pm0.29$ & $3.28\pm0.69$ & $57.00\pm0.99$ & $3.09\pm0.27$ & $65.33\pm0.30$ & $3.08\pm0.22$\\
VanillaCP+ER (2025)\cite{dualcp2025} & $82.65\pm0.22$ & $0.71\pm0.40$ & $71.19\pm0.47$ & $1.54\pm0.71$ & $72.17\pm0.50$ & $5.64\pm0.10$\\
EdgeCL (2026)\cite{edgecl2026} & $89.68\pm1.29$ & $5.29\pm1.37$ & $85.49\pm2.64$ & $3.55\pm2.09$ & $75.01\pm1.20$ & $17.89\pm0.89$\\
\bottomrule
\end{tabular}
\end{table*}

\begin{figure*}[t]
\centering
\input{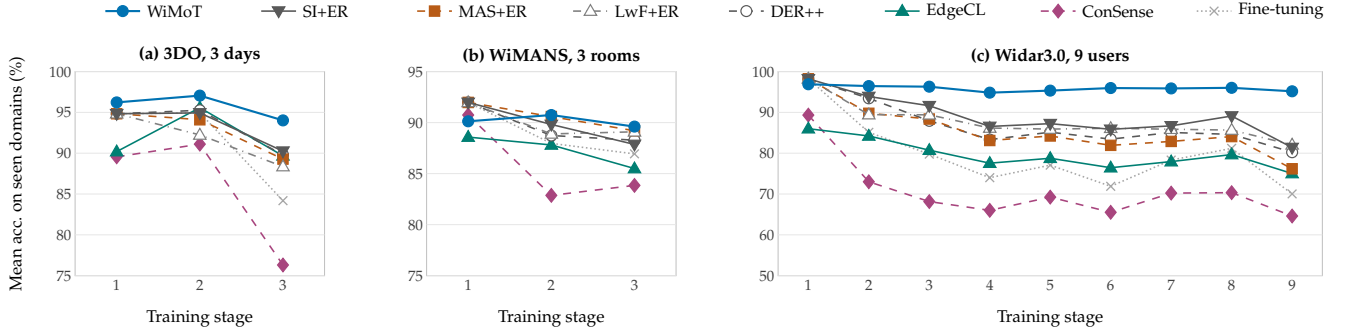}
\caption{Mean test accuracy over the observed domains after each training stage (original order, three-run means)}
\label{fig:stage}
\end{figure*}

Table~\ref{tab:main} shows that WiMoT attains the highest AFA on all three streams without storing any sample. The strongest baselines all use replay: SI+ER on 3DO, MAS+ER on WiMANS, and LwF+ER on Widar3.0. WiMoT exceeds them by $3.77\pm0.67$ and $13.11\pm1.22$ pp on 3DO and Widar3.0, in every seed. On WiMANS WiMoT matches MAS+ER in accuracy ($+0.41\pm0.64$ pp, positive in two of three seeds) while eliminating its forgetting (AFR 0.00 against 2.78 pp). The margin over the strongest method without replay is larger: $4.64\pm0.27$, $1.03\pm0.18$, and $16.10\pm1.81$ pp over SI on 3DO and LwF on the other streams. The two Wi-Fi methods, adapted here to the frozen-feature interface, stay below the strongest generic baselines on all three streams (EdgeCL 89.68, 85.49, and 75.01\%). The comparison does not hinge on the input format: if each baseline uses the better of the 9728-dimensional input and the 9216-dimensional Mantis block alone (Supplementary Table~S9), every lead remains, changing by at most 0.21 pp.

The size of the lead follows from what each method keeps. The replay baselines protect earlier domains through a cache of less than 10\% of the training records, whereas WiMoT keeps the moments of every training record of every earlier domain, and transport keeps these moments exact after each adapter update (Proposition~\ref{prop:exact}(i)); its classifier is therefore refitted on all earlier data in the current representation. The lead is largest on Widar3.0, the longest stream, where the forgetting of the baselines accumulates over eight updates, and smallest on WiMANS, which has three domains and the fewest records per class (Section~\ref{sec:scope}).

The lead comes with low forgetting: the AFR of WiMoT is 0.32, 0.00, and 0.58 pp, against 5.64, 2.78, and 10.25 pp for the strongest baselines. Only LwF+ER and DUCT+ER forget less on 3DO ($-1.17$ and $-0.87$ pp), and both have an AFA lower by 5.66 and 9.57 pp; RQ2 examines this trade-off.

Fig.~\ref{fig:stage} shows that the final differences arise from retention rather than from first-domain accuracy. On 3DO WiMoT starts 1.39 pp above the cross-entropy networks of the classic baselines (96.22 against 94.83\%); on WiMANS and Widar3.0 it starts 1.86 and 1.33 pp below them and overtakes every baseline from the second stage on. Over the nine Widar3.0 users, its accuracy on the observed users stays between 94.8 and 96.9\%, whereas LwF+ER falls from 98.22 to 82.05\% and fine-tuning to 70.02\%. The flat curve reflects the design of the classifier: every update refits OSCL from the transported moments of all users with equal weight, so earlier users enter prediction with the same weight as the newest one, whereas the baselines see earlier users only through the cache or a regularizer. After the ninth user, WiMoT keeps at least 93.8\% on each of the first eight users; the lowest final accuracy, 86.0\%, is on user 17, which also ends lowest in every other order tested (Supplementary Section~S4).

\subsection{RQ2: Plasticity--Stability Balance and the Role of Each Mechanism}\label{sec:rq2}
\emph{Balance.} A low AFR can also be reached by learning little. Table~\ref{tab:balance} therefore reports the accuracy on the domain just learned (New, plasticity), the accuracy on the earlier domains after the last update (Earlier, retention), and their mean, the Balance. With $A_{s,t}$ as in Section~\ref{sec:experiments},
\begin{equation}
 \mathrm{New}=\frac1{T-1}\sum_{t=2}^{T}A_{t,t},\quad
 \mathrm{Earlier}=\frac1{T-1}\sum_{t=1}^{T-1}A_{T,t},
\end{equation}
and $\mathrm{Balance}=(\mathrm{New}+\mathrm{Earlier})/2$. New excludes the first domain, which is learned without earlier domains, and all three quantities are computed for every run and then averaged over seeds. WiMoT has the highest Balance on 3DO and Widar3.0 (94.94 and 95.69\%), ahead of the best baseline (SI+ER in both) by $3.03\pm0.53$ and $8.51\pm0.58$ pp in every seed. On WiMANS its Balance of 89.12\% is within seed variation of the best, MAS+ER. The gain comes from retention: WiMoT has the highest Earlier accuracy of all methods (96.82, 90.74, and 96.31\%), 3.38, 2.31, and 13.48 pp above the best baseline, while its New accuracy stays close to the best baseline (93.07, 87.50, and 95.08\% against 96.56, 90.12, and 95.75\%). The baselines that learn new domains better forget more (AFR at least 5.29, 1.23, and 19.25 pp), and the two that forget less on 3DO, LwF+ER and DUCT+ER, have a Balance 6.25 and 8.97 pp lower. WiMoT thus gives up little plasticity for a large gain in stability. The small loss in New accuracy is the price of restricting later updates to an affine adapter, which transport requires (Section~\ref{sec:transport}); the gain in Earlier accuracy is what transport buys, as the ladder below shows.

\begin{table*}[t]
\centering\scriptsize
\setlength{\tabcolsep}{2.6pt}
\caption{Balance, New, Earlier (\%), and AFR (pp) of WiMoT and selected baselines (mean $\pm$ sample SD, seeds 42--44). Bold: best of each column and dataset among all twenty-two methods; WiMoT is also bold when it ranks second}
\label{tab:balance}
\begin{tabular}{l@{\hspace{4pt}}cccc@{\hspace{8pt}}cccc@{\hspace{8pt}}cccc}
\toprule
& \multicolumn{4}{c}{3DO} & \multicolumn{4}{c}{WiMANS} & \multicolumn{4}{c}{Widar3.0}\\
\cmidrule(lr){2-5}\cmidrule(lr){6-9}\cmidrule(lr){10-13}
Method & Balance & AFR & New & Earlier & Balance & AFR & New & Earlier & Balance & AFR & New & Earlier\\
\midrule
WiMoT & $\boldsymbol{94.94\pm0.20}$ & 0.32 & 93.07 & \textbf{96.82} & $\boldsymbol{89.12\pm0.61}$ & \textbf{0.00} & 87.50 & \textbf{90.74} & $\boldsymbol{95.69\pm0.77}$ & \textbf{0.58} & 95.08 & \textbf{96.31}\\
Fine-tuning & $87.62\pm0.82$ & 17.62 & 96.49 & 78.76 & $87.42\pm0.94$ & 4.48 & 88.89 & 85.96 & $81.56\pm0.17$ & 28.78 & 95.28 & 67.83\\
SI+ER & $91.92\pm0.54$ & 5.64 & 93.60 & 90.24 & $88.27\pm0.35$ & 4.01 & 89.66 & 86.88 & $87.18\pm0.59$ & 13.67 & 92.94 & 81.42\\
MAS+ER & $90.97\pm0.75$ & 5.21 & 91.71 & 90.24 & $\boldsymbol{89.27\pm0.67}$ & 2.78 & \textbf{90.12} & 88.43 & $80.01\pm0.80$ & 10.47 & 83.31 & 76.72\\
LwF+ER & $88.69\pm0.30$ & $\mathbf{-1.17}$ & 83.94 & 93.44 & $88.12\pm0.74$ & 1.23 & 88.43 & 87.81 & $86.31\pm0.40$ & 10.25 & 89.78 & 82.83\\
\bottomrule
\end{tabular}
\end{table*}

\emph{Mechanisms.} Table~\ref{tab:ablation} removes one component at a time and retrains with seeds 42--46 (Supplementary Table~S10 gives the New and Earlier accuracies). Each mechanism of Section~\ref{sec:rationale} is needed. \emph{Plasticity:} without SupCon, AFA falls by 2.54--5.12 pp on all streams in every seed. \emph{Stability of the representation:} without the historical loss, AFR rises by 0.36--1.11 pp on all streams, and AFA falls by $1.05\pm0.56$ pp on WiMANS (every seed) and $0.52\pm0.73$ pp on Widar3.0 (four of five seeds); on 3DO the change, $-0.09\pm0.44$ pp, is within seed variation. \emph{Classifier over all domains:} without the stored historical moments, OSCL sees only current-domain statistics and the historical loss is dropped; AFA falls by 17.04, 5.31, and 47.10 pp, and AFR rises by 30.28, 5.28, and 52.03 pp. On the same adapters and moments, the strongest alternative classifier (LDA on 3DO and WiMANS, QDA on Widar3.0) has an AFA lower by 1.93, 0.74, and 3.84 pp; OSCL is higher in every seed on 3DO and Widar3.0 and in three of five on WiMANS (Supplementary Section~S8). \emph{Affine adaptation:} freezing the adapter after the first domain lowers AFA by 1.07, 3.33, and 5.93 pp in every seed, most on the longest stream, so the adapter still contributes after the first domain; the transport ladder below shows that this contribution requires transport. \emph{Basis:} the amplitude branch matters most on 3DO and Widar3.0 ($-9.52$ and $-9.32$ pp without it), and the PCA shortcut on WiMANS and Widar3.0 ($-2.47$ and $-3.14$ pp). The New and Earlier accuracies show that each mechanism acts on the quantity that its role predicts (Supplementary Table~S10): without SupCon the New accuracy falls by 5.60, 6.85, and 6.45 pp, whereas without the historical loss the New accuracy changes by at most 0.28 pp and the loss appears in the Earlier accuracy and AFR. The classifier comparison points to the same conclusion: nearest class mean, which ignores the covariances, loses 13.09 and 28.19 pp on 3DO and Widar3.0, so the stored covariances, not only the means, carry the earlier domains; OSCL exceeds LDA and QDA, which use the same moments with fixed decision rules, consistent with its optimization of the combining labels for the estimated error.

\begin{table*}[t]
\centering\scriptsize
\setlength{\tabcolsep}{2.2pt}
\caption{Single-component ablation (five runs): AFA (\%), paired AFA difference from full WiMoT (pp; number of seeds with higher AFA), and AFR (pp). Bold: best AFA and lowest AFR of each dataset (the control row is not ranked); the full model is also bold when it ranks second}
\label{tab:ablation}
\resizebox{\textwidth}{!}{%
\begin{tabular}{lccccccccc}
\toprule
& \multicolumn{3}{c}{3DO} & \multicolumn{3}{c}{WiMANS} & \multicolumn{3}{c}{Widar3.0}\\
\cmidrule(lr){2-4}\cmidrule(lr){5-7}\cmidrule(lr){8-10}
Configuration & AFA & $\Delta$AFA & AFR & AFA & $\Delta$AFA & AFR & AFA & $\Delta$AFA & AFR\\
\midrule
\textbf{WiMoT (full)} & $\boldsymbol{93.98\pm0.15}$ & -- & $\boldsymbol{0.26}$ & $\boldsymbol{89.57\pm0.55}$ & -- & $-0.28$ & $\boldsymbol{95.24\pm0.88}$ & -- & $0.42$\\
\midrule
WiMoT w/o historical moments & $76.94\pm2.44$ & $-17.04\pm2.58$ (0/5) & $30.54$ & $84.26\pm1.05$ & $-5.31\pm1.44$ (0/5) & $5.00$ & $48.15\pm1.11$ & $-47.10\pm1.07$ (0/5) & $52.45$\\
WiMoT w/o SupCon & $91.44\pm0.77$ & $-2.54\pm0.77$ (0/5) & $\boldsymbol{-0.57}$ & $84.44\pm2.61$ & $-5.12\pm2.31$ (0/5) & $\boldsymbol{-3.89}$ & $92.31\pm0.66$ & $-2.93\pm1.07$ (0/5) & $\boldsymbol{0.15}$\\
WiMoT w/o $\mathcal L_{\mathrm{hist}}$ & $93.89\pm0.43$ & $-0.09\pm0.44$ (2/5) & $0.62$ & $88.52\pm0.86$ & $-1.05\pm0.56$ (0/5) & $0.83$ & $94.73\pm0.86$ & $-0.52\pm0.73$ (1/5) & $1.10$\\
WiMoT, frozen adapter ($t>1$) & $92.92\pm0.43$ & $-1.07\pm0.32$ (0/5) & $1.33$ & $86.23\pm0.60$ & $-3.33\pm0.86$ (0/5) & $0.09$ & $89.32\pm0.51$ & $-5.93\pm0.59$ (0/5) & $0.40$\\
WiMoT w/o amplitude branch & $84.47\pm0.58$ & $-9.52\pm0.59$ (0/5) & $0.43$ & $89.07\pm1.17$ & $-0.49\pm1.67$ (2/5) & $-0.56$ & $85.93\pm0.64$ & $-9.32\pm1.29$ (0/5) & $2.67$\\
WiMoT w/o PCA shortcut & $93.76\pm0.31$ & $-0.23\pm0.19$ (1/5) & $0.57$ & $87.10\pm1.14$ & $-2.47\pm0.85$ (0/5) & $0.74$ & $92.10\pm0.66$ & $-3.14\pm0.32$ (0/5) & $1.08$\\
\midrule
OSCL $\to$ nearest class mean & $80.89\pm0.89$ & $-13.09\pm0.83$ (0/5) & $14.04$ & $87.84\pm0.80$ & $-1.73\pm0.77$ (0/5) & $0.56$ & $67.05\pm1.28$ & $-28.19\pm1.51$ (0/5) & $18.20$\\
OSCL $\to$ LDA & $92.06\pm0.49$ & $-1.93\pm0.61$ (0/5) & $1.64$ & $88.83\pm1.33$ & $-0.74\pm1.39$ (1/5) & $-0.56$ & $84.43\pm0.94$ & $-10.81\pm0.59$ (0/5) & $2.18$\\
OSCL $\to$ QDA & $91.00\pm0.70$ & $-2.98\pm0.73$ (0/5) & $1.23$ & $80.93\pm1.33$ & $-8.64\pm1.09$ (0/5) & $1.11$ & $91.41\pm0.50$ & $-3.84\pm0.96$ (0/5) & $2.48$\\
\midrule
\textit{Control:} WiMoT + replay cache & $94.16\pm0.38$ & $+0.17\pm0.25$ (4/5) & $0.09$ & $88.95\pm1.12$ & $-0.62\pm0.79$ (0/5) & $0.56$ & $94.58\pm0.81$ & $-0.67\pm0.34$ (0/5) & $0.70$\\
\bottomrule
\end{tabular}}\\[2pt]
\parbox{\textwidth}{\scriptsize Without historical moments, OSCL uses only current-domain class statistics, and $\mathcal L_{\mathrm{hist}}$, which needs stored moments, is dropped. Without SupCon, the first domain has no training signal and later domains are trained by $\mathcal L_{\mathrm{hist}}$ alone. In the three rows ``OSCL $\to$'', the OSCL classifier is replaced by the named classifier fitted on the same adapters and stored moments (Supplementary Section~S8). The control adds the replay cache of the baselines (240/36/42 individuals) to the later-domain minibatches; it is not part of WiMoT. Seeds 42--44 of the full method are the runs of Table~\ref{tab:main}.}
\end{table*}

\emph{Transport.} Table~\ref{tab:ladder} isolates the transport of the stored moments on the same features, from the frozen adapter (R0) to the full method (R4); R0, R3, and R4 agree with the corresponding rows of Table~\ref{tab:ablation} within 0.09 pp. In R1 and R2, the moments of each domain are transported once, when the domain is learned, and then kept while the adapter continues to change, so the classifier is fitted on statistics that no longer match the test features. Three results follow. First, adapting without transport is worse than not adapting: the AFA of R1 is below that of R0 by 11.07, 1.05, and 25.57 pp in every seed, because the accuracy on earlier domains collapses (74.36 and 61.03\% on 3DO and Widar3.0, against 96.78 and 96.40\% for R4). Second, transport accounts for most of the gain of the full method: R1$\to$R3 raises AFA by 12.05, 3.40, and 31.01 pp in every seed, of the 12.10, 4.44, and 31.45 pp of R1$\to$R4. Third, the historical loss helps only on a valid memory: with transport it adds $+0.05$ (within seed variation), $+1.05$, and $+0.44$ pp (R3$\to$R4), whereas without transport it does not repair the stale memory (R1$\to$R2: $-3.29$, $+0.80$, and $+2.95$ pp), and R2 remains 28.50 pp below R4 on Widar3.0. All effects are smallest on WiMANS, which has only three domains.

\begin{table*}[t]
\centering\scriptsize
\setlength{\tabcolsep}{3.4pt}
\caption{Transport ladder: AFA and Earlier accuracy (\%, mean $\pm$ sample SD, seeds 42--46) with and without transport of the stored moments. Bold: best of each column}
\label{tab:ladder}
\begin{tabular}{lcccccc}
\toprule
& \multicolumn{2}{c}{3DO} & \multicolumn{2}{c}{WiMANS} & \multicolumn{2}{c}{Widar3.0}\\
\cmidrule(lr){2-3}\cmidrule(lr){4-5}\cmidrule(lr){6-7}
Configuration & AFA & Earlier & AFA & Earlier & AFA & Earlier\\
\midrule
R0 Frozen adapter ($t>1$) & $92.98\pm0.47$ & $95.39\pm0.46$ & $86.17\pm0.59$ & $87.96\pm0.73$ & $89.35\pm0.51$ & $90.88\pm0.56$\\
R1 SupCon, not transported & $81.91\pm1.66$ & $74.36\pm2.65$ & $85.12\pm0.55$ & $84.72\pm0.73$ & $63.78\pm2.51$ & $61.03\pm2.55$\\
R2 SupCon $+\,\mathcal L_{\mathrm{hist}}$, not transported & $78.62\pm1.53$ & $68.94\pm2.09$ & $85.93\pm1.49$ & $85.83\pm1.52$ & $66.73\pm2.37$ & $64.45\pm2.49$\\
R3 SupCon, transported & $93.96\pm0.46$ & $96.62\pm0.57$ & $88.52\pm0.86$ & $89.35\pm0.65$ & $94.79\pm0.88$ & $95.70\pm0.70$\\
R4 SupCon $+\,\mathcal L_{\mathrm{hist}}$, transported (full WiMoT) & $\boldsymbol{94.01\pm0.14}$ & $\boldsymbol{96.78\pm0.30}$ & $\boldsymbol{89.57\pm0.55}$ & $\boldsymbol{90.83\pm0.21}$ & $\boldsymbol{95.23\pm0.90}$ & $\boldsymbol{96.40\pm0.81}$\\
\bottomrule
\end{tabular}
\end{table*}

\subsection{RQ3: Robustness of the Lead}\label{sec:robustness}
\begin{table*}[t]
\centering\footnotesize
\setlength{\tabcolsep}{4pt}
\caption{Widar3.0 AFA and AFR (mean $\pm$ sample SD, seeds 42--44) in further user orders and rooms. $\Delta$: paired AFA difference to the strongest listed baseline. Bold: best of each column}
\label{tab:widarext}
\begin{tabular}{lcccccc}
\toprule
& \multicolumn{4}{c}{Room 1, nine users} & Room 2 & Room 3\\
\cmidrule(lr){2-5}\cmidrule(lr){6-6}\cmidrule(lr){7-7}
Method & Original & Reverse & 4-8-9-2-3-5-1-7-6 & 6-7-4-3-1-2-5-8-9 & 2 users & 4 users\\
\midrule
\multicolumn{7}{l}{\textit{AFA (\%) $\uparrow$}}\\
\textbf{WiMoT} & $\boldsymbol{95.16\pm1.23}$ & $\boldsymbol{94.47\pm0.45}$ & $\boldsymbol{94.27\pm0.38}$ & $\boldsymbol{94.79\pm0.23}$ & $\boldsymbol{92.62\pm0.89}$ & $\boldsymbol{96.32\pm0.60}$\\
Fine-tuning & $70.02\pm0.38$ & $75.58\pm1.12$ & $71.31\pm0.69$ & $68.54\pm0.49$ & $82.10\pm1.36$ & $65.94\pm0.33$\\
LwF (2016) & $79.06\pm0.60$ & $81.53\pm0.34$ & $79.90\pm1.01$ & $77.21\pm0.94$ & $77.63\pm0.84$ & $74.72\pm0.79$\\
SI+ER (2017) & $81.38\pm0.95$ & $84.89\pm0.52$ & $83.14\pm0.23$ & $80.27\pm0.48$ & $92.17\pm0.51$ & $84.30\pm0.87$\\
DER++ (2020) & $80.25\pm1.01$ & $82.27\pm0.59$ & $79.68\pm1.88$ & $79.56\pm1.00$ & $89.71\pm0.70$ & $86.40\pm0.76$\\
EdgeCL (2026) & $75.01\pm1.20$ & $76.15\pm0.56$ & $74.86\pm1.08$ & $73.88\pm1.23$ & $81.54\pm1.16$ & $78.67\pm0.68$\\
\midrule
\multicolumn{7}{l}{\textit{AFR (pp) $\downarrow$}}\\
\textbf{WiMoT} & $\boldsymbol{0.58\pm0.58}$ & $\boldsymbol{1.42\pm0.79}$ & $\boldsymbol{1.22\pm0.32}$ & $\boldsymbol{1.11\pm0.13}$ & $\boldsymbol{-0.45\pm0.39}$ & $\boldsymbol{0.52\pm0.46}$\\
Fine-tuning & $28.78\pm0.34$ & $22.58\pm1.32$ & $27.11\pm0.67$ & $29.83\pm0.52$ & $24.38\pm1.40$ & $41.54\pm0.46$\\
LwF (2016) & $14.75\pm0.43$ & $11.78\pm0.42$ & $11.64\pm0.88$ & $16.39\pm0.93$ & $4.03\pm1.78$ & $18.55\pm1.36$\\
SI+ER (2017) & $13.67\pm0.75$ & $10.33\pm0.55$ & $11.67\pm0.51$ & $14.33\pm0.38$ & $3.80\pm1.03$ & $15.58\pm1.56$\\
DER++ (2020) & $16.72\pm1.39$ & $14.33\pm0.76$ & $16.94\pm2.02$ & $17.14\pm0.82$ & $5.59\pm0.77$ & $12.63\pm0.67$\\
EdgeCL (2026) & $17.89\pm0.89$ & $16.39\pm0.70$ & $18.94\pm1.33$ & $18.86\pm1.35$ & $5.15\pm1.03$ & $18.47\pm1.18$\\
\midrule
$\Delta$ (pp) & $+13.78\pm0.85$ & $+9.58\pm0.28$ & $+11.14\pm0.15$ & $+14.52\pm0.26$ & $+0.45\pm0.51$ & $+9.92\pm1.27$\\
Seeds with $\Delta>0$ & 3/3 & 3/3 & 3/3 & 3/3 & 2/3 & 3/3\\
\bottomrule
\end{tabular}
\end{table*}

\begin{table}[t]
\centering\scriptsize
\setlength{\tabcolsep}{2.4pt}
\caption{Robustness to domain order on 3DO and WiMANS (mean $\pm$ SD over six orders; AFA in \%, AFR in pp). Bold: best of each column; WiMoT is also bold when it ranks second}
\label{tab:orders}
\begin{tabular}{lcccc}
\toprule
Method & \multicolumn{2}{c}{3DO} & \multicolumn{2}{c}{WiMANS}\\
\cmidrule(lr){2-3}\cmidrule(lr){4-5}
& AFA $\uparrow$ & AFR $\downarrow$ & AFA $\uparrow$ & AFR $\downarrow$\\
\midrule
\multicolumn{5}{l}{\textit{Without replay}}\\
\textbf{WiMoT} & $\boldsymbol{93.71\pm2.19}$ & $\boldsymbol{3.89\pm3.26}$ & $\boldsymbol{88.80\pm0.89}$ & $\boldsymbol{-0.31\pm0.46}$\\
Fine-tuning & $84.56\pm0.73$ & $16.04\pm1.30$ & $86.51\pm0.57$ & $4.76\pm1.16$\\
EWC & $84.56\pm0.74$ & $16.05\pm1.32$ & $86.52\pm0.56$ & $4.73\pm1.13$\\
SI & $87.74\pm1.19$ & $10.59\pm0.93$ & $87.55\pm0.28$ & $3.40\pm1.00$\\
MAS & $86.06\pm1.77$ & $6.29\pm3.20$ & $85.63\pm2.70$ & $2.80\pm1.24$\\
LwF & $86.28\pm1.09$ & $5.12\pm1.51$ & $87.67\pm0.97$ & $1.18\pm1.24$\\
ConSense & $77.30\pm1.99$ & $14.72\pm4.35$ & $82.22\pm1.52$ & $6.79\pm2.19$\\
\midrule
\multicolumn{5}{l}{\textit{With replay (common cache 240/36)}}\\
ER & $88.80\pm0.76$ & $8.10\pm1.00$ & $87.67\pm0.90$ & $2.85\pm1.85$\\
EWC+ER & $88.77\pm0.77$ & $8.14\pm1.03$ & $88.08\pm0.86$ & $2.52\pm1.75$\\
SI+ER & $89.02\pm0.81$ & $6.62\pm0.75$ & $88.46\pm0.83$ & $2.01\pm1.48$\\
MAS+ER & $87.90\pm0.97$ & $5.53\pm1.28$ & $87.77\pm1.14$ & $1.93\pm0.60$\\
LwF+ER & $88.54\pm0.99$ & $\boldsymbol{0.02\pm2.75}$ & $87.88\pm0.81$ & $1.00\pm0.57$\\
DER++ & $88.20\pm2.45$ & $6.24\pm3.29$ & $87.91\pm0.35$ & $2.83\pm0.95$\\
A-GEM & $85.75\pm0.56$ & $14.35\pm0.78$ & $86.69\pm0.71$ & $4.27\pm1.14$\\
EdgeCL & $86.43\pm2.38$ & $8.09\pm2.57$ & $85.70\pm0.69$ & $2.91\pm0.89$\\
\bottomrule
\end{tabular}
\end{table}

\begin{figure}[t]
\centering
\definecolor{fbblue}{HTML}{0072B2}\definecolor{fborange}{HTML}{B85B16}\definecolor{fbgreen}{HTML}{007F73}%
\begin{tikzpicture}[x=1mm,y=1mm,
tick/.style={font=\fontsize{5.6}{6.4}\selectfont,inner sep=1.2pt,text=black!75},
lab/.style={font=\fontsize{6.4}{7.2}\selectfont,inner sep=1pt},
leg/.style={font=\fontsize{6.0}{7}\selectfont,inner sep=1pt,anchor=west}]
\draw[black!10,line width=.25pt] (10.00,0) -- (10.00,46.00);
\node[tick,anchor=north] at (10.00,0) {0};
\draw[black!10,line width=.25pt] (20.00,0) -- (20.00,46.00);
\node[tick,anchor=north] at (20.00,0) {4};
\draw[black!10,line width=.25pt] (30.00,0) -- (30.00,46.00);
\node[tick,anchor=north] at (30.00,0) {8};
\draw[black!10,line width=.25pt] (40.00,0) -- (40.00,46.00);
\node[tick,anchor=north] at (40.00,0) {12};
\draw[black!10,line width=.25pt] (50.00,0) -- (50.00,46.00);
\node[tick,anchor=north] at (50.00,0) {16};
\draw[black!10,line width=.25pt] (60.00,0) -- (60.00,46.00);
\node[tick,anchor=north] at (60.00,0) {20};
\draw[black!10,line width=.25pt] (70.00,0) -- (70.00,46.00);
\node[tick,anchor=north] at (70.00,0) {24};
\draw[black!10,line width=.25pt] (0,0.00) -- (70.00,0.00);
\node[tick,anchor=east] at (0,0.00) {-8};
\draw[black!10,line width=.25pt] (0,9.20) -- (70.00,9.20);
\node[tick,anchor=east] at (0,9.20) {0};
\draw[black!10,line width=.25pt] (0,18.40) -- (70.00,18.40);
\node[tick,anchor=east] at (0,18.40) {+8};
\draw[black!10,line width=.25pt] (0,27.60) -- (70.00,27.60);
\node[tick,anchor=east] at (0,27.60) {+16};
\draw[black!10,line width=.25pt] (0,36.80) -- (70.00,36.80);
\node[tick,anchor=east] at (0,36.80) {+24};
\draw[black!10,line width=.25pt] (0,46.00) -- (70.00,46.00);
\node[tick,anchor=east] at (0,46.00) {+32};
\draw[black!55,line width=.4pt] (0,9.20) -- (70.00,9.20);
\draw[black!30,line width=.3pt] (0,0) rectangle (70.00,46.00);
\draw[black!45,dashed,line width=.45pt] (1.80,0.92) -- (66.05,43.49);
\path[fill=fbblue,draw=black,line width=.45pt] (24.10,13.54) circle (1.05);
\path[fill=fbblue,draw=white,line width=.2pt] (28.57,15.00) circle (0.80);
\path[fill=fbblue,draw=white,line width=.2pt] (28.10,17.03) circle (0.80);
\path[fill=fbblue,draw=white,line width=.2pt] (26.35,16.51) circle (0.80);
\path[fill=fbblue,draw=white,line width=.2pt] (20.48,10.50) circle (0.80);
\path[fill=fbblue,draw=white,line width=.2pt] (18.15,11.86) circle (0.80);
\path[fill=fbblue,draw=white,line width=.2pt] (17.12,11.96) circle (0.80);
\path[fill=fbblue,draw=white,line width=.2pt] (20.15,12.43) circle (0.80);
\path[fill=fbblue,draw=white,line width=.2pt] (29.80,16.46) circle (0.80);
\path[fill=fbblue,draw=white,line width=.2pt] (20.60,18.27) circle (0.80);
\path[fill=fbblue,draw=white,line width=.2pt] (20.27,17.60) circle (0.80);
\path[fill=fbblue,draw=white,line width=.2pt] (27.60,13.97) circle (0.80);
\path[fill=fbblue,draw=white,line width=.2pt] (26.05,16.89) circle (0.80);
\path[fill=fbblue,draw=white,line width=.2pt] (24.65,20.19) circle (0.80);
\path[fill=fborange,draw=black,line width=.45pt] (16.00,8.73) rectangle (17.89,10.62);
\path[fill=fborange,draw=white,line width=.2pt] (10.83,7.30) rectangle (12.27,8.74);
\path[fill=fborange,draw=white,line width=.2pt] (13.90,10.61) rectangle (15.35,12.05);
\path[fill=fborange,draw=white,line width=.2pt] (16.61,9.66) rectangle (18.04,11.10);
\path[fill=fborange,draw=white,line width=.2pt] (8.50,6.24) rectangle (9.95,7.68);
\path[fill=fborange,draw=white,line width=.2pt] (13.13,7.41) rectangle (14.57,8.85);
\path[fill=fborange,draw=white,line width=.2pt] (15.46,6.82) rectangle (16.89,8.26);
\path[fill=fborange,draw=white,line width=.2pt] (13.90,7.07) rectangle (15.35,8.51);
\path[fill=fborange,draw=white,line width=.2pt] (11.98,4.57) rectangle (13.42,6.01);
\path[fill=fborange,draw=white,line width=.2pt] (6.21,10.14) rectangle (7.64,11.58);
\path[fill=fborange,draw=white,line width=.2pt] (6.58,2.68) rectangle (8.02,4.12);
\path[fill=fborange,draw=white,line width=.2pt] (12.36,5.41) rectangle (13.80,6.85);
\path[fill=fborange,draw=white,line width=.2pt] (5.43,3.04) rectangle (6.87,4.48);
\path[fill=fborange,draw=white,line width=.2pt] (2.33,3.27) rectangle (3.77,4.71);
\path[fill=fbgreen,draw=black,line width=.45pt] (35.62,25.59) -- (34.48,23.62) -- (36.77,23.62) -- cycle;
\path[fill=fbgreen,draw=white,line width=.2pt] (39.18,23.01) -- (38.31,21.51) -- (40.05,21.51) -- cycle;
\path[fill=fbgreen,draw=white,line width=.2pt] (45.82,26.90) -- (44.95,25.40) -- (46.69,25.40) -- cycle;
\path[fill=fbgreen,draw=white,line width=.2pt] (35.83,21.22) -- (34.96,19.72) -- (36.70,19.72) -- cycle;
\path[fill=fbgreen,draw=white,line width=.2pt] (19.50,10.72) -- (18.63,9.22) -- (20.37,9.22) -- cycle;
\path[fill=fbgreen,draw=white,line width=.2pt] (41.58,21.61) -- (40.71,20.11) -- (42.45,20.11) -- cycle;
\path[fill=fbgreen,draw=white,line width=.2pt] (27.77,22.36) -- (26.90,20.86) -- (28.64,20.86) -- cycle;
\path[fill=fbgreen,draw=white,line width=.2pt] (32.08,21.53) -- (31.21,20.03) -- (32.95,20.03) -- cycle;
\path[fill=fbgreen,draw=white,line width=.2pt] (30.78,25.95) -- (29.91,24.45) -- (31.65,24.45) -- cycle;
\path[fill=fbgreen,draw=white,line width=.2pt] (64.80,45.84) -- (63.93,44.34) -- (65.67,44.34) -- cycle;
\path[fill=fbgreen,draw=white,line width=.2pt] (35.83,29.31) -- (34.96,27.81) -- (36.70,27.81) -- cycle;
\path[fill=fbgreen,draw=white,line width=.2pt] (45.20,30.78) -- (44.33,29.28) -- (46.07,29.28) -- cycle;
\path[fill=fbgreen,draw=white,line width=.2pt] (44.18,31.61) -- (43.31,30.11) -- (45.05,30.11) -- cycle;
\path[fill=fbgreen,draw=white,line width=.2pt] (47.03,40.56) -- (46.16,39.06) -- (47.90,39.06) -- cycle;
\node[lab,anchor=north] at (35.00,-3.6) {AFR of the strongest baseline (pp)};
\node[lab,rotate=90,anchor=south] at (-5.6,23.00) {Lead of WiMoT in AFA (pp)};
\path[fill=fbblue,draw=white,line width=.2pt] (3.20,43.00) circle (0.80);
\node[leg] at (4.8,43.00) {3DO};
\path[fill=fborange,draw=white,line width=.2pt] (2.48,38.88) rectangle (3.92,40.32);
\node[leg] at (4.8,39.60) {WiMANS};
\path[fill=fbgreen,draw=white,line width=.2pt] (3.20,37.20) -- (2.33,35.70) -- (4.07,35.70) -- cycle;
\node[leg] at (4.8,36.20) {Widar3.0};
\path[fill=white,draw=black,line width=.45pt] (3.20,32.80) circle (1.05);
\node[leg] at (4.8,32.80) {main setting};
\draw[black!45,dashed,line width=.45pt] (1.6,29.40) -- (4.4,29.40);
\node[leg] at (4.8,29.40) {least-squares line, $r=0.94$};
\end{tikzpicture}
\caption{AFA lead of WiMoT over the strongest baseline of each setting against the AFR of that baseline in the 42 settings of RQ3 (domain orders, rooms, label budgets, and frozen encoders; three-seed means, Supplementary Table~S7). Pearson $r=0.94$, Spearman $\rho=0.94$; $r=0.80$ after the dataset means are removed}
\label{fig:leadafr}
\end{figure}
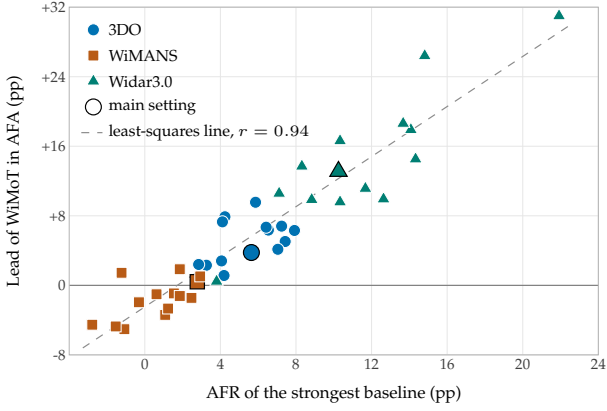

\emph{Domain orders, rooms, and label budgets.} We measure the lead in 24 settings that change the arrival order, the room, or the number of labels in later domains; in each setting it is measured against the strongest baseline among those run there (all 21 in the main settings and 5--14 elsewhere; Supplementary Section~S4 and Table~S7). Orders number the domains by their position in the original order; on Widar3.0, users 5 and 10--17 of Room~1 are numbered 1--9. On 3DO WiMoT leads in all six domain orders in every seed, by 1.13--6.81 pp (Supplementary Table~S1). Widar3.0 gives the largest and most stable lead (Table~\ref{tab:widarext}): in the reverse order and two random orders of the Room-1 users, WiMoT exceeds SI+ER by 9.58--14.52 pp in every seed, and in Room~3 (four users) it exceeds DER++ by $9.92\pm1.27$ pp in every seed. Even in Room~2, with two users and a single update (one test window corresponds to 0.67 pp), WiMoT remains ahead of SI+ER on average ($+0.45\pm0.51$ pp). Across these five Widar3.0 streams its AFR is $-0.45$ to 1.42 pp, against at least 3.80 pp for every listed baseline. With only 25\% and 10\% of the labels in every later domain, WiMoT still leads by 2.81 and 2.40 pp on 3DO and by 9.85 and 10.57 pp on Widar3.0 in every seed. None of the sources of the lead depends on the arrival order: transport is exact for any sequence of affine updates, and the classifier weights all observed domains equally. Fewer labels make each new moment less accurate, but a moment, once stored, is not degraded by later updates.

Averaged over the six domain orders of 3DO and WiMANS (Table~\ref{tab:orders}), WiMoT has the highest AFA on both streams and the lowest AFR on WiMANS; on 3DO only LwF+ER forgets less on average. Its forgetting depends more on the order than its accuracy: over the six 3DO orders its AFR ranges from 0.32 to 8.15 pp.

\emph{Frozen encoders.} Table~\ref{tab:backbone} replaces MantisV2 by five released time-series foundation models and by a randomly initialized MantisV2, without retuning any recipe or coefficient. Five baselines (fine-tuning, LwF, SI+ER, DER++, EdgeCL) are rerun on the features of each encoder, so the lead here is over the strongest of these five and differs from the lead in Table~\ref{tab:main}. On 3DO and Widar3.0 WiMoT exceeds them with every encoder and in every seed, by 3.77--9.56 and 13.70--30.99 pp, so the lead does not depend on a particular encoder: the baselines receive the same frozen features, and the difference lies in how earlier domains are retained. On WiMANS WiMoT leads with both MantisV2 encoders ($+1.03$ and $+1.44$ pp); with the five other encoders the strongest baseline is ahead, which Section~\ref{sec:scope} examines. On 3DO even the randomly initialized MantisV2 reaches 96.03\%, which, together with the ablation of the amplitude branch ($-9.52$ pp on 3DO; Table~\ref{tab:ablation}), identifies the amplitude descriptor as the main carrier of 3DO class information.

\begin{table}[t]
\centering\scriptsize
\setlength{\tabcolsep}{2.4pt}
\caption{AFA of WiMoT (\%) with other frozen encoders and paired difference from the strongest of five baselines (pp). Bold: WiMoT has the highest AFA among itself and the five baselines}
\label{tab:backbone}
\begin{tabular}{lccc}
\toprule
Encoder & 3DO & WiMANS & Widar3.0\\
\midrule
MantisV2 (pretrained) & \textbf{94.02} ($+3.77$) & \textbf{89.61} ($+1.03$) & \textbf{95.16} ($+13.78$)\\
MantisV2 (random init.) & \textbf{96.03} ($+7.89$) & \textbf{78.29} ($+1.44$) & \textbf{87.33} ($+30.99$)\\
MOMENT-1-large & \textbf{94.82} ($+7.30$) & 78.50 ($-5.04$) & \textbf{86.20} ($+16.62$)\\
TSPulse-r1 & \textbf{95.56} ($+9.56$) & 68.42 ($-4.53$) & \textbf{84.69} ($+26.40$)\\
Chronos-2 & \textbf{93.88} ($+6.31$) & 86.11 ($-3.40$) & \textbf{88.12} ($+13.70$)\\
TimesFM 2.5 & \textbf{91.18} ($+4.15$) & 81.58 ($-2.67$) & \textbf{87.36} ($+17.90$)\\
TimesFM 3.0 & \textbf{93.75} ($+6.69$) & 80.86 ($-4.73$) & \textbf{87.43} ($+18.62$)\\
\bottomrule
\end{tabular}
\end{table}

\emph{Where the lead comes from.} The size of the lead follows the forgetting of the strongest baseline: over the 42 settings of this section, the 24 above and the 18 of Table~\ref{tab:backbone}, the lead grows almost linearly with the AFR of that baseline (Fig.~\ref{fig:leadafr}; Pearson $r=0.94$, Spearman $\rho=0.94$), also after dataset means are removed ($r=0.80$, $\rho=0.62$; Supplementary Table~S7). This agrees with RQ2, where the advantage of WiMoT came from retention; the association is weaker on 3DO, and as a correlation it does not establish a mechanism.

\emph{Loss weights.} The weights of SupCon and of the historical loss (0.5 and 1) are the same for all datasets. Varying them in seeds 42--44 (Fig.~\ref{fig:sweep}) shows that the results are not sensitive to them: for weight ratios $\lambda_{\mathrm{sc}}/\lambda_{\mathrm h}$ between 0.1 and 2, the mean AFA varies by at most 0.18, 0.51, and 0.25 pp on 3DO, WiMANS, and Widar3.0, less than the largest seed SD of each series (0.28, 1.58, and 1.37 pp), and stays at least 3.69 and 12.91 pp above the strongest baseline on 3DO and Widar3.0 and on par with it on WiMANS ($-0.10$ to $+0.41$ pp). Setting either weight to zero lowers AFA by 0.91--1.36 pp on WiMANS and Widar3.0 but not on 3DO, as in the ablation.

\begin{figure*}[t]
\centering
\input{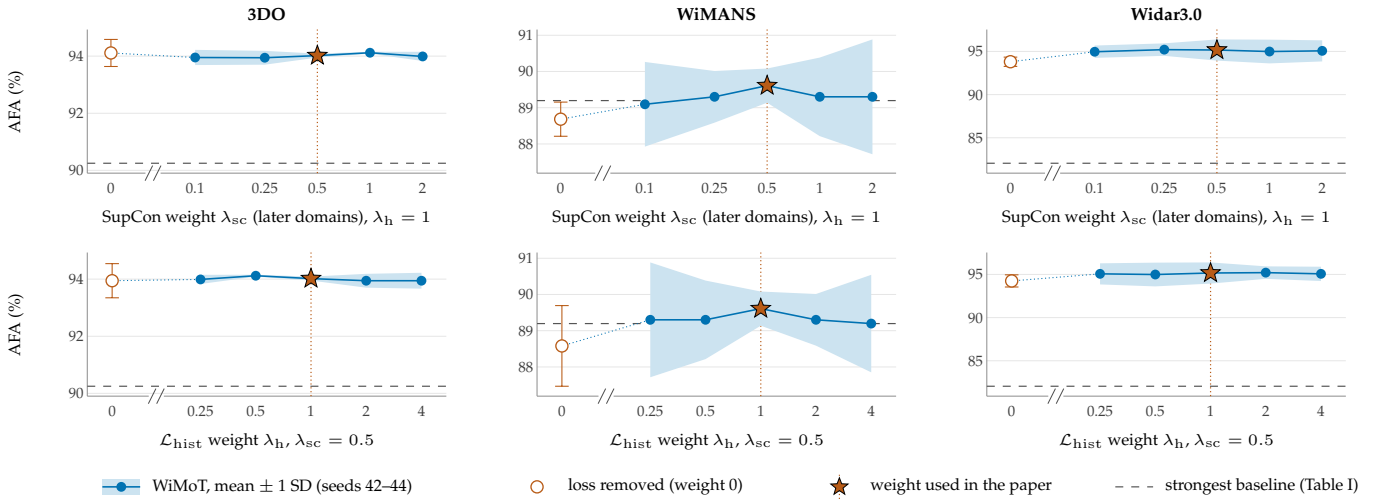}
\caption{AFA of WiMoT against the SupCon weight (upper row, later domains, $\lambda_{\mathrm h}=1$) and the historical-loss weight (lower row, $\lambda_{\mathrm{sc}}=0.5$); mean $\pm$ SD over seeds 42--44. Star: weight of the paper; dashed line: strongest baseline of Table~\ref{tab:main}; open circle: weight 0. Vertical axes differ between datasets. Supplementary Table~S12 lists all values}
\label{fig:sweep}
\end{figure*}

\subsection{RQ4: Deployment Cost}\label{sec:deploy}
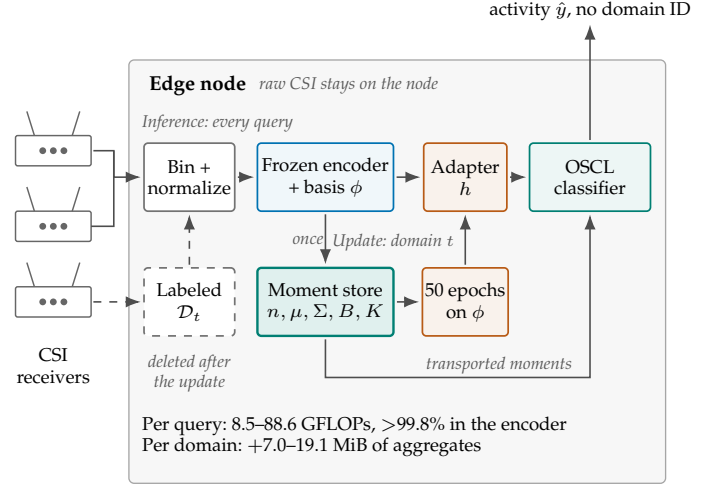
\begin{figure}[t]
\centering
\definecolor{fbblue}{HTML}{0072B2}\definecolor{fborange}{HTML}{B85B16}\definecolor{fbgreen}{HTML}{007F73}%
\begin{tikzpicture}[x=1mm,y=1mm,>=Latex,font=\fontsize{7}{8.2}\selectfont,
 box/.style={align=center,inner sep=1pt,rounded corners=1.5pt,line width=.65pt,minimum height=9mm},
 arr/.style={->,draw=black!70,line width=.6pt,shorten >=.3mm,shorten <=.3mm},
 note/.style={align=center,inner sep=.5pt,text=black!65,font=\fontsize{6.4}{7.4}\selectfont\itshape}]
\path[use as bounding box] (0,-1) rectangle (88,64);
\fill[black!3,rounded corners=2pt] (17,0) rectangle (88,56);
\draw[black!35,rounded corners=2pt,line width=.5pt] (17,0) rectangle (88,56);
\node[anchor=west,font=\fontsize{7.6}{9}\selectfont\bfseries] at (18.5,52.6) {Edge node};
\node[anchor=west,note] at (34,52.6) {raw CSI stays on the node};
\foreach \y in {44,34,24}{
 \draw[black!60,fill=white,rounded corners=.8pt,line width=.55pt] (2,\y-2) rectangle (12,\y+2);
 \draw[black!60,line width=.55pt] (4.5,\y+2)--(3.5,\y+5.2) (9.5,\y+2)--(10.5,\y+5.2);
 \fill[black!60] (5.6,\y) circle(.45) (7,\y) circle(.45) (8.4,\y) circle(.45);
}
\node[align=center] at (7,15.5) {CSI\\receivers};
\node[note,anchor=west] at (18.5,47.6) {Inference: every query};
\node[box,draw=black!55,fill=white,minimum width=12mm] (pre) at (25,40.5) {Bin +\\normalize};
\node[box,draw=fbblue,fill=fbblue!6,minimum width=18mm] (enc) at (43,40.5) {Frozen encoder\\+ basis $\phi$};
\node[box,draw=fborange,fill=fborange!8,minimum width=11mm] (ada) at (61.5,40.5) {Adapter\\$h$};
\node[box,draw=fbgreen,fill=fbgreen!7,minimum width=16mm] (rd) at (78,40.5) {OSCL\\classifier};
\draw[arr] (pre)--(enc); \draw[arr] (enc)--(ada); \draw[arr] (ada)--(rd);
\draw[arr] (rd.north)--(78,61);
\node[anchor=south,inner sep=.5pt] at (78,61) {activity $\hat y$, no domain ID};
\draw[arr] (12,44)--(15,44)|-(pre.west);
\draw[black!70,line width=.6pt] (12,34)--(15,34)--(15,44);
\node[note,anchor=west] at (43.7,31.9) {Update: domain $t$};
\node[box,draw=black!55,dashed,fill=white,minimum width=12mm] (lab) at (25,24) {Labeled\\$\mathcal D_t$};
\node[box,draw=fbgreen,fill=fbgreen!12,minimum width=18mm,line width=.9pt] (mom) at (43,24) {Moment store\\$n,\mu,\Sigma,B,K$};
\node[box,draw=fborange,fill=fborange!8,minimum width=11mm] (trn) at (61.5,24) {50 epochs\\on $\phi$};
\draw[arr,dashed] (12,24)--(lab.west);
\draw[arr,dashed] (lab.north)--(pre.south);
\draw[arr] (enc.south)--node[note,left,pos=.5]{once}(mom.north);
\draw[arr] (mom)--(trn);
\draw[arr] (trn.north)--(ada.south);
\draw[arr] (mom.south)--(43,14)-|(rd.south);
\node[note,anchor=south] at (66,14.2) {transported moments};
\node[note] at (25,14.6) {deleted after\\the update};
\node[anchor=west,align=left,inner sep=.5pt] at (18.5,6.2) {Per query: 8.5--88.6 GFLOPs, $>$99.8\% in the encoder\\Per domain: $+$7.0--19.1 MiB of aggregates};
\end{tikzpicture}
\caption{Deployment design. Queries need neither a domain identity nor stored samples. A new domain is encoded once, adds its moments to the store, and updates the adapter on cached basis features; its recordings are then deleted.}
\label{fig:deploy}
\end{figure}

Fig.~\ref{fig:deploy} maps the evaluated procedure onto an edge node such as a home gateway; Table~\ref{tab:deploy} gives its analytic costs, and Table~\ref{tab:cpu} measures them on a CPU. Each query is encoded and classified by the domain-marginalized OSCL classifier without its domain identity or stored samples. When a new acquisition condition is declared, a labeled calibration session supplies $\mathcal D_t$; the node encodes these arrivals once, adds their moments to the store, trains the adapter on the cached basis features, refits the classifier, and deletes the recordings, so raw CSI need not persist after an update.

\begin{table}[t]
\centering\scriptsize
\setlength{\tabcolsep}{2.6pt}
\caption{Computation and storage of WiMoT}
\label{tab:deploy}
\begin{tabular}{lccc}
\toprule
Quantity & 3DO & WiMANS & Widar3.0\\
\midrule
Inference per query (GFLOPs) & 8.55 & 44.31 & 88.60\\
Arrival encoding per domain (TFLOPs) & 20.5 & 16.7 & 39.5--39.9\\
Gradient updates per domain (TFLOPs)$^{\dagger}$ & 0.14 & 0.02 & 0.07\\
\midrule
Prepared recordings per domain (MiB) & 167.1 & 116.8 & 275.6--278.1\\
Aggregates added per domain (MiB) & 7.0 & 19.1 & 13.0\\
Class moments, full storage (MiB) & 18.04 & 54.11 & 108.21\\
\quad upper triangle (MiB) & 9.05 & 27.16 & 54.32\\
\quad rank 32 (MiB) & 1.16 & 3.49 & 6.97\\
\quad rank 32, paired $\Delta$AFA (pp) & $+1.07$ & $+0.00$ & $-0.49$\\
Label-moment constraints (MiB) & 3.03 & 3.06 & 9.14\\
\midrule
Replay arrays of the baselines (MiB) & 8.91 & 1.34 & 1.56\\
\bottomrule
\end{tabular}\\[2pt]
\parbox{\columnwidth}{\scriptsize $^{\dagger}$Dense estimate on cached basis features, excluding the statistical-head fits. Upper-triangle storage reproduces every stage-wise test accuracy of full storage. Rank-$r$ storage keeps the $r$ leading eigenpairs of each class covariance and one isotropic residual variance.}
\end{table}

\begin{table}[t]
\centering\scriptsize
\setlength{\tabcolsep}{2.6pt}
\caption{Measured cost of WiMoT on a laptop CPU (Apple M4 Pro, float32; four threads, one thread in parentheses)}
\label{tab:cpu}
\begin{tabular}{lccc}
\toprule
Quantity & 3DO & WiMANS & Widar3.0\\
\midrule
Query, encoder (ms) & 33 (61) & 158 (324) & 320 (724)\\
Query, classifier (ms) & $<0.1$ & $<0.3$ & $<0.5$\\
First domain, training (s) & 18.4 (27.0) & 3.7 (4.5) & 3.8 (5.2)\\
Later domain, update (s)$^{\ddagger}$ & 2.4--2.7 & 2.2--4.5 & 1.3--24.6\\
Arrival encoding per domain (s) & 96 (162) & 62 (141) & 144--145 (352--354)\\
\midrule
Peak memory, query encoding (GiB) & 0.79 & 2.38 & 4.47\\
Peak memory, updates (GiB) & 1.95 & 0.82 & 1.34\\
\bottomrule
\end{tabular}\\[2pt]
\parbox{\columnwidth}{\scriptsize Query times are medians over 20 queries; seed 42. $^{\ddagger}$Adapter training, moment transport, and classifier refit after the first domain; the encoding of the arrivals is listed separately.}
\end{table}

The frozen encoder accounts for more than 99.8\% of the inference computation, and the same encoder cost applies to every Mantis-input baseline; the basis, adapter, and classifier add 15.9--16.3 MFLOPs per query. The gradient steps of an update add 0.15--0.68\% to the encoding cost of the arrivals. The moment memory of a domain is 6.1--23.8 times smaller than its prepared recordings (Table~\ref{tab:deploy}). Full $512\times512$ float64 covariances make the moment store larger than the replay arrays of the baselines, but it can be compressed: storing the upper triangles halves it without changing any prediction, and rank-32 storage reduces it to 6.4\% (108.21 to 6.97 MiB on Widar3.0, the order of the replay arrays) at an AFA change of $-0.49$ to $+1.07$ pp. Supplementary Table~S6 reports ranks 8 and 128.

\emph{Measured cost.} To check the analytic costs, we ran the released implementation on a laptop CPU without a GPU, with four threads, a parallelism comparable to that of a quad-core edge board, and with one thread. Each stream was learned domain by domain with seed 42; the AFA of these runs equals that of the GPU runs with the same seed on WiMANS and Widar3.0 and differs from it by at most 0.08 pp on 3DO. Since every channel is interpolated to 512 samples before the encoder, its cost depends only on the input shape, and the encoder was timed on windows of the prepared shape. Table~\ref{tab:cpu} agrees with the analytic costs. A query takes 33--321 ms with four threads (61--724 ms with one), of which the classifier contributes less than 0.5 ms. An update after the first domain takes 1.3--24.6 s, whereas encoding the arrivals of a domain once takes 62--145 s, so the encoder also dominates the cost of an update. On Widar3.0 the update time grows with the number of stored domains, from 1.3 s at the second user to 24.6 s at the ninth, because the historical loss and the classifier refit run over the moments of all domains. Encoding a Widar3.0 query passes its 540 channels through the encoder in one batch and needs 4.47 GiB at peak; encoding the channels in smaller groups would lower this. The laptop is faster than typical edge hardware, so an embedded node would take longer.

\section{Discussion}\label{sec:discussion}
\subsection{What the Stored Moments Replace}
Replay keeps a few hundred individuals; WiMoT keeps the first two moments of all individuals of each domain and class. Three results indicate that these moments carry what a small replay cache would add: removing them lowers AFA by 5.31--47.10 pp, adding the replay cache of the baselines to WiMoT changes AFA by $+0.17$, $-0.62$, and $-0.67$ pp (Table~\ref{tab:ablation}), and the strongest replay baselines forget 2.78--10.25 pp where WiMoT forgets at most 0.58 pp. WiMoT also exceeds a classic network trained jointly on all domains on 3DO and Widar3.0 (94.02 and 95.16\% against 90.64 and 94.84\%; 89.61 against 91.46\% on WiMANS; Supplementary Table~S8), without revisiting any earlier recording. This reference trains a cross-entropy network on the same input rather than the basis and classifier of WiMoT, so it is not an upper bound for WiMoT. The moments describe classes only by means and covariances, so structure beyond the second moment is lost once the arrivals are deleted; this may matter most when few records per class and domain are available, as on WiMANS.

\subsection{Privacy}
A replay method keeps raw CSI recordings of individual people; WiMoT keeps, for each domain and class, only a count, a mean, and a covariance of frozen-encoder features, and deletes the recordings and per-sample features after each update. This reduces what a stored model exposes in three ways. First, each moment aggregates at least 42 recordings, and the map from recordings to their first two moments is many-to-one: every set of recordings with the same mean and covariance produces the same memory, so no individual recording can be read back or replayed, unlike the cache of a replay method. Second, the moments describe the feature distribution of a class in the 512-dimensional basis, not any recorded CSI sequence, so recovering a recording would require inverting both the aggregation and the frozen encoder. Third, keeping only the statistics that the classifier needs follows the principle of data minimization, the same argument that analytic continual learning makes for storing correlation matrices instead of samples\cite{acil2022}. The protection is qualitative: aggregates of small groups can still reveal information about them, and a formal guarantee would require, for example, noise calibrated to the moments.

\subsection{Limitations}\label{sec:scope}
The differences between methods are smallest on WiMANS, where the accuracy advantage of WiMoT depends on the order, the label budget, and the encoder: in three of the six domain orders MAS or SI+ER is higher in AFA by 0.93--1.95 pp (Supplementary Table~S1); with 25\% and 10\% of the later-domain labels MAS is higher by $1.23\pm2.64$ and $1.44\pm1.46$ pp, within one seed SD; and with the five other foundation encoders the strongest baseline is higher by 2.67--5.04 pp. Three properties of this stream explain why the design gains least there. First, the source of the lead is weak. The lead of WiMoT comes from retention (Section~\ref{sec:rq2}), but WiMANS has only three domains, and in each of its 14 settings the strongest baseline forgets at most 2.93 pp, against 2.85--7.92 pp on 3DO and 3.80--21.92 pp on Widar3.0 (Supplementary Table~S7); with little forgetting to remove, the plasticity cost of the affine adapter (New accuracy 87.50 against 90.12\% for MAS+ER; Table~\ref{tab:balance}) is not compensated. Second, the moments rest on few records. Each WiMANS class and domain provides 42 training records for 512-dimensional moments, so every class covariance has rank at most 41 and the classifier relies on shrinkage, whereas 3DO provides 800 records and Widar3.0 about 75; reducing the labels lowers this number further, and with 25\% and 10\% of the labels the advantage disappears. Third, recognition on WiMANS depends entirely on the frozen encoder. The amplitude descriptor carries little class information there (at most 32.51\% AFA for any method; Supplementary Table~S8), and after the first domain WiMoT can change the encoder features only through an affine adapter, whereas the baselines retrain a nonlinear network on every domain, so a weaker encoder affects WiMoT more. WiMoT is therefore less suited to short streams on which replay baselines already forget little, to domains with few labeled records per class relative to the basis dimension, and to encoders whose frozen features separate the classes poorly. On 3DO the common input is not the most informative one: with the amplitude descriptor alone, WiMoT reaches 99.02\% instead of 94.02\%, and it remains 1.16 pp above SI+ER given the same input (97.86\%, every seed; Supplementary Table~S8), consistent with the amplitude dependence of 3DO found in Section~\ref{sec:robustness}.

The moment memory grows with the number of domains and classes, latency was measured on a laptop CPU rather than on an embedded node, and energy was not measured. The recent and Wi-Fi baselines are adapted to a frozen-feature interface, and the Widar3.0 partitions separate repetitions rather than recording sessions (Supplementary Section~S6).

\section{Conclusion}
We proposed WiMoT, a replay-free continual learning algorithm for Wi-Fi sensing that stores class means and covariances of earlier domains instead of samples. Restricting later updates to an affine adapter lets the same map transport the stored moments, which keeps the memory valid (without transport, AFA falls by 3.4--31.0 pp); SupCon supplies plasticity, a closed-form historical loss keeps the representation stable, and the OSCL classifier is refitted from the transported moments of all domains. On three CSI streams, WiMoT forgets at most 0.58 pp and has the highest average final accuracy on 3DO and Widar3.0, 3.77 and 13.11 pp above the strongest of twenty-one baselines, and the lead holds across domain orders, rooms, label budgets, and frozen encoders. On WiMANS it matches the strongest replay baseline in accuracy while eliminating its forgetting. Unannounced domain shifts, unlabeled arrivals, and on-device measurements are the next steps.

\section*{Data and Code Availability}
The source code, experiment and reproduction scripts, partition indices, and the configuration of every reported table are publicly available under the MIT License at \url{https://github.com/Newsteinwell/WiMoT}.
All three datasets are public: 3DO through Zenodo, DOI 10.5281/zenodo.10925351\cite{dataset3do2024}; WiMANS from its authors\cite{wimans2024}; and Widar3.0 from its authors\cite{widar2019}. Supplementary Section~S6 lists the SHA-256 digests of the Widar3.0 archives used.

\end{document}

\title{Supplementary Material for ``WiMoT: Replay-Free Continual Wi-Fi Sensing via Moment Transport''}
\author{Chunchun~Yang,
Yang~Cao,
and~Zhitang~Chen}
\markboth{Supplementary Material}{Yang \MakeLowercase{\textit{et al.}}: Supplementary Material for WiMoT}
\maketitle

\noindent Section, table, figure, and reference numbers without the prefix ``S'' refer to the main paper. Sections~\ref{sec:baselines}--\ref{sec:readout} give the implementation details of the baselines and of WiMoT. Section~\ref{sec:extended} reports the full results of the robustness studies of Section~5.4, Section~\ref{sec:inputsource} analyzes the input features, and Section~\ref{sec:widarsupp} documents the Widar3.0 protocol. Sections~\ref{sec:balance}--\ref{sec:sweep} complete the ablation and the sensitivity analysis of Sections~5.3 and 5.4.

\section{Baseline Implementation Details}\label{sec:baselines}
\subsection{Classic Baselines}
The classic comparison comprises twelve configurations: fine-tuning, ER, EWC, SI, MAS, LwF, DER++, and A-GEM, with separate native and ER-augmented versions of EWC, SI, MAS, and LwF. Each receives a 9728-dimensional input: the 9216-dimensional grouped means and standard deviations of frozen MantisV2 channel embeddings and the 512-dimensional amplitude descriptor, normalized as in WiMoT (Section~\ref{sec:inputsource}). A separately initialized $9728\to512\to128\to C$ MLP, with batch normalization and ReLU after each hidden layer, is trained with cross-entropy (CE) on the first domain, and all its parameters remain trainable during later updates. No baseline uses the semantic stem, amplitude branch, PCA shortcut, contrastive initialization, projector, or OSCL classifier of WiMoT; only the frozen foundation encoder and the recording partitions are shared. The main comparison therefore evaluates complete learning procedures; Section~\ref{sec:inputsource} examines how much the results depend on the input of the baselines.

Fine-tuning uses current CE. EWC uses true-label per-example empirical diagonal Fisher importance and a parameter anchor for each earlier domain. MAS uses absolute gradients of the squared output norm; both estimates use up to 256 class-balanced current observations, without extra historical visits. SI accumulates importance along the optimization path with stabilizer 0.1, including first-domain training. LwF distills the previous-stage teacher on current inputs with temperature 2. EWC+ER, SI+ER, and MAS+ER minimize the mean CE over mixed current/cached batches plus their original parameter-importance penalties. LwF+ER uses the same mixed CE but applies teacher distillation only to current examples. Importance estimation for EWC and MAS remains current-only; SI accumulates importance along the mixed optimization path. The +ER variants are thus our replay-augmented versions of originally replay-free methods, whereas ER, DER++, and A-GEM already use replay and retain their own losses and gradient rules.

ER trains on current and retained labels. DER++ adds stored-logit squared error and replay-label CE, each weighted by 0.5; acquisition logits are not refreshed. A-GEM constrains a conflicting current gradient using a retained-reference gradient. Reference records are partitioned across current minibatches, giving one visit per retained record per epoch. Because the projected gradient is applied with AdamW, the original guarantee that the reference loss does not increase holds only for plain gradient steps.

The ER-family baselines retain the recording identities of the common class-priority cache (Section~\ref{sec:cache}), with capacities 240/36/42 and one old-individual visit per epoch. Within the classic comparison, batch size, optimizer, learning rate, epoch schedule, and domain order are common. For EWC, SI, MAS, and LwF, native and +ER coefficients are selected separately from three prespecified candidates using the equally weighted mean validation AFA of 3DO and WiMANS at the first training seed. The selected native values are 1, 10, 1, and 0.5; the +ER values are $10^4$, 10, 0.1, and 0.5 (EWC, SI, MAS, LwF), selected under the 9728-dimensional input. Each coefficient is shared across datasets and fixed for the other seeds, domain orders, and the Widar3.0 streams. The settings of WiMoT, including the training-only selection of its score family (Section~\ref{sec:readout}), were likewise fixed before any test evaluation.

\subsection{Recent Mechanisms with Aligned Historical Individuals}
We evaluate CN-DPM (ICLR 2020), DUCT (CVPR 2025), and DualCP with its VanillaCP variant (AAAI 2025). Their method cores are adapted independently to the Mantis features; their image benchmarks are not reproduced. Each reads the 9728-dimensional input without any parameter of WiMoT.

The replay variants align both the recording identities and the repeated visits. All use the common class-priority reservoir of Section~\ref{sec:cache}, refreshed after each domain and held fixed during the next update. For each +ER update, current and cached examples are pooled and shuffled without replacement into batches of 128 for 50 epochs, giving exactly 50 optimization visits per retained example. Losses retain sample-proportional weighting without an additional replay coefficient. Retained-ID files are checked against the reservoir records at every update. The alignment fixes which old individuals are used and how often; the methods still differ in feature payload, model state, aggregate memory, and computation per current sample.

CN-DPM retains Gaussian VAE density and discriminative experts, Dirichlet-process usage priors, novelty-triggered expert creation, and density-prior marginalization without query domain identity. Its aligned variant replaces cross-domain novelty retention with the common reservoir rather than adding a second persistent historical cache. Temporary novelty STM contains current-domain rows only and is cleared at the boundary. The first domain retains native one-pass wake and 8000/2000 generative/discriminative sleep steps with batches of 50. Later domains use 50 mixed wake passes; historical rows update existing experts and cannot enter the repeatedly sampled sleep STM. The common log-concentration selected from $\{-10,0,10\}$ was $-10$.

DUCT uses a $9728\to512\to128$ CE-trained first-domain MLP as its task-vector reference, with frozen Mantis weights. Task-vector and semantic optimal-transport head consolidation are retained. Native training uses 100/50 epochs and five additional current-only head-refit epochs; head/backbone learning rates are 0.1/0.001, transport regularization is 0.01, and historical-head mixing is 0.25. DUCT+ER updates the backbone and expanded head using current-slot CE for current examples and full-head CE for cached examples, whose targets are their acquired training-domain/class slots. Task-vector merging and transport consolidation are unchanged, and the five head-refit epochs use current examples only, preserving the 50 historical visits. Prediction maximizes over all slots and maps to activity labels without a supplied domain identity. The shared merge coefficient selected from $\{0.1,0.25,0.5\}$ was 0.5.

DualCP constructs fixed QR/ETF concept prototypes using first-domain class-mean guidance. Coarse/fine calibrators have 512 hidden units and 9728-dimensional outputs, optimized with dual dot-regression ($\alpha=0.5$), SGD at learning rate 0.1, and 100/50 epochs. DualCP+ER and VanillaCP+ER train each new-domain calibrator on the mixed pool using the original dot-regression losses and fixed concept prototypes for both current and cached activity labels; CE is not substituted. Earlier calibrators remain frozen. Domain means are computed from current rows alone in uncentered variance-scaled Mantis coordinates, so replay does not contaminate routing anchors. DualCP checks genuine two-level groups on every dataset; VanillaCP uses singleton groups and threshold 1. No query domain identity or extra CLIP encoder is supplied.

The grouping threshold of DualCP required one deviation from the protocol. With the 9728-dimensional input, the uncentered amplitude block makes the first-domain class means highly cosine-similar, so the prespecified candidates 0.1 and 0.5 produce a single group (``Not a dual hierarchy'') on 3DO (both candidates) and on WiMANS (0.1), and DualCP cannot be run. We declared these candidates ineligible and added the default of the implementation, 0.85, using the validation runs alone and before any test evaluation with this value. As the only eligible candidate, it was fixed for DualCP and DualCP+ER (mean validation AFA 62.29 and 69.12\%). Both methods remain far below the strongest baseline on every stream (43.93 and 57.00\% AFA on WiMANS), so this deviation does not affect the comparison with the strongest baseline.

\subsection{Wi-Fi Continual-Learning Methods}
ConSense (AAAI 2025) and EdgeCL (IEEE TWC 2026) were proposed for CSI recognition. We adapt their published method cores to the common interface; their CSI preprocessing pipelines and datasets are not reproduced. Both read the 9728-dimensional input as 19 ordered 512-dimensional tokens, the 18 group mean/SD tokens of the Mantis output and the amplitude descriptor, each rescaled to unit coordinate variance after first-domain normalization; the Gaussian range encoding of ConSense spans 19 positions. Both use Adam at learning rate $10^{-3}$ without weight decay, batch size 128, and 100/50 epochs, matching the common schedule rather than the original iteration counts, and the random generator that drives dropout is reset to a fixed value before every trajectory. Neither receives a domain identity at inference.

\emph{ConSense.} The network applies Gaussian range positional encoding with ten learnable Gaussians (initial width 8), one eight-head self-attention layer with a residual connection and layer normalization, token-mean pooling, and an MLP with widths $512\to512\to128\to C$ and ReLU. All parameters are trained on the first domain. Afterwards the attention weights, positional encoding, and normalization are frozen. Each later domain adds a parallel adapter $\tanh(XW_{\mathrm{down}})W_{\mathrm{up}}$ with bottleneck 64 that generates key and value prefixes; earlier prefixes remain frozen and all prefixes are concatenated at inference. Before training on a domain, the mean activation of every neuron in each MLP linear layer, including the output layer, is computed on the current training data and compared with its value stored at the end of the previous domain. Neurons whose change does not exceed $\epsilon$ receive zero gradients. Because the class set is fixed, the output layer is shared and masked like the hidden layers. Only per-neuron mean activations persist between domains; no individual is stored. The threshold was selected from $\{0.01,0.1,1\}$; the validation AFA averaged over 3DO and WiMANS was 80.88, 80.33, and 79.94\%, so $\epsilon=0.01$ was selected.

\emph{EdgeCL.} A two-layer per-token MLP encoder (widths 128 and 64, ReLU) with a learned positional embedding replaces the original CSI timestamp encoder. It is followed by two transformer encoder layers with eight heads, feed-forward width 128, and dropout 0.1, and the first token is mapped to the class logits. After each domain, exemplars are selected per class in the 64-dimensional encoder space: 90\% of the quota by k-means++ clustering, taking the sample nearest each centroid, and the remainder by herding toward the class mean. Each exemplar stores the distilled label $\mathrm{softmax}(z/\eta)$ with $\eta=2$, computed by the model at the end of its domain. During later updates the replay term is the cross-entropy between the distilled label and $\mathrm{softmax}(z/\eta)$, averaged with the current CE over the mixed batch. Every step applies the first-order sharpness-aware update: the gradient is evaluated at $\theta+\epsilon g/\|g\|_2$ and applied at $\theta$. The original method keeps a fixed number of exemplars per class and domain; here the total row budget equals that of all other replay methods, split equally over classes and retained domains (on 3DO, 80 per class after the first domain and 40 per class and domain after the second). Older domains are reduced by applying the same hybrid rule to their retained rows in the current encoder space. EdgeCL therefore matches the row count and visit bound of the reservoir methods but selects its own identities. The deviation radius was selected from $\{0.01,0.03,0.1\}$; the validation AFA averaged over 3DO and WiMANS was 85.73, 86.16, and 84.98\%, so the original default, $\epsilon=0.03$, was selected.

\section{Model Details, Replay Cache, and Memory}\label{sec:cache}
\subsection{Feature Normalization}
First-domain semantic normalization uses coordinate means and sample standard deviations, followed by block scaling. The amplitude block uses an uncentered first-domain RMS scale. The fixed random amplitude stem has an intermediate ReLU and a first-domain output calibration. The 256-dimensional PCA shortcut is fitted only to first-domain semantic training features and normalized by a fixed first-domain block RMS. Preprocessing does not change the model graph: WiMANS and Widar3.0 use per-record input standardization, whereas 3DO preserves the amplitude scale.

\subsection{Common Replay Cache of the Baselines}
Initially each class retains $m_1=\lfloor N_1/(10C)\rfloor$ individuals. The later per-class capacity is
\begin{equation}
 m_t=\min\{m_{t-1},\lfloor N_t/(9C)\rfloor\},
\end{equation}
and the total capacity $M_t=Cm_t$ is shared across all observed domains. After training on domain $t$, the cache is refreshed using only $\mathcal B_{t-1}\cup\mathcal D_t$. Each example receives a fixed priority from a SHA-256 hash of the training seed, acquisition-domain index, and within-domain training index. Within each activity class, the $m_t$ candidates with the smallest priorities are kept and the rest discarded; retained examples preserve their priorities across updates. Selection is class-balanced, with no separate quota for each domain, and evicted examples cannot be recovered. The resulting $\mathcal B_t$ is held fixed throughout the next domain's training. Each later-domain epoch traverses current arrivals and retained individuals once, giving the replay share
\begin{equation}
 r_t=\frac{|\mathcal B_{t-1}|}{N_t+|\mathcal B_{t-1}|}.
\end{equation}
Because $|\mathcal B_{t-1}|\le Cm_{t-1}\le N_{t-1}/9$, the bound $r_t\le0.1$ holds whenever $N_t\ge N_{t-1}$. In all evaluated streams $m_t=m_1$ and $Cm_1\le N_t/10$, so $r_t\le1/11$; the implementation additionally asserts $9|\mathcal B_{t-1}|\le N_t$ before each update. The resulting capacities are 240 windows on 3DO (80 per class), 36 records on WiMANS (four per class), and 42 instances on Widar3.0 (seven per class), with replay shares of $240/(2400+240)=9.09\%$, $36/(378+36)=8.70\%$, and 8.54--8.61\%. The cache stores input features and labels and excludes validation and test data; A-GEM uses the cached examples for reference gradients, and CN-DPM uses them in mixed wake passes.

\subsection{Memory Accounting}
WiMoT stores no individuals, but it does store aggregate moments and model state. Baseline replay arrays of the 9728-dimensional input at the final stage occupy 8.91, 1.34, and 1.56 MiB on 3DO, WiMANS, and Widar3.0. EdgeCL stores the same number of 9728-dimensional rows plus a $C$-dimensional distilled label per row. The class moments of WiMoT occupy 18.04, 54.11, and 108.21 MiB, and the historical second and label-cross moments 3.03, 3.06, and 9.14 MiB; the $512\times512$ class covariances in float64 account for 99.8\% of the class-moment payload. DER++ adds 2880/1296 bytes of retained logits on 3DO/WiMANS. EWC, SI, and MAS retain parameter importance and anchor arrays; LwF retains a teacher network; ConSense retains its frozen prefixes and per-neuron mean activations. These method-specific states are distinct from individual replay.

The classic baselines update approximately 5.05 million MLP parameters throughout the stream. CN-DPM additionally retains its VAE experts; DUCT retains its initialization and cumulative task vector; DualCP and VanillaCP retain one calibrator per observed domain, fixed prototypes, and domain means. WiMoT updates 82,563/83,337/82,950 parameters on 3DO/WiMANS/Widar3.0 after freezing the first-domain basis: the adapter, the projector, and the linear map $W$ of the historical loss. Stored non-backbone model tensors occupy approximately 19.26 MiB for a classic baseline and 30.08 MiB for WiMoT, including its frozen stems and basis state. These figures exclude the shared frozen backbone, optimizer state, fitted statistical heads, and archival evaluation outputs. Because the aggregates grow with the number of domains and classes, the total memory of WiMoT is not constant.

The per-query computation of Table~9 was obtained with the PyTorch FLOP counter: one CSI channel costs 160.7 million floating-point operations for the evaluated tokenizer and first three attention layers, and the attention products, which the counter does not register on CPU, add 3.3 million. Because every channel is encoded separately, the query cost grows linearly with the channel count (52, 270, and 540 channels). The classifier estimate assumes dense 128-dimensional Gaussian evaluations for every observed domain. After the ninth Widar3.0 user, the encoder weights, the stored model tensors, and the aggregates total about 156 MiB, of which 117 MiB are aggregates.

\subsection{Loss Construction}
\emph{SupCon.} With $I$ the anchors that have positives and $P(i)$ the other same-label minibatch observations, the SupCon loss on the normalized projections $u_i$ with temperature $\tau=0.07$ is
\begin{equation}
 \mathcal L_{\mathrm{SC}}=-\frac1{|I|}\sum_{i\in I}\frac1{|P(i)|}\sum_{p\in P(i)}
 \log\frac{\exp(u_i^\top u_p/\tau)}{\sum_{a\ne i}\exp(u_i^\top u_a/\tau)}.
\end{equation}
\emph{Historical loss.} With $W_t$ and $d_t$ the weight and bias of the linear map of the historical loss, the class-centered composition of adapter and map is
\begin{equation}
 F_t=\Bigl(I_C-\tfrac1C\mathbf 1\mathbf 1^{\top}\Bigr)\bigl[W_tA_t,\ W_tb_t+d_t\bigr]\in\mathbb R^{C\times513}.
\end{equation}
The map receives no gradient in the first domain, and its output scale is fixed at the end of that domain as $a_1=\tr(F_1B_1F_1^{\top})^{1/2}$. The loss is the class-balanced squared error of $\widetilde F_t=F_t/a_1$ against the centered one-hot targets, evaluated from the stored moments $B_s$ and $K_s$ (Proposition~1); each term is clipped at zero for numerical safety.

\section{OSCL Classifier}\label{sec:readout}
The classifier of Section~4.5 follows OSCL (Ref.~[17] of the main paper). This section describes how it is fitted from the transported moments.

Pooled same-class covariance includes within-domain covariance and the between-domain variation of same-class means. A shared regularized inverse square root with shrinkage 0.5 toward its trace-scaled identity whitens all historical and current moments, class covariances are mixed with the identity at strength 0.25, and each target domain adds a within-task preconditioner with shrinkage 0.1.

For each target domain and activity pair, OSCL optimizes a $2t$-dimensional combining-label vector $\eta$. Index the domain/class cells of the pair by $k$, with sample proportions $\pi_k=n_k/n$, transformed means $\mu_k$, and covariances $\Sigma_k$, and let $Z$ have rows $\pi_k\mu_k^\top$. The score direction and its moments are
\begin{equation}
 \begin{aligned}
 w&=Z^\top\eta,\qquad m_j=\eta^\top Z\mu_j,\qquad v_j=\eta^\top Q_j\eta,\\
 Q_j&=Z\Sigma_jZ^\top+
 \operatorname{diag}_k\!\left(\frac{\pi_k\tr(\Sigma_k\Sigma_j)}{n}\right),
 \end{aligned}
\end{equation}
where the diagonal term retains the mean-estimation variance in the finite-sample Gaussian risk. The binary error of the target is estimated at the density-intersection threshold of the two score distributions, and L-BFGS-B minimizes its logarithm over $\eta\in[-10,10]^{2t}$ from fixed signed labels, for at most 200 iterations; nonfinite or worse-risk candidates revert to the initialization. Normalized pair directions form a bank $D_s$, whose full spectrum, including singular-value weights, is combined with a class-mean prior:
\begin{equation}
 G_s=0.5\frac{D_s^\top D_s}{\|D_s\|_F^2}
 +\frac{0.5}{L_s}U_s^\top U_s,\qquad R_s^\top R_s=G_s,
\end{equation}
where the rows of $U_s$ are the normalized nonzero transformed training means. The spectral weights are retained because replacing the bank by its orthogonal span can erase combining-label effects. Projected class covariances are pooled within each target and shrunk by 0.5 toward the trace-scaled identity.

The parameters of the statistical head are fitted with two-fold cross-fitting on the training set. Current training arrivals supply all query individuals; no historical individual is available. Held-out fold contributions are subtracted from historical sufficient moments before refitting the statistical head. Observed tasks receive equal weight, and observed classes within each task receive equal weight. Missing task/class groups are not filled; a singleton belongs to one fold only.

Cross-fitting selects the score family. Candidates comprise Gaussian scores and Student-$t$ scores with $\nu\in\{3,10,30\}$, each finite-$\nu$ case paired with Gaussian or Student-$t$ gating; the Student scale matrix is $(\nu-2)/\nu$ times the original covariance. The maximum held-out-fold log likelihood, with fixed-order tie breaking, determines the choice. All datasets use this estimator and candidate family, although the learned choices can differ, and the final head is refitted on the full training moments. Only the statistical head is cross-fitted: the encoder, the adapter, and the linear map $W$ of the historical loss are trained on the full training inputs, so the selection is not protected against optimism at the representation level. The prediction is the arg-max of the chosen OSCL posterior; $W$ is not used at inference. The selection changes only the score family and gate of the final classifier; the combining-label risk remains Gaussian.

\section{Extended Results}\label{sec:extended}
This section gives the setup and complete results of the studies summarized in Section~5.4 of the main paper, and the five-seed baseline records behind the ablation of Table~4.

\subsection{Additional Seeds and Reproducibility}
The ablation of Table~4 uses seeds 42--46. For it, seeds 45 and 46 were added under the original order for all classic baselines, ConSense, EdgeCL, and WiMoT. In this study the baselines read the 9216-dimensional Mantis block alone (Section~\ref{sec:mantisonly}), so their values are not combined with Table~2. Over the five seeds, WiMoT attains $93.98\pm0.15$\% and $89.57\pm0.55$\% AFA on 3DO and WiMANS, ConSense $76.96\pm1.18$\% and $82.16\pm2.09$\%, and EdgeCL $88.87\pm1.00$\% and $86.85\pm1.90$\%. Relative to SI+ER, WiMoT changes AFA by $+3.51\pm0.25$ pp on 3DO (positive in five of five seeds) and $+1.05\pm1.17$ pp on WiMANS (four of five).

Two checks bound the run-to-run variation. Re-executing the original order at seed 42 reproduces fine-tuning identically on both datasets. Training WiMoT with seeds 42--44 on a second machine reproduces the WiMANS AFA of every seed and changes the 3DO and Widar3.0 AFA by at most 0.09 and 0.07 pp per seed, owing to nondeterministic GPU kernels.

The reproducibility package (Python package \texttt{wimot}) contains the code, the partition manifests, the job lists and locked coefficients of every table, and the validation scores of every candidate value of each coefficient selection (24 selections, 132 runs). It also contains the per-seed test accuracy matrices of all reported trajectories. For WiMoT these are 120 trajectories of the five-seed ablation, whose full-method rows include the three-seed records of Table~2; 9 of the reproducibility check; 30 of the five further domain orders; 18 with reduced labels; 45 of the covariance-storage study; 15 of the further Widar3.0 orders and rooms; and 54 of the encoder study. For the baselines with the 9728-dimensional input they are the 198 trajectories of Table~2 (twenty-one configurations and the joint-training reference), 420 of the order study, 180 with reduced labels, together with the retained training indices of every seed, domain, and fraction, 75 of the further Widar3.0 orders and rooms, and 270 of the encoder study; the 56 baseline trajectories of the additional seeds use the 9216-dimensional input. Stage-wise test predictions are provided for 1337 baseline trajectories with the 9216-dimensional input (all except the joint-training reference, the reduced-label runs, and the MOMENT and randomly initialized MantisV2 runs), and the class scores behind them are an archive attachment.

\subsection{Domain Orders}
WiMoT, the twelve classic baselines, ConSense, and EdgeCL were repeated under the five remaining permutations of the 3DO and WiMANS streams, with seeds 42--44, the selected coefficients, and the unchanged training recipes; CN-DPM, DUCT, and DualCP were not repeated. Partitions and seeds are those of the original order, and test evaluation started only after all training trajectories had finished. Table~\ref{tab:perorder} lists the per-order AFA behind Table~7 and the order rows of Fig.~4 of the main paper.

\begin{table}[t]
\centering\scriptsize
\setlength{\tabcolsep}{2.6pt}
\caption{AFA (\%) for each domain order; three-run means. Domains are numbered by their position in the original order; best classic: highest of the twelve classic baselines in that order}
\label{tab:perorder}
\begin{tabular}{llcccl}
\toprule
Data & Order & WiMoT & ConSense & EdgeCL & Best classic\\
\midrule
3DO & 123 & 94.02 & 76.31 & 89.68 & 90.25 (SI+ER)\\
 & 132 & 94.03 & 80.21 & 86.05 & 88.99 (SI+ER)\\
 & 213 & 96.54 & 74.39 & 88.69 & 89.73 (EWC+ER)\\
 & 231 & 95.31 & 76.62 & 86.28 & 88.94 (SI+ER)\\
 & 312 & 90.64 & 78.19 & 83.88 & 89.51 (LwF+ER)\\
 & 321 & 91.75 & 78.06 & 83.99 & 89.44 (DER++)\\
\midrule
WiMANS & 123 & 89.61 & 83.85 & 85.49 & 89.20 (MAS+ER)\\
 & 132 & 88.58 & 81.58 & 86.52 & 89.61 (MAS)\\
 & 213 & 89.92 & 84.26 & 85.60 & 88.07 (MAS+ER)\\
 & 231 & 88.89 & 81.38 & 86.42 & 87.86 (DER++)\\
 & 312 & 87.45 & 81.89 & 85.49 & 89.40 (SI+ER)\\
 & 321 & 88.37 & 80.35 & 84.67 & 89.30 (SI+ER)\\
\bottomrule
\end{tabular}
\end{table}

On 3DO, WiMoT leads in every order, by 1.13--6.81 pp. On WiMANS the order decides which method is ahead. Numbering the WiMANS rooms by their original position (1, classroom; 2, meeting room; 3, empty room), WiMoT leads in orders 123, 213, and 231 by 0.41--1.85 pp, whereas MAS is higher by 1.03 pp in order 132 and SI+ER by 1.95 and 0.93 pp in orders 312 and 321. Averaged over orders, its lead over SI+ER, the strongest order-averaged baseline, is 0.34 pp, within the variation across orders, against 4.69 pp on 3DO. Its WiMANS AFR is the lowest of all methods in four of the six orders; in the other two a single baseline forgets less, by 0.16 and 0.30 pp.

Retention depends more on the order than accuracy does. The 3DO AFR of WiMoT ranges from 0.32 to 8.15 pp across orders, and LwF+ER, which forgets least on 3DO when averaged over orders (0.02 against 3.89 pp), has a 5.17 pp lower AFA. ConSense and EdgeCL remain below the strongest classic baseline in every order; the order-averaged 3DO AFA of EdgeCL is 3.25 pp below its value in the original order, whereas its WiMANS AFA varies little across orders (SD 0.69 pp).

\subsection{Widar3.0: Main Stream, Further User Orders, and Rooms}
Compared with 3DO and WiMANS, the Widar3.0 stream changes the dataset and task, triples the number of domains, and requires seven further representation updates after the second domain. All twenty-one baselines keep the recipes and coefficients selected on 3DO and WiMANS, which need not be optimal for this longer stream. The strongest baseline is LwF+ER, with $82.05\pm1.07$\% AFA and 10.25 pp AFR. The strongest recent mechanisms are EdgeCL ($75.01\pm1.20$\%) and CN-DPM ($74.84\pm2.22$\%), from which WiMoT differs by $+20.15\pm1.64$ and $+20.32\pm1.83$ pp. DualCP+ER, DualCP, VanillaCP+ER, and VanillaCP forget least among the baselines (AFR 3.08--5.72 pp) but reach only 65.33--73.14\% AFA. Several non-CE baselines already start lower: after the first user, CN-DPM and EdgeCL attain 85.8 and 86.0\%, and ConSense 89.3\%. DUCT and DUCT+ER start at 95.1\% but decline to 37.41 and 40.74\% AFA, below fine-tuning (70.02\%), and DUCT varies by 7.26 pp across seeds; ConSense also ends below fine-tuning (64.62\%). Table~\ref{tab:widardomain} lists the final accuracy on each user.

\begin{table}[t]
\centering\scriptsize
\setlength{\tabcolsep}{2.4pt}
\caption{Widar3.0 final test accuracy (\%) per user after the ninth stage: three-run means}
\label{tab:widardomain}
\begin{tabular}{lccccccccc}
\toprule
User & 5 & 10 & 11 & 12 & 13 & 14 & 15 & 16 & 17\\
\midrule
WiMoT & 96.9 & 96.9 & 95.3 & 93.8 & 98.0 & 97.8 & 96.0 & 95.8 & 86.0\\
LwF+ER & 80.0 & 75.3 & 81.1 & 82.9 & 83.6 & 86.9 & 82.0 & 90.9 & 75.8\\
SI+ER & 85.8 & 78.0 & 78.0 & 84.0 & 81.1 & 84.7 & 78.9 & 80.9 & 81.1\\
CN-DPM & 82.2 & 67.1 & 71.6 & 75.8 & 74.2 & 72.9 & 68.9 & 80.7 & 80.2\\
EdgeCL & 73.3 & 61.3 & 73.6 & 81.6 & 76.2 & 74.4 & 73.3 & 81.3 & 80.0\\
ConSense & 67.8 & 48.0 & 58.0 & 67.6 & 68.2 & 67.1 & 62.2 & 62.4 & 80.2\\
\bottomrule
\end{tabular}
\end{table}

To test whether the Widar3.0 result depends on the user order or on Room~1, three further orders of the nine Room-1 users and two streams recorded in other rooms were run. The orders are the reverse order and two permutations drawn with a fixed generator seed. The other rooms are all Widar3.0 recording sessions outside Room~1 whose archives contain the same six gesture codes: Room~2 (users 1 and 6) and Room~3 (users 3, 7, 8, and 9), each forming a user-incremental stream in ascending user order (Section~\ref{sec:widarsupp}). For each stream, WiMoT and five baselines (fine-tuning, LwF, SI+ER, DER++, and EdgeCL) were trained with seeds 42--44 on one machine, and testing started after all trajectories of that stream had finished training. Table~6 of the main paper summarizes the results; the per-seed and per-user details are as follows.

In the reverse order and the two random orders, WiMoT attains $94.47\pm0.45$, $94.27\pm0.38$, and $94.79\pm0.23$\% AFA and exceeds SI+ER, the strongest baseline in each order, by $+9.58\pm0.28$, $+11.14\pm0.15$, and $+14.52\pm0.26$ pp in every seed; its AFR stays at 1.11--1.42 pp, against 10.33--14.33 pp for SI+ER. The low final accuracy on the last user of the original order (Table~\ref{tab:widardomain}) is a property of user 17 rather than of recency: user 17 ends at 85.3\% when it comes first (reverse order), 87.1\% when third, and 86.7\% when last (random orders), while every other user ends at 93.6\% or higher. In Room~3, WiMoT exceeds DER++ in every seed. Room~2 has two domains and therefore a single update; there the per-seed differences from SI+ER are $0.00$, $+0.34$, and $+1.01$ pp, where one of the 149 test windows per user corresponds to 0.67 pp, and the AFR is $-0.45$ pp for WiMoT against 3.80 pp for SI+ER; DER++ ($89.71$\%) is 2.91 pp below WiMoT.

\subsection{Reduced Labels in Later Domains}
The first domain was kept fully labeled, and in every later domain $\max(2,\lceil f n_c\rceil)$ training instances of each class $c$ with $n_c$ available were retained, for $f=0.25$ and $f=0.10$; the remaining arrivals of that domain are not used. Per later domain and class, this leaves 200 and 80 of 800 windows on 3DO, 11 and 5 of 42 records on WiMANS, and 19 and 8 of 74--75 instances on Widar3.0. The subsets are class-balanced random draws from the recorded sessions, made independently for each seed, domain, and fraction; they emulate a smaller labeling effort within the same sessions, not shorter calibration sessions. WiMoT, eight classic configurations (fine-tuning, SI, MAS, LwF, ER, SI+ER, LwF+ER, and DER++), ConSense, and EdgeCL were trained with seeds 42--44 and the recipes and coefficients of Table~2; the CE baselines reuse their first-domain initializations from the full-label runs. Validation and test partitions are unchanged. The common replay cache keeps its full-label capacity of 240/36/42 individuals, so it constitutes 27--29\% of each later update at $f=0.25$ and 44--50\% at $f=0.10$; this favors the replay baselines relative to the full-label comparison. All trajectories finished training before test evaluation. Table~\ref{tab:fewlabel} reports the results.

\begin{table*}[t]
\centering\footnotesize
\setlength{\tabcolsep}{4pt}
\caption{AFA (\%) with reduced labels in every later domain (mean $\pm$ sample SD, three runs); the first domain is fully labeled, and the best value of each column is in bold}
\label{tab:fewlabel}
\begin{tabular}{lcccccc}
\toprule
Method & \multicolumn{2}{c}{3DO} & \multicolumn{2}{c}{WiMANS} & \multicolumn{2}{c}{Widar3.0}\\
\cmidrule(lr){2-3}\cmidrule(lr){4-5}\cmidrule(lr){6-7}
Later-domain label fraction & 25\% & 10\% & 25\% & 10\% & 25\% & 10\%\\
Labels per class and later domain & 200 & 80 & 11 & 5 & 19 & 8\\
\midrule
\multicolumn{7}{l}{\textit{Without replay}}\\
\textbf{WiMoT} & $\boldsymbol{93.24\pm0.73}$ & $\boldsymbol{92.04\pm0.21}$ & $85.39\pm0.94$ & $82.72\pm2.16$ & $\boldsymbol{85.56\pm1.02}$ & $\boldsymbol{75.75\pm1.15}$\\
Fine-tuning & $84.19\pm1.44$ & $83.11\pm1.40$ & $80.25\pm2.83$ & $75.62\pm0.93$ & $64.22\pm0.49$ & $57.80\pm1.79$\\
SI (2017) & $87.90\pm0.60$ & $83.94\pm1.88$ & $80.66\pm2.80$ & $75.93\pm0.82$ & $65.60\pm0.99$ & $59.48\pm1.58$\\
MAS (2018) & $87.47\pm0.96$ & $86.67\pm0.76$ & $\boldsymbol{86.63\pm1.70}$ & $\boldsymbol{84.16\pm1.39}$ & $60.79\pm0.59$ & $56.32\pm1.58$\\
LwF (2016) & $83.95\pm1.32$ & $82.31\pm0.70$ & $82.10\pm0.93$ & $80.04\pm3.51$ & $69.33\pm0.98$ & $60.59\pm3.56$\\
ConSense (2025) & $75.32\pm2.33$ & $73.47\pm2.17$ & $76.85\pm1.54$ & $73.66\pm1.46$ & $54.79\pm2.79$ & $53.16\pm1.34$\\
\midrule
\multicolumn{7}{l}{\textit{With replay; common cache 240 / 36 / 42 as in Table~2}}\\
ER (2019) & $90.01\pm0.71$ & $89.09\pm0.59$ & $80.56\pm0.82$ & $78.81\pm1.55$ & $73.23\pm0.75$ & $62.27\pm1.39$\\
SI+ER (2017) & $90.44\pm0.52$ & $89.64\pm0.33$ & $80.45\pm1.28$ & $79.53\pm1.89$ & $75.70\pm1.18$ & $63.80\pm1.45$\\
LwF+ER (2016) & $86.46\pm0.39$ & $83.31\pm1.29$ & $80.76\pm1.58$ & $80.45\pm1.81$ & $72.17\pm0.37$ & $63.90\pm1.12$\\
DER++ (2020) & $87.79\pm0.76$ & $87.36\pm0.51$ & $80.45\pm2.36$ & $79.73\pm1.28$ & $73.36\pm2.19$ & $64.47\pm1.39$\\
EdgeCL (2026) & $85.69\pm2.71$ & $84.44\pm1.79$ & $81.69\pm0.36$ & $81.58\pm1.08$ & $69.41\pm1.78$ & $65.19\pm2.24$\\
\bottomrule
\end{tabular}
\end{table*}

On 3DO and Widar3.0, WiMoT exceeds the strongest baseline at both fractions in every seed. On 3DO, SI+ER is the strongest baseline, and the paired AFA difference is $+2.81\pm1.24$ pp at 25\% and $+2.40\pm0.33$ pp at 10\% of the labels, with 3.72 and 2.76 pp lower AFR. On Widar3.0, the strongest baseline is SI+ER at 25\% and EdgeCL at 10\% of the labels, and the differences are $+9.85\pm2.12$ and $+10.57\pm1.10$ pp, with 7.72 and 6.00 pp lower AFR. On WiMANS, WiMoT has the lowest AFR of all methods at both fractions, 2.31 and 3.09 pp lower than that of MAS, the strongest baseline in AFA; in AFA, MAS is higher by $1.23\pm2.64$ pp at 25\% (WiMoT ahead in one of three seeds) and $1.44\pm1.46$ pp at 10\% (behind in every seed). With 11 or 5 labeled records per class, the class moments of each new WiMANS domain rest on very few samples, which is a plausible cause that was not tested separately.

Fewer labels cost WiMoT 0.78 and 1.98 pp of AFA on 3DO, 4.22 and 6.89 pp on WiMANS, and 9.60 and 19.41 pp on Widar3.0, while the strongest reduced-label baseline changes by $+0.19$ and $-0.61$, $-1.54$ and $-4.01$, and $-5.68$ and $-9.83$ pp relative to its own full-label AFA. The lead therefore narrows from 3.77 pp with full labels to 2.81 and 2.40 pp on 3DO and from 13.11 pp to 9.85 and 10.57 pp on Widar3.0, but remains positive in every seed.

\subsection{Other Frozen Encoders}
Each encoder is applied channelwise to the same 512-sample channel series. Its per-channel embedding is the mean over the context-patch tokens of its last encoder layer; for TSPulse-r1, whose tokens have 24 coordinates, it is the flattened last layer (the model's full-embedding mode), and the randomly initialized MantisV2 uses the layer and token rule of the pretrained one. Embeddings wider than 512 coordinates are reduced to 512 by a PCA fitted only on the first-domain training channels of each dataset, which retains 97.9--99.9\% of their variance. The nine-group pooling, the amplitude descriptor, recipes, coefficients, and seeds are unchanged; none was retuned for the new encoders. For each encoder, WiMoT and five baselines (fine-tuning, LwF, SI+ER, DER++, and EdgeCL), all reading the features of that encoder together with the amplitude descriptor, were trained with seeds 42--44 on one machine and tested after all of its trajectories had finished training. Table~\ref{tab:backbone} extends Table~8 of the main paper by the identity of the strongest baseline and the number of seeds in which WiMoT is ahead, and Table~\ref{tab:encoders} lists the encoders.

\begin{table*}[t]
\centering\scriptsize
\setlength{\tabcolsep}{2pt}
\caption{AFA (\%) with other frozen encoders (mean $\pm$ sample SD, three runs), the strongest of five baselines (fine-tuning, LwF, SI+ER, DER++, EdgeCL), and the paired AFA difference of WiMoT from it (pp; number of seeds with higher AFA). Bold: WiMoT is higher than the strongest baseline}
\label{tab:backbone}
\begin{tabular}{lclclcl}
\toprule
& \multicolumn{2}{c}{3DO} & \multicolumn{2}{c}{WiMANS} & \multicolumn{2}{c}{Widar3.0}\\
\cmidrule(lr){2-3}\cmidrule(lr){4-5}\cmidrule(lr){6-7}
Encoder & WiMoT & Strongest baseline & WiMoT & Strongest baseline & WiMoT & Strongest baseline\\
\midrule
MantisV2 (pretrained) & $\boldsymbol{94.02\pm0.07}$ & $90.25$ SI+ER ($+3.77$, 3/3) & $\boldsymbol{89.61\pm0.47}$ & $88.58$ LwF ($+1.03$, 3/3) & $\boldsymbol{95.16\pm1.23}$ & $81.38$ SI+ER ($+13.78$, 3/3)\\
MantisV2 (random init.) & $\boldsymbol{96.03\pm0.46}$ & $88.14$ EdgeCL ($+7.89$, 3/3) & $\boldsymbol{78.29\pm1.39}$ & $76.85$ EdgeCL ($+1.44$, 3/3) & $\boldsymbol{87.33\pm1.08}$ & $56.35$ SI+ER ($+30.99$, 3/3)\\
\midrule
MOMENT-1-large & $\boldsymbol{94.82\pm1.22}$ & $87.53$ DER++ ($+7.30$, 3/3) & $78.50\pm0.99$ & $83.54$ SI+ER ($-5.04$, 0/3) & $\boldsymbol{86.20\pm0.73}$ & $69.58$ LwF ($+16.62$, 3/3)\\
TSPulse-r1 & $\boldsymbol{95.56\pm0.71}$ & $86.01$ DER++ ($+9.56$, 3/3) & $68.42\pm0.78$ & $72.94$ DER++ ($-4.53$, 0/3) & $\boldsymbol{84.69\pm0.30}$ & $58.30$ LwF ($+26.40$, 3/3)\\
Chronos-2 & $\boldsymbol{93.88\pm0.39}$ & $87.57$ DER++ ($+6.31$, 3/3) & $86.11\pm1.35$ & $89.51$ SI+ER ($-3.40$, 0/3) & $\boldsymbol{88.12\pm0.31}$ & $74.42$ LwF ($+13.70$, 3/3)\\
TimesFM 2.5 & $\boldsymbol{91.18\pm1.29}$ & $87.03$ DER++ ($+4.15$, 3/3) & $81.58\pm0.99$ & $84.26$ SI+ER ($-2.67$, 0/3) & $\boldsymbol{87.36\pm0.24}$ & $69.46$ DER++ ($+17.90$, 3/3)\\
TimesFM 3.0 & $\boldsymbol{93.75\pm0.03}$ & $87.06$ DER++ ($+6.69$, 3/3) & $80.86\pm2.16$ & $85.60$ SI+ER ($-4.73$, 0/3) & $\boldsymbol{87.43\pm2.03}$ & $68.81$ SI+ER ($+18.62$, 3/3)\\
\bottomrule
\end{tabular}\\[2pt]
\parbox{\textwidth}{\scriptsize The pretrained-MantisV2 row uses the WiMoT entries and the five baselines of Table~2; the strongest of the five can differ from the strongest of all twenty-one baselines. Coefficients of all methods were selected with MantisV2 features. Results can differ across GPUs and software stacks: between the two stacks used here, 33 of the 54 WiMoT encoder trajectories are identical and 21 differ (random-initialized MantisV2 9, TimesFM 3.0 6, Chronos-2 3, MOMENT 3) by $-0.31$ to $+0.93$ pp AFA per seed, without a change in the sign of any lead or in the strongest baseline (largest change of a mean lead: 0.62 pp, WiMANS with random-initialized MantisV2, $+1.44$ to $+2.06$ pp); the six Room-2 and Room-3 trajectories are identical.}
\end{table*}

\begin{table*}[t]
\centering\footnotesize
\setlength{\tabcolsep}{5pt}
\caption{Frozen encoders: parameters of the loaded checkpoint (millions), per-channel embedding width $D$, repository revision, weight digest (first 12 hex digits), and variance retained by the 512-coordinate PCA (3DO / WiMANS / Widar3.0)}
\label{tab:encoders}
\begin{tabular}{lrrlll}
\toprule
Encoder & Param. & $D$ & Revision & Weights & PCA variance (\%)\\
\midrule
MOMENT-1-large & 341.2 & 1024 & \texttt{ca58581b} & \texttt{a56928052ac6} & 99.11 / 98.75 / 98.87\\
TSPulse-r1 & 1.08 & 3312 & \texttt{2e64fcdc} & \texttt{57fa03b67d14} & 97.95 / 99.41 / 99.26\\
Chronos-2 & 119.5 & 768 & \texttt{29ec3766} & \texttt{ddcda3c7508b} & 99.86 / 99.93 / 99.90\\
TimesFM 2.5 (200M) & 231.3 & 1280 & \texttt{1d952420} & \texttt{2f776efe6245} & 99.17 / 99.45 / 99.38\\
TimesFM 3.0 & 330.7 & 1280 & \texttt{43046b85} & \texttt{a7592b0a8432} & 99.00 / 99.37 / 99.32\\
MantisV2, random init. & 4.19 & 512 & -- & -- & no PCA\\
\bottomrule
\end{tabular}
\end{table*}

Beyond the leads reported in the main paper, the per-encoder results show three things. First, replacing MantisV2 hurts the baselines more than WiMoT. The strongest baseline loses 2.11--4.24 pp on 3DO (to 86.01--88.14\%) and 6.96--25.03 pp on Widar3.0 (to 56.35--74.42\%), whereas WiMoT changes by $-2.84$ to $+2.01$ pp on 3DO (to 91.18--96.03\%) and by $-10.47$ to $-7.04$ pp on Widar3.0 (to 84.69--88.12\%). The lead is therefore larger with each of the six other encoders than with pretrained MantisV2 on 3DO, and with five of the six on Widar3.0 (Chronos-2: 13.70 against 13.78 pp). The same pattern appears in forgetting: with the six other encoders, the strongest baseline has 4.11--7.92 pp AFR on 3DO and 8.31--21.92 pp on Widar3.0, against $-0.53$ to 0.82 and 0.81--1.83 pp for WiMoT.

Second, on WiMANS the strongest baseline with the five other foundation models is SI+ER (MOMENT-1-large, Chronos-2, TimesFM 2.5, and TimesFM 3.0) or DER++ (TSPulse-r1), and it is higher than WiMoT in every seed. The highest WiMANS AFA of the study is nevertheless that of WiMoT with pretrained MantisV2 (89.61\%, against 89.51\% for SI+ER with Chronos-2). The coefficients of all methods were selected with MantisV2 features, and the small number of training records per class on WiMANS (Section~6.3 of the main paper) affects the moments of every encoder; the study does not separate these two factors.

Third, on 3DO the randomly initialized MantisV2, TSPulse-r1, and MOMENT-1-large give WiMoT a higher AFA (96.03, 95.56, and 94.82\%) than pretrained MantisV2 (94.02\%). This is consistent with Section~\ref{sec:amponly}, which shows that the 3DO class information is carried mainly by the amplitude statistics.

Weights were loaded from local copies of the Hugging Face repositories at the listed revisions, and every feature file records the SHA-256 digests of the weights, the prepared inputs, and the encoder state. The randomly initialized MantisV2 uses the vendor architecture with seed 20261014. Chronos-2 runs in float32 because its bfloat16 embeddings changed by up to 6\% with the batch composition; the other encoders run under bfloat16 autocast. For the five foundation models, re-encoding the first two recordings of every feature file in a separate batch changed their embeddings by at most $1.7\times10^{-5}$ of the mean embedding magnitude, and by zero for MOMENT-1-large and both TimesFM versions; extraction aborts above 0.05. TSPulse-r1 is run with its masking disabled, as the masking is otherwise random in evaluation mode. The TimesFM 3.0 Widar3.0 features were extracted on a second machine with the same code, weights, and PyTorch build, and their encoder-state and input digests equal those of the first machine. All training and testing of one encoder ran on one machine.

\subsection{Covariance Storage}
Packed storage keeps only the upper triangle of each symmetric covariance. Rank-$r$ storage keeps the $r$ leading eigenpairs of each class covariance and one isotropic residual variance, the mean of the remaining eigenvalues, from which the covariance is reconstructed before transport; it needs $513r+1$ numbers per class instead of $512^2$. All other statistics are unchanged. The study used seeds 42--44 on one machine; Table~\ref{tab:memory} reports the results.

\begin{table*}[t]
\centering\footnotesize
\setlength{\tabcolsep}{5pt}
\caption{Covariance storage: class-moment payload (MiB) and paired AFA difference from full storage (pp, mean $\pm$ SD over seeds 42--44; number of seeds higher/lower)}
\label{tab:memory}
\begin{tabular}{lcccccc}
\toprule
& \multicolumn{2}{c}{3DO} & \multicolumn{2}{c}{WiMANS} & \multicolumn{2}{c}{Widar3.0}\\
\cmidrule(lr){2-3}\cmidrule(lr){4-5}\cmidrule(lr){6-7}
Storage & MiB & $\Delta$AFA & MiB & $\Delta$AFA & MiB & $\Delta$AFA\\
\midrule
Full & 18.04 & $94.03$\,\% & 54.11 & $89.61$\,\% & 108.21 & $95.19$\,\%\\
Upper triangle & 9.05 & identical & 27.16 & identical & 54.32 & identical\\
Rank 128 & 4.54 & $+0.31\pm0.49$ (2/1) & 13.63 & identical & 27.26 & identical\\
Rank 32 & 1.16 & $+1.07\pm0.62$ (3/0) & 3.49 & $+0.00\pm0.62$ (1/1) & 6.97 & $-0.49\pm0.31$ (0/3)\\
Rank 8 & 0.32 & $+0.25\pm0.79$ (2/1) & 0.95 & $-0.41\pm0.78$ (1/2) & 1.90 & $-2.74\pm0.39$ (0/3)\\
\bottomrule
\end{tabular}\\[2pt]
\parbox{0.8\textwidth}{\scriptsize The Full row gives the AFA of the full-storage run of this study; ``identical'' means identical stage-wise test accuracies in every seed.}
\end{table*}

Packed storage halves the payload and, as expected from symmetry, reproduces every stage-wise test accuracy of full storage on all three streams. Rank-128 storage is lossless up to numerical precision on WiMANS and Widar3.0, where each class covariance is estimated from at most 42 and 75 training instances and therefore has rank below 128; on 3DO, with 800 windows per class and domain, it changes AFA by $+0.31\pm0.49$ pp. Further truncation is dataset-dependent. At rank 32, 3DO AFA rises by $1.07\pm0.62$ pp in every seed, WiMANS is unchanged on average, and Widar3.0 loses $0.49\pm0.31$ pp in every seed; at rank 8, the Widar3.0 loss grows to $2.74\pm0.39$ pp. The 3DO gain suggests that discarding small-variance directions regularizes the 3DO class covariances; this effect was not used to choose any configuration of the paper.

\subsection{Lead and Baseline Forgetting}
Table~\ref{tab:leadafr} lists the 42 settings of the correlation analysis in Section~5.4 of the main paper. For each setting, the lead is the three-seed mean AFA of WiMoT minus that of the baseline with the highest mean AFA among the baselines run in that setting (21 in the main settings, 14 in the further 3DO and WiMANS orders, 10 with reduced labels, and 5 in the further Widar3.0 orders and rooms and with other encoders). The Widar3.0 column includes the Room-2 and Room-3 streams.

\begin{table*}[t]
\centering\scriptsize
\setlength{\tabcolsep}{2.6pt}
\caption{AFA lead of WiMoT over the strongest baseline of each setting (pp), AFR of that baseline (pp), and mean AFR of all baselines of the setting (pp); three-seed means}
\label{tab:leadafr}
\begin{tabular}{lrrr@{\hspace{10pt}}lrrr@{\hspace{10pt}}lrrr}
\toprule
\multicolumn{4}{c}{3DO} & \multicolumn{4}{c}{WiMANS} & \multicolumn{4}{c}{Widar3.0}\\
\cmidrule(lr){1-4}\cmidrule(lr){5-8}\cmidrule(lr){9-12}
Setting & Lead & AFR$_{\mathrm{best}}$ & AFR$_{\mathrm{mean}}$ & Setting & Lead & AFR$_{\mathrm{best}}$ & AFR$_{\mathrm{mean}}$ & Setting & Lead & AFR$_{\mathrm{best}}$ & AFR$_{\mathrm{mean}}$\\
\midrule
Main & $+3.77$ & $5.64$ & $7.27$ & Main & $+0.41$ & $2.78$ & $4.44$ & Main & $+13.11$ & $10.25$ & $16.21$\\
Order 1-3-2 & $+5.04$ & $7.43$ & $9.12$ & Order 1-3-2 & $-1.03$ & $0.62$ & $3.69$ & Order 4-8-9-2-3-5-1-7-6 & $+11.14$ & $11.67$ & $17.26$\\
Order 2-1-3 & $+6.81$ & $7.24$ & $9.83$ & Order 2-1-3 & $+1.85$ & $1.85$ & $3.78$ & Order 6-7-4-3-1-2-5-8-9 & $+14.52$ & $14.33$ & $19.31$\\
Order 2-3-1 & $+6.36$ & $6.54$ & $9.18$ & Order 2-3-1 & $+1.03$ & $2.93$ & $3.08$ & Reverse order & $+9.58$ & $10.33$ & $15.08$\\
Order 3-1-2 & $+1.13$ & $4.19$ & $9.42$ & Order 3-1-2 & $-1.95$ & $-0.31$ & $1.81$ & Room 2 & $+0.45$ & $3.80$ & $8.59$\\
Order 3-2-1 & $+2.31$ & $3.26$ & $7.72$ & Order 3-2-1 & $-0.93$ & $1.54$ & $2.80$ & Room 3 & $+9.92$ & $12.63$ & $21.35$\\
10\% labels & $+2.40$ & $2.85$ & $5.56$ & 10\% labels & $-1.44$ & $2.47$ & $4.51$ & 10\% labels & $+10.57$ & $7.11$ & $11.93$\\
25\% labels & $+2.81$ & $4.06$ & $6.73$ & 25\% labels & $-1.23$ & $1.85$ & $4.88$ & 25\% labels & $+9.85$ & $8.83$ & $13.57$\\
Chronos-2 & $+6.31$ & $7.92$ & $6.45$ & Chronos-2 & $-3.40$ & $1.08$ & $0.83$ & Chronos-2 & $+13.70$ & $8.31$ & $12.72$\\
MantisV2 random & $+7.89$ & $4.24$ & $8.28$ & MantisV2 random & $+1.44$ & $-1.23$ & $1.88$ & MantisV2 random & $+30.99$ & $21.92$ & $24.73$\\
MOMENT-1-large & $+7.30$ & $4.11$ & $4.22$ & MOMENT-1-large & $-5.04$ & $-1.08$ & $-1.73$ & MOMENT-1-large & $+16.62$ & $10.33$ & $15.82$\\
TimesFM 2.5 & $+4.15$ & $7.04$ & $5.89$ & TimesFM 2.5 & $-2.67$ & $1.23$ & $0.93$ & TimesFM 2.5 & $+17.90$ & $14.08$ & $14.21$\\
TimesFM 3.0 & $+6.69$ & $6.42$ & $6.13$ & TimesFM 3.0 & $-4.73$ & $-1.54$ & $-1.27$ & TimesFM 3.0 & $+18.62$ & $13.67$ & $15.68$\\
TSPulse-r1 & $+9.56$ & $5.86$ & $4.84$ & TSPulse-r1 & $-4.53$ & $-2.78$ & $-2.59$ & TSPulse-r1 & $+26.40$ & $14.81$ & $19.13$\\
\bottomrule
\end{tabular}\\[2pt]
\parbox{\textwidth}{\scriptsize Orders of 3DO and WiMANS number the domains by their position in the original order; Widar3.0 orders number the Room-1 users as in Table~6 of the main paper.}
\end{table*}

Over all 42 settings, the lead correlates with the AFR of the strongest baseline (Pearson $r=0.94$, $p<0.001$; Spearman $\rho=0.94$) and with the mean baseline AFR ($r=0.89$, $\rho=0.89$). After the dataset means are removed, the correlations are $r=0.80$, $\rho=0.62$ and $r=0.60$, $\rho=0.35$ ($p<0.001$ and $p=0.023$ for the two Spearman coefficients). Within each dataset (14 settings), the Spearman correlation of the lead with the AFR of the strongest baseline is $0.40$ on 3DO ($p=0.16$), $0.59$ on WiMANS ($p=0.027$), and $0.75$ on Widar3.0 ($p=0.002$), and with the mean baseline AFR $-0.21$ ($p=0.46$), $0.70$ ($p=0.006$), and $0.49$ ($p=0.07$). The association is therefore present within WiMANS and Widar3.0 and weak within 3DO. The $p$-values are descriptive: the 42 settings share recordings, seeds, and, within a study, baselines, and they pool different baseline sets. Because forgetting lowers the final accuracy of the baseline that defines the lead, the correlation does not by itself identify a mechanism. WiMoT is ahead in 32 of the 42 settings (14 of 14 on 3DO and Widar3.0, 4 of 14 on WiMANS); since the encoder study contributes 6 of the 14 settings of each dataset, these counts are better read per study than as a total.

\section{Input Features of WiMoT and the Baselines}\label{sec:inputsource}
WiMoT reads the Mantis block and the amplitude descriptor through separate stems, whereas the baselines of Table~2 read their concatenation. This section examines whether the comparison depends on this input: Sections~\ref{sec:addamp} and \ref{sec:amponly} vary the input of four representative baselines and of WiMoT, and Section~\ref{sec:mantisonly} repeats all twenty-one baselines with the Mantis block alone.

\begin{table*}[t]
\centering\scriptsize
\setlength{\tabcolsep}{2.1pt}
\caption{Test AFA (\%, mean $\pm$ sample SD, seeds 42--44) with the Mantis block (9216), the concatenation of the Mantis block and the amplitude descriptor (9728), and the amplitude descriptor alone (512)}
\label{tab:inputsource}
\begin{tabular}{lccccccccc}
\toprule
& \multicolumn{3}{c}{3DO} & \multicolumn{3}{c}{WiMANS} & \multicolumn{3}{c}{Widar3.0}\\
\cmidrule(lr){2-4}\cmidrule(lr){5-7}\cmidrule(lr){8-10}
Method & 9216 & 9728 & 512 & 9216 & 9728 & 512 & 9216 & 9728 & 512\\
\midrule
WiMoT (Table~2) & -- & $94.02\pm0.07$ & -- & -- & $89.61\pm0.47$ & -- & -- & $95.16\pm1.23$ & --\\
WiMoT, amplitude stem$^{b}$ & -- & -- & $97.37\pm0.23$ & -- & -- & $25.82\pm1.25$ & -- & -- & $85.90\pm0.37$\\
WiMoT, amplitude basis$^{a}$ & -- & -- & $99.02\pm0.16$ & -- & -- & $32.51\pm2.06$ & -- & -- & $94.00\pm0.15$\\
\midrule
SI+ER & $90.24\pm0.22$ & $90.25\pm0.67$ & $97.86\pm0.59$ & $87.96\pm0.31$ & $87.86\pm0.64$ & $23.97\pm1.39$ & $81.73\pm0.62$ & $81.38\pm0.95$ & $52.54\pm1.37$\\
DER++ & $89.76\pm0.78$ & $89.81\pm0.89$ & $95.89\pm1.04$ & $88.27\pm0.53$ & $88.27\pm0.53$ & $24.90\pm1.39$ & $80.27\pm0.04$ & $80.25\pm1.01$ & $46.79\pm1.30$\\
EdgeCL & $89.86\pm0.29$ & $89.68\pm1.29$ & $82.41\pm1.20$ & $86.63\pm1.58$ & $85.49\pm2.64$ & $28.40\pm1.88$ & $74.77\pm0.38$ & $75.01\pm1.20$ & $27.36\pm0.33$\\
CN-DPM & $90.45\pm0.11$ & $88.48\pm0.50$ & $97.44\pm0.32$ & $85.08\pm0.36$ & $86.32\pm1.25$ & $27.88\pm1.70$ & $75.16\pm0.84$ & $74.84\pm2.22$ & $42.02\pm0.15$\\
\midrule
Joint training (non-continual) & $89.65\pm0.25$ & $90.64\pm0.76$ & $99.03\pm0.05$ & $90.95\pm0.36$ & $91.46\pm0.94$ & $28.19\pm0.71$ & $93.51\pm0.11$ & $94.84\pm0.41$ & $94.42\pm0.19$\\
\bottomrule
\end{tabular}\\[2pt]
\parbox{\textwidth}{\scriptsize $^{a}$Semantic stem and PCA shortcut read the amplitude descriptor; no foundation encoder. $^{b}$Only the fixed random amplitude stem. WiMoT reads both blocks through separate stems and is listed under 9728. The 9728 column equals Table~2; Section~\ref{sec:inputsetting} describes the machines used for the other columns.}
\end{table*}
\subsection{Setting}\label{sec:inputsetting}
SI+ER, DER++, EdgeCL, and CN-DPM were trained with three inputs: the 9216-dimensional Mantis block, the 9728-dimensional concatenation of the Mantis block and the amplitude descriptor (the input of Table~2), and the 512-dimensional amplitude descriptor alone. The concatenated input is normalized as in WiMoT, the Mantis block by first-domain coordinate standardization and division by $\sqrt{9216}$ and the amplitude block by its uncentered first-domain RMS times $\sqrt{512}$; the amplitude-only input uses the same amplitude scaling. The MLP baselines widen only their input layer, EdgeCL receives the amplitude descriptor as a nineteenth 512-dimensional token, and CN-DPM applies its own coordinate normalization to each input. For every new input, the SI+ER stabilizer, the EdgeCL radius, and the CN-DPM log-concentration are chosen from the candidates of Section~\ref{sec:baselines} by the equally weighted seed-42 validation AFA of 3DO and WiMANS; DER++ keeps $\alpha=\beta=0.5$. SI+ER selected 10 and CN-DPM $-10$ for all inputs; EdgeCL selected 0.01 for the amplitude-only input and 0.03 for the concatenated input. Seeds 42--44, partitions, domain orders, replay caches, and schedules are those of Table~2, and the 9728-dimensional column of Table~\ref{tab:inputsource} equals the corresponding entries of Table~2.

Two amplitude-only variants of WiMoT use the recipe and classifier of the main paper unchanged. In the first, the Mantis input is set to zero and the PCA shortcut is omitted, so all input information passes through the fixed random amplitude stem. In the second, the semantic stem and the first-domain PCA shortcut read the amplitude descriptor in place of the Mantis block, and no foundation encoder is evaluated. The first-domain 3DO amplitude descriptors have numerical rank 101, because they are interpolated from few channels; the PCA shortcut therefore keeps 101 components there and sets its remaining 155 coordinates to zero, whereas WiMANS and Widar3.0 keep 256 components. Both variants are analysis controls; WiMoT remains the configuration of Table~2. As a non-continual reference, the MLP of the classic baselines was trained with CE for 100 epochs on the union of the training partitions of all domains, with normalization fitted on this union, for each of the three inputs.

The runs used two machines with the same GPU model and software environment. The 9216- and 9728-dimensional baselines were trained on the first machine, except the 9728-dimensional EdgeCL, whose records are those of Table~2; the amplitude-only runs and the joint references were trained on the second. Both machines also ran WiMoT and the 9216-dimensional baselines, so every paired difference below uses runs from a single machine. WiMoT reproduces Table~2 on both machines in every seed. The 9216-dimensional runs of SI+ER, DER++, and CN-DPM are identical on the two machines and differ by at most 0.64 pp from the 9216-dimensional means of Table~\ref{tab:mantisonly}, which were obtained on a third machine. EdgeCL is not bitwise reproducible across GPUs: with identical code and seeds, its per-seed AFA differed by up to 3.06 pp between the two machines (9216-dimensional means of $87.65\pm0.95$, $86.32\pm2.32$, and $74.40\pm0.81$\% on the second), so EdgeCL differences of this size are not interpreted.

\subsection{Adding the Amplitude Descriptor}\label{sec:addamp}
Appending the amplitude descriptor to the Mantis block changes the AFA of SI+ER, DER++, and EdgeCL by $-1.14$ to $+0.24$ pp, within seed variation. For CN-DPM it lowers 3DO AFA by $1.97\pm0.45$ pp, in every seed, and raises its AFR from 3.86 to 8.18 pp; on WiMANS and Widar3.0 the changes are $+1.23\pm1.35$ and $-0.32\pm1.61$ pp. Giving SI+ER and DER++ the 128-dimensional output of the fixed random amplitude stem of WiMoT in place of the descriptor changed their AFA by $-0.74$ to $+0.41$ pp. The joint references gain 0.99, 0.51, and 1.33 pp from the concatenation. The CE-trained networks therefore make little use of the amplitude block when it is appended to the Mantis block, although the descriptor alone is informative on 3DO and Widar3.0 (Section~\ref{sec:amponly}). With the 9728-dimensional input, WiMoT exceeds SI+ER by $3.77\pm0.67$, $1.75\pm0.78$, and $13.78\pm0.85$ pp on 3DO, WiMANS, and Widar3.0, DER++ by $4.21\pm0.86$, $1.34\pm0.36$, and $14.91\pm0.23$ pp, EdgeCL by $4.34\pm1.36$, $4.12\pm2.78$, and $20.15\pm1.64$ pp, and CN-DPM by $5.54\pm0.47$, $3.29\pm1.25$, and $20.32\pm1.83$ pp, positive in every seed in all twelve comparisons.

\subsection{Amplitude Descriptor Alone}\label{sec:amponly}
On 3DO, the class information is carried mainly by the amplitude statistics. With the amplitude descriptor alone, SI+ER, CN-DPM, and DER++ reach 97.86, 97.44, and 95.89\%, above the 94.02\% of WiMoT in its main configuration, which is lower by $3.84\pm0.52$, $3.43\pm0.36$, and $1.87\pm0.97$ pp in every seed. Given the same amplitude-only input, however, the WiMoT variant whose basis reads the descriptor reaches $99.02\pm0.16$\%, close to the joint reference of 99.03\%, and exceeds these three baselines by $1.16\pm0.46$, $1.57\pm0.37$, and $3.13\pm0.88$ pp in every seed; it is also $5.00\pm0.10$ pp above the main configuration. The variant with the random amplitude stem only reaches $97.37\pm0.23$\%: $0.49\pm0.44$ pp below SI+ER (every seed), $0.07\pm0.38$ pp below CN-DPM (one of three seeds ahead), and $1.48\pm0.82$ pp above DER++ (every seed). The 3DO lead of Table~2 is thus obtained against baselines that read the Mantis block together with the amplitude descriptor; baselines that read only the descriptor exceed the main configuration of WiMoT, but not WiMoT given the same input.

On WiMANS, the amplitude descriptor carries little class information: every method stays between 23.97 and 32.51\% with this input, and the joint reference reaches 28.19\%, compared with a chance level of 11.1\% for nine classes. Recognition on WiMANS therefore relies on the foundation-encoder features, in line with the encoder study, where a randomly initialized MantisV2 lowers the AFA of WiMoT from 89.61 to 78.29\% (Table~\ref{tab:backbone}).

On Widar3.0, the amplitude descriptor suffices for joint training, which reaches 94.42\%, but the continual baselines reach only 27.36--52.54\% with it, with AFR between 43.92 and 58.44 pp. The WiMoT variant whose basis reads the descriptor reaches $94.00\pm0.15$\% with 0.64 pp AFR and exceeds SI+ER, DER++, EdgeCL, and CN-DPM by $41.46\pm1.42$, $47.21\pm1.35$, $66.64\pm0.30$, and $51.98\pm0.28$ pp in every seed, the largest differences among the three inputs. With the random amplitude stem alone, it reaches $85.90\pm0.37$\%. The configuration of Table~2 remains higher than the variant without a foundation encoder, by $1.16\pm1.08$ pp, in every seed.

\subsection{Mantis Block Alone for All Baselines}\label{sec:mantisonly}
Table~\ref{tab:mantisonly} repeats all twenty-one baselines with the 9216-dimensional Mantis block alone, with coefficients selected for that input, and compares them with the 9728-dimensional input of Table~2. The input has little effect on the strongest baseline, whose AFA changes by $-0.20$, $+0.62$, and $+0.17$ pp from the 9216- to the 9728-dimensional input on 3DO, WiMANS, and Widar3.0. Several recent methods are more sensitive: DualCP, whose threshold had to be changed from 0.1 to 0.85 (Section~\ref{sec:baselines}), differs by up to 24.38 pp on WiMANS, and DUCT+ER varies strongly across seeds with the 9728-dimensional input on that stream. Against the strongest baseline of each input, the lead of WiMoT is $+3.56\pm0.16$, $+1.03\pm0.47$, and $+13.28\pm0.97$ pp with the Mantis block alone and $+3.77\pm0.67$, $+0.41\pm0.64$, and $+13.11\pm1.22$ pp with the 9728-dimensional input. Against the better variant of every baseline it is $+3.56\pm0.16$, $+0.41\pm0.64$, and $+13.11\pm1.22$ pp, positive in every seed on 3DO and Widar3.0 and in two of three seeds on WiMANS. The leads on 3DO and Widar3.0 thus do not depend on the input of the baselines, and on WiMANS the AFA difference stays within the seed variation for either input.

\begin{table*}[t]
\centering\scriptsize
\setlength{\tabcolsep}{2.6pt}
\caption{Test AFA (\%, mean $\pm$ sample SD, seeds 42--44) of the twenty-one baselines with the 9216-dimensional Mantis block alone and with the 9728-dimensional input (Table~2), and the paired AFA difference of WiMoT from the strongest baseline (pp; seeds ahead). Bold: highest baseline of each column. The last two rows compare with the better of the two variants of every baseline, listed under 9728}
\label{tab:mantisonly}
\begin{tabular}{lcccccc}
\toprule
& \multicolumn{2}{c}{3DO} & \multicolumn{2}{c}{WiMANS} & \multicolumn{2}{c}{Widar3.0}\\
\cmidrule(lr){2-3}\cmidrule(lr){4-5}\cmidrule(lr){6-7}
Method & 9216 & 9728 & 9216 & 9728 & 9216 & 9728\\
\midrule
\textbf{WiMoT} & -- & $\boldsymbol{94.02\pm0.07}$ & -- & $\boldsymbol{89.61\pm0.47}$ & -- & $\boldsymbol{95.16\pm1.23}$\\
\midrule
Fine-tuning & $82.98\pm0.66$ & $84.19\pm1.13$ & $86.52\pm0.18$ & $86.93\pm1.17$ & $72.47\pm1.57$ & $70.02\pm0.38$\\
EWC & $82.97\pm0.61$ & $84.19\pm1.10$ & $86.32\pm0.18$ & $86.93\pm1.17$ & $72.47\pm1.24$ & $70.62\pm0.37$\\
SI & $89.69\pm0.55$ & $89.38\pm0.21$ & $87.14\pm0.71$ & $87.76\pm1.52$ & $77.53\pm0.69$ & $75.51\pm0.28$\\
MAS & $87.55\pm0.69$ & $87.72\pm1.13$ & $86.93\pm0.18$ & $88.17\pm1.28$ & $68.77\pm0.31$ & $68.44\pm0.58$\\
LwF & $84.45\pm0.29$ & $85.17\pm0.92$ & $87.24\pm0.78$ & $88.58\pm0.62$ & $79.46\pm0.42$ & $79.06\pm0.60$\\
ER & $89.76\pm0.77$ & $89.71\pm0.72$ & $87.45\pm0.18$ & $86.73\pm0.31$ & $76.49\pm0.67$ & $76.94\pm0.58$\\
EWC+ER & $89.73\pm0.81$ & $89.64\pm0.76$ & $87.35\pm0.31$ & $87.04\pm0.62$ & $76.52\pm0.78$ & $78.91\pm1.06$\\
SI+ER & $90.36\pm0.20$ & $\boldsymbol{90.25\pm0.67}$ & $88.17\pm0.47$ & $87.86\pm0.64$ & $\boldsymbol{81.88\pm0.76}$ & $81.38\pm0.95$\\
MAS+ER & $87.79\pm0.42$ & $89.28\pm0.96$ & $87.65\pm0.31$ & $\boldsymbol{89.20\pm1.07}$ & $71.14\pm0.11$ & $76.17\pm0.68$\\
LwF+ER & $87.20\pm0.66$ & $88.36\pm0.29$ & $88.37\pm0.18$ & $89.09\pm0.47$ & $81.41\pm0.34$ & $\boldsymbol{82.05\pm1.07}$\\
DER++ & $89.86\pm0.81$ & $89.81\pm0.89$ & $\boldsymbol{88.58\pm0.82}$ & $88.27\pm0.53$ & $80.37\pm0.22$ & $80.25\pm1.01$\\
A-GEM & $85.51\pm0.33$ & $86.22\pm0.75$ & $86.73\pm0.62$ & $86.63\pm0.71$ & $73.06\pm0.95$ & $72.91\pm0.50$\\
CN-DPM & $\boldsymbol{90.45\pm0.11}$ & $88.48\pm0.50$ & $85.60\pm0.64$ & $86.32\pm1.25$ & $75.80\pm0.85$ & $74.84\pm2.22$\\
\midrule
DUCT & $76.82\pm12.55$ & $84.80\pm1.64$ & $79.84\pm2.62$ & $74.49\pm7.86$ & $42.10\pm1.32$ & $37.41\pm7.26$\\
DUCT+ER & $84.95\pm0.82$ & $84.44\pm1.04$ & $79.01\pm3.21$ & $61.11\pm31.64$ & $37.63\pm7.16$ & $40.74\pm0.32$\\
DualCP & $77.59\pm0.31$ & $77.14\pm0.56$ & $68.31\pm0.94$ & $43.93\pm0.18$ & $53.51\pm0.64$ & $65.95\pm0.34$\\
DualCP+ER & $81.42\pm0.49$ & $81.92\pm0.29$ & $71.09\pm1.52$ & $57.00\pm0.99$ & $53.28\pm0.11$ & $65.33\pm0.30$\\
VanillaCP & $78.24\pm0.11$ & $77.81\pm0.04$ & $73.05\pm1.39$ & $67.70\pm0.64$ & $69.23\pm0.57$ & $73.14\pm0.31$\\
VanillaCP+ER & $82.69\pm0.06$ & $82.65\pm0.22$ & $73.87\pm1.39$ & $71.19\pm0.47$ & $69.28\pm0.54$ & $72.17\pm0.50$\\
ConSense & $77.05\pm1.28$ & $76.31\pm0.39$ & $81.17\pm2.23$ & $83.85\pm1.08$ & $56.00\pm1.37$ & $64.62\pm0.93$\\
EdgeCL & $89.23\pm0.98$ & $89.68\pm1.29$ & $86.93\pm1.58$ & $85.49\pm2.64$ & $74.89\pm0.83$ & $75.01\pm1.20$\\
\midrule
Strongest: baseline & CN-DPM & SI+ER & DER++ & MAS+ER & SI+ER & LwF+ER\\
Strongest: lead & $+3.56\pm0.16$ (3/3) & $+3.77\pm0.67$ (3/3) & $+1.03\pm0.47$ (3/3) & $+0.41\pm0.64$ (2/3) & $+13.28\pm0.97$ (3/3) & $+13.11\pm1.22$ (3/3)\\
Better variant: baseline & -- & CN-DPM & -- & MAS+ER & -- & LwF+ER\\
Better variant: lead & -- & $+3.56\pm0.16$ (3/3) & -- & $+0.41\pm0.64$ (2/3) & -- & $+13.11\pm1.22$ (3/3)\\
\bottomrule
\end{tabular}
\end{table*}

\section{Widar3.0 Protocol}\label{sec:widarsupp}
\subsection{Source and Partitions}
The main stream uses the Widar3.0 CSI archives of Room~1 recorded on 2018-11-30. Their SHA-256 digests are

{\tiny
\noindent\texttt{20181130\_user5\_10\_11.zip}\\
\texttt{82c8225aeec37258598e22c927f6c4b8a40211013fa2b741c926f6ef162be2f0}\\
\texttt{20181130\_user12\_13\_14.zip}\\
\texttt{fc47cb4d946265ce219bfff60cc2798a50dd36afda2b7894958eb0728cc1352a}\\
\texttt{20181130\_user15\_16\_17.zip}\\
\texttt{4c753d0aed8505cee6daf5440f78f7e92286836b634ed9b0cdf79d1cfac968f7}\par}

\noindent Each of the nine users is one domain, in the order 5, 10, 11, 12, 13, 14, 15, 16, 17. The six labels are gestures 1--6 of these archives (push and pull, sweep, clap, slide, and horizontal O and zigzag), which all nine users performed. Each user performed each gesture at five torso locations and five face orientations, five repetitions each, recorded by six receivers with three antennas. An instance is identified by user, gesture, torso location, face orientation, and repetition; its six receiver files are kept together. Repetitions 1--3 form the training partition, repetition 4 the validation partition, and repetition 5 the test partition, so the three partitions contain the same location-orientation grid but different performances, giving 450, 150, and 150 instances per user before exclusions. The partitions therefore separate repetitions within a session rather than recording sessions.

Each receiver file is parsed as Intel 5300 CSI with three receive antennas and one transmit antenna. Amplitudes of the 30 subcarriers are averaged within 300 equal packet-order bins per receiver, and the six receivers are stacked into $540\times300$ channels. Each instance is standardized by the mean and SD of all its entries. Instances in which any receiver records fewer than 300 packets are excluded: user12-1-3-5-4 from validation, and user13-5-5-3-1, user13-6-1-3-2, user14-4-1-1-3, user15-3-5-1-2, user15-5-4-5-2, user17-5-3-3-3, and user17-5-5-2-2 from training (user-gesture-location-orientation-repetition). All test partitions are complete with 150 instances. Mantis features and the amplitude descriptor are computed with the same frozen encoder, grouping, and descriptor as for WiMANS; the encoder identity is checked against the WiMANS record.

\subsection{Settings and Budgets}
All baselines use the coefficients selected on 3DO and WiMANS (Section~\ref{sec:baselines}) without change, and WiMoT and the baselines keep their 3DO/WiMANS recipes, with eight 50-epoch updates after the first domain. DualCP and VanillaCP calibrators are frozen after their domain, so the stage-$k$ predictor is evaluated as the final model restricted to its first $k$ calibrators and domain means; this is identical to the stage-$k$ model and avoids storing nine growing checkpoints. The history-aligned methods were checked at every update against the retained-identity records of the common reservoir. The common replay cache has $M=42$ training instances, seven per class, shared by all observed users; the replay share is $42/(N_t+42)$, between 8.54\% and 8.61\% for $N_t\in[446,450]$. EdgeCL keeps $\lfloor M/C\rfloor$ exemplars per class and divides them over the retained domains; when the division is not exact, the remainder goes to the most recent domains. For 3DO and WiMANS this rule coincides with the original equal split.

\subsection{Further Rooms}
The Room-2 and Room-3 streams use the Widar3.0 CSI archives with SHA-256 digests

{\tiny
\noindent\texttt{20181204.zip}\\
\texttt{84c08311715027b381633d78444b674ac2f4720dc847ed1ca207b43b5578dc8a}\\
\texttt{20181209.zip}\\
\texttt{5977c10525068ab5a037af1187befb7bf70f65f5d39329b6903ea00c36e811ca}\\
\texttt{20181211.zip}\\
\texttt{65c8a9db3902825c11fb269fa5f2eea9ed3fae963d8c33c54cd63b87a7e76a20}\par}

\noindent Room~2 contains user 1 (2018-12-04) and user 6 (2018-12-09), Room~3 users 3, 7, 8, and 9 (2018-12-11). Only gesture codes 1--6 enter the streams; according to the dataset documentation they denote the same six gestures as in Room~1. Partitions, binning, normalization, features, recipes, and coefficients are those of Room~1. Excluded instances (user-gesture-location-orientation-repetition) are, for Room~2, user6-3-1-1-5 (empty receiver file), user1-1-2-1-2, user1-1-5-5-1, and user1-4-4-1-5 (fewer than 300 packets), and, for Room~3, user8-1-1-1-1, user8-3-3-3-5, and user9-1-1-1-1 (empty receiver file), and user9-2-2-3-1, user9-2-4-3-2, user9-3-2-2-1, user9-3-4-2-3, user9-4-5-5-2, user9-5-2-3-2, and user9-2-2-5-5 (fewer than 300 packets). The training partitions contain 443--450 instances per user. The replay cache of the replay baselines follows the same capacity rule, seven instances per class. Features are computed with the same frozen encoder as for Room~1, and the encoder identity is checked against the Room-1 record.

\section{Plasticity and Stability in the Ablation}\label{sec:balance}
Table~3 of the main paper reports, for WiMoT and the baselines, the accuracy on the domain just learned (new-domain accuracy, plasticity), the accuracy on the earlier domains after the last update (retention), and AFR. Table~\ref{tab:balanceabl} gives the same quantities for the single-component ablation of Table~4, computed from the stage-wise test-accuracy matrices of seeds 42--46; the script \texttt{scripts/plasticity\_stability.py} of the reproducibility package recomputes them.

\begin{table*}[t]
\centering\scriptsize
\setlength{\tabcolsep}{2.6pt}
\caption{New-domain accuracy and earlier-domain accuracy (\%) and AFR (pp), as in Table~3 of the main paper, for the single-component ablation of Table~4; five runs, seeds 42--46}
\label{tab:balanceabl}
\begin{tabular}{l@{\hspace{4pt}}ccc@{\hspace{10pt}}ccc@{\hspace{10pt}}ccc}
\toprule
& \multicolumn{3}{c}{3DO} & \multicolumn{3}{c}{WiMANS} & \multicolumn{3}{c}{Widar3.0}\\
\cmidrule(lr){2-4}\cmidrule(lr){5-7}\cmidrule(lr){8-10}
Method & New & Earlier & AFR & New & Earlier & AFR & New & Earlier & AFR\\
\midrule
\textbf{WiMoT (full)} & 93.10 & 96.75 & 0.26 & 87.22 & 90.83 & $-0.28$ & 94.97 & 96.38 & 0.42\\
w/o SupCon & 87.50 & 97.22 & $-0.57$ & 80.37 & 85.37 & $-3.89$ & 88.52 & 94.02 & 0.15\\
w/o $\mathcal L_{\mathrm{hist}}$ & 93.09 & 96.58 & 0.62 & 86.94 & 89.35 & 0.83 & 94.93 & 95.67 & 1.10\\
w/o historical moments & 97.92 & 66.73 & 30.54 & 86.48 & 82.96 & 5.00 & 94.47 & 43.65 & 52.45\\
Frozen adapter ($t>1$) & 92.67 & 95.37 & 1.33 & 83.70 & 88.06 & 0.09 & 88.07 & 90.88 & 0.40\\
\bottomrule
\end{tabular}
\end{table*}

Each component acts mainly on the quantity that its role in Section~4.1 predicts. Removing supervised contrastive learning lowers the new-domain accuracy by 5.60, 6.85, and 6.45 pp, while AFR does not rise; because this row also trains the first domain without SupCon, Section~\ref{sec:sweep} separates the effect on later domains. Removing the historical loss leaves the new-domain accuracy within 0.28 pp, lowers the earlier-domain accuracy by 0.17, 1.48, and 0.71 pp, and raises AFR by 0.36, 1.11, and 0.68 pp. Without the stored historical moments, the classifier reads only the current domain: the new-domain accuracy changes by at most 0.74 pp on WiMANS and Widar3.0 and rises on 3DO, while the earlier-domain accuracy falls by 30.02, 7.87, and 52.73 pp. Freezing the adapter lowers the new-domain accuracy by 0.43, 3.52, and 6.90 pp, most on the longest stream.

\section{Classifier Replacement}\label{sec:readoutabl}
Table~4 of the main paper replaces the OSCL classifier (Section~4.5) by simpler classifiers fitted on the same adapters and the same stored moments. This section describes these classifiers and gives the complete results (Table~\ref{tab:readoutabl}). The WiMoT models are those of the ablation (seeds 42--46), so the OSCL rows coincide with the full-method rows of Table~4.

\emph{Implementation.} Every classifier uses only the class moments stored for the observed domains, transported through the adapter of the current update (Section~4.3); it uses no historical sample and no domain identity, and every observed domain has the same weight. \emph{Nearest class mean} assigns a sample to the class with the closest mean in Euclidean distance, where the class mean is the average of the transported domain means. \emph{LDA} forms for each class the mixture of its domain moments, whose covariance adds the between-domain scatter of the means to the average within-domain covariance, averages the class covariances into one shared covariance, and uses uniform class priors. \emph{QDA} uses one Gaussian per class with the same mixture moments, without domain gate or marginalization over domains. The covariances of LDA and QDA are regularized as $(1-\beta)\Sigma+\beta\operatorname{tr}(\Sigma)/d\,I$ with $\beta=\kappa d/(n-1+\kappa d)$ and $\kappa=0.25$, the pseudo-count prior of the OSCL classifier ($d$: feature dimension, $n$: number of samples behind the moment). This regularization was fixed in advance, and no alternative classifier was tuned separately; the rows with a fixed $\beta=10^{-3}$ show the effect of a different ridge. The \emph{linear probe} uses the linear map $W$ of the historical loss as the classification head. Since $W$ is trained to reproduce only the moments of earlier domains and is never fitted to the newest domain, its new-domain accuracy is low, and its AFR is negative on 3DO and WiMANS because its accuracy after the first update is low. Only its AFA and earlier-domain accuracy are informative, and it is therefore not included in Table~4.

\emph{Results.} On 3DO and Widar3.0 the AFA of OSCL is higher than that of every alternative classifier in all five seeds. The strongest alternative is LDA on 3DO (lower by $1.93\pm0.61$ pp) and QDA on Widar3.0 (lower by $3.84\pm0.96$ pp). On WiMANS the nearest class mean is lower by $1.73\pm0.77$ pp and QDA by $8.64\pm1.09$ pp (five of five seeds), whereas LDA is lower by $0.74\pm1.39$ pp, with per-seed differences of $+0.62$, $+1.23$, $+2.78$, $-0.93$, and $0.00$ pp; the corresponding differences in new-domain and earlier-domain accuracy are $1.39\pm1.64$ and $0.28\pm1.37$ pp. With the fixed ridge of $10^{-3}$ the AFA difference to LDA is $0.68\pm1.40$ pp (higher in three seeds, lower in two). On WiMANS OSCL is thus clearly ahead of the nearest class mean and QDA, but not separated from LDA.

\begin{table}[t]
\centering\scriptsize
\setlength{\tabcolsep}{2.4pt}
\caption{AFA, AFR (pp), and new- and earlier-domain accuracy (\%) with the OSCL classifier replaced by other classifiers fitted on the same adapters and transported moments; mean $\pm$ SD, seeds 42--46. The OSCL rows are the full-WiMoT rows of Table~IV; ridge rows use a fixed $\beta=10^{-3}$. $^{\dagger}$Fitted to earlier domains only (Section~\ref{sec:readoutabl})}
\label{tab:readoutabl}
\begin{tabular}{lcccc}
\toprule
Classifier & AFA & AFR & New & Earlier\\
\midrule
\multicolumn{5}{l}{\textit{3DO}}\\
OSCL (WiMoT) & \textbf{93.98$\pm$0.15} & 0.26$\pm$0.31 & 93.10$\pm$0.13 & 96.75$\pm$0.34\\
Nearest class mean & 80.89$\pm$0.89 & 14.04$\pm$0.98 & 88.91$\pm$0.97 & 78.95$\pm$0.41\\
LDA & 92.06$\pm$0.49 & 1.64$\pm$0.90 & 90.98$\pm$0.89 & 95.69$\pm$0.63\\
QDA & 91.00$\pm$0.70 & 1.23$\pm$0.26 & 89.53$\pm$0.98 & 95.47$\pm$0.23\\
Linear probe $W^{\dagger}$ & 85.77$\pm$0.49 & $-$9.97$\pm$0.48 & 70.80$\pm$0.97 & 98.22$\pm$0.35\\
LDA, ridge $10^{-3}$ & 92.04$\pm$0.47 & 1.61$\pm$0.99 & 90.95$\pm$0.91 & 95.72$\pm$0.67\\
QDA, ridge $10^{-3}$ & 89.64$\pm$0.73 & 0.21$\pm$0.59 & 86.94$\pm$0.85 & 95.57$\pm$0.34\\
\midrule
\multicolumn{5}{l}{\textit{WiMANS}}\\
OSCL (WiMoT) & \textbf{89.57$\pm$0.55} & $-$0.28$\pm$0.53 & 87.22$\pm$1.25 & 90.83$\pm$0.21\\
Nearest class mean & 87.84$\pm$0.80 & 0.56$\pm$1.28 & 86.11$\pm$2.29 & 88.06$\pm$0.60\\
LDA & 88.83$\pm$1.33 & $-$0.56$\pm$1.40 & 85.83$\pm$1.21 & 90.56$\pm$1.37\\
QDA & 80.93$\pm$1.33 & 1.11$\pm$1.81 & 77.50$\pm$2.85 & 81.94$\pm$1.57\\
Linear probe $W^{\dagger}$ & 85.62$\pm$1.10 & $-$8.80$\pm$2.45 & 73.89$\pm$2.23 & 88.52$\pm$1.06\\
LDA, ridge $10^{-3}$ & 88.89$\pm$1.33 & $-$0.65$\pm$1.16 & 86.02$\pm$1.11 & 90.56$\pm$1.11\\
QDA, ridge $10^{-3}$ & 63.46$\pm$1.85 & 11.76$\pm$4.76 & 62.59$\pm$2.52 & 63.98$\pm$3.91\\
\midrule
\multicolumn{5}{l}{\textit{Widar3.0}}\\
OSCL (WiMoT) & \textbf{95.24$\pm$0.88} & 0.42$\pm$0.49 & 94.97$\pm$0.35 & 96.38$\pm$0.82\\
Nearest class mean & 67.05$\pm$1.28 & 18.20$\pm$1.76 & 81.43$\pm$0.70 & 68.32$\pm$1.67\\
LDA & 84.43$\pm$0.94 & 2.18$\pm$0.53 & 82.58$\pm$0.57 & 87.08$\pm$0.97\\
QDA & 91.41$\pm$0.50 & 2.48$\pm$0.34 & 92.58$\pm$0.57 & 93.58$\pm$0.57\\
Linear probe $W^{\dagger}$ & 81.94$\pm$0.80 & 0.27$\pm$0.41 & 48.77$\pm$1.59 & 86.35$\pm$0.69\\
LDA, ridge $10^{-3}$ & 84.37$\pm$0.99 & 2.20$\pm$0.47 & 82.63$\pm$0.62 & 86.95$\pm$1.00\\
QDA, ridge $10^{-3}$ & 90.62$\pm$0.52 & 2.00$\pm$0.55 & 89.88$\pm$0.78 & 92.73$\pm$0.69\\
\bottomrule
\end{tabular}
\end{table}

\section{Sensitivity to the Loss Weights}\label{sec:sweep}
The loss of a later domain is $\lambda_{\mathrm{sc}}\mathcal L_{\mathrm{SC}}+\lambda_{\mathrm h}\mathcal L_{\mathrm{hist}}$, with $\lambda_{\mathrm{sc}}=0.5$ and $\lambda_{\mathrm h}=1$ fixed in advance for all datasets. Both weights were varied separately on seeds 42--44 of all three streams (Fig.~5 of the main paper, Table~\ref{tab:sweep}); the results were not used to adjust them.

\emph{Only the ratio matters.} AdamW is invariant to a joint rescaling of the loss, so the weights act only through the ratio $\lambda_{\mathrm{sc}}/\lambda_{\mathrm h}$. Settings with the same ratio, for example $(0.25,0.5)$, $(0.5,1)$, and $(1,2)$, gave identical accuracy matrices in all seeds. A $3\times3$ grid over $\lambda_{\mathrm{sc}}\in\{0.25,0.5,1\}$ and $\lambda_{\mathrm h}\in\{0.5,1,2\}$ contains only five different ratios and is therefore not reported separately. The two rows of Fig.~5 of the main paper are two sections of the same ratio axis, $\lambda_{\mathrm{sc}}$ in the upper row and $0.5/\lambda_{\mathrm h}$ in the lower row.

\emph{First domain.} In Fig.~5 of the main paper and the first two blocks of Table~\ref{tab:sweep}, $\lambda_{\mathrm{sc}}$ acts on the later domains only, and the first domain is trained with SupCon at the weight of the paper. The last block of Table~\ref{tab:sweep} applies $\lambda_{\mathrm{sc}}$ to the first domain as well; there, $\lambda_{\mathrm{sc}}=0$ reproduces the accuracy matrices of the row ``w/o SupCon'' of Table~4 in every seed 42--44, and $\lambda_{\mathrm h}=0$ those of the row ``w/o $\mathcal L_{\mathrm{hist}}$''.

\emph{Results.} Section~5.4 of the main paper summarizes the insensitivity of AFA for ratios between 0.1 and 2. The full results add three observations. First, setting a weight to zero lowers AFA on WiMANS (by 0.93 and 1.03 pp) and Widar3.0 (by 1.36 and 0.91 pp) but not on 3DO ($+0.09$ and $-0.08$ pp). Second, SupCon matters in both phases: removing it from the later domains lowers the Widar3.0 new-domain accuracy from 95.08 to 90.19\%, and also removing it from the first domain lowers AFA further on all three streams (last block of Table~\ref{tab:sweep}). Third, the historical loss controls forgetting on the longest stream: on Widar3.0 the AFR grows as $\lambda_{\mathrm h}$ decreases (0.31 pp at $\lambda_{\mathrm h}=2$, 0.58 pp at 1, 1.08 pp at 0.25, and 1.33 pp at 0), while the new-domain accuracy stays between 94.6 and 95.2\%.

\begin{table*}[t]
\centering\scriptsize
\setlength{\tabcolsep}{1.2pt}
\caption{AFA, AFR (pp), and new- and earlier-domain accuracy (\%) of WiMoT for different loss weights; mean $\pm$ SD, seeds 42--44. Bold: weights of the paper. In the last block $\lambda_{\mathrm{sc}}$ also acts on the first domain}
\label{tab:sweep}
\begin{tabular}{lcccc@{\hspace{4pt}}cccc@{\hspace{4pt}}cccc}
\toprule
 & \multicolumn{4}{c}{3DO} & \multicolumn{4}{c}{WiMANS} & \multicolumn{4}{c}{Widar3.0}\\
\cmidrule(lr){2-5}\cmidrule(lr){6-9}\cmidrule(lr){10-13}
Weight & AFA & AFR & New & Earlier & AFA & AFR & New & Earlier & AFA & AFR & New & Earlier\\
\midrule
\multicolumn{13}{l}{\textit{SupCon weight $\lambda_{\mathrm{sc}}$ (later domains), $\lambda_{\mathrm h}=1$}}\\
0 & 94.11$\pm$0.47 & $-$0.10$\pm$0.74 & 92.96$\pm$0.76 & 96.92$\pm$0.42 & 88.68$\pm$0.47 & $-$0.93$\pm$0.46 & 85.65$\pm$1.67 & 89.97$\pm$0.27 & 93.80$\pm$0.54 & $-$0.14$\pm$0.21 & 90.19$\pm$0.59 & 95.42$\pm$0.38\\
0.1 & 93.95$\pm$0.27 & 0.10$\pm$0.10 & 92.81$\pm$0.58 & 96.92$\pm$0.36 & 89.09$\pm$1.17 & $-$0.62$\pm$0.96 & 86.57$\pm$1.22 & 90.28$\pm$0.80 & 94.96$\pm$0.71 & 0.42$\pm$0.25 & 94.56$\pm$0.25 & 96.14$\pm$0.82\\
0.25 & 93.94$\pm$0.24 & 0.22$\pm$0.24 & 92.90$\pm$0.43 & 96.82$\pm$0.39 & 89.30$\pm$0.71 & 0.15$\pm$0.27 & 87.19$\pm$1.49 & 90.28$\pm$0.46 & 95.21$\pm$0.72 & 0.31$\pm$0.19 & 94.86$\pm$0.32 & 96.36$\pm$0.84\\
\textbf{0.5} & \textbf{94.02$\pm$0.07} & \textbf{0.32$\pm$0.33} & \textbf{93.07$\pm$0.10} & \textbf{96.82$\pm$0.46} & \textbf{89.61$\pm$0.47} & \textbf{0.00$\pm$0.46} & \textbf{87.50$\pm$1.22} & \textbf{90.74$\pm$0.00} & \textbf{95.16$\pm$1.23} & \textbf{0.58$\pm$0.58} & \textbf{95.08$\pm$0.42} & \textbf{96.31$\pm$1.13}\\
1 & 94.12$\pm$0.03 & 0.17$\pm$0.15 & 93.08$\pm$0.11 & 96.96$\pm$0.18 & 89.30$\pm$1.08 & $-$0.15$\pm$1.41 & 87.04$\pm$0.93 & 90.43$\pm$0.96 & 94.99$\pm$1.37 & 0.89$\pm$0.51 & 95.03$\pm$0.72 & 96.11$\pm$1.23\\
2 & 93.99$\pm$0.16 & 0.21$\pm$0.31 & 92.90$\pm$0.35 & 96.94$\pm$0.23 & 89.30$\pm$1.58 & $-$0.46$\pm$1.67 & 86.73$\pm$0.96 & 90.28$\pm$1.67 & 95.06$\pm$1.22 & 1.08$\pm$0.22 & 95.22$\pm$0.59 & 96.08$\pm$1.17\\
\midrule
\multicolumn{13}{l}{\textit{Historical-loss weight $\lambda_{\mathrm h}$, $\lambda_{\mathrm{sc}}=0.5$}}\\
0 & 93.94$\pm$0.60 & 0.43$\pm$0.30 & 92.97$\pm$0.85 & 96.76$\pm$0.28 & 88.58$\pm$1.11 & 0.31$\pm$0.71 & 86.57$\pm$1.22 & 89.35$\pm$0.93 & 94.25$\pm$0.70 & 1.33$\pm$0.22 & 94.86$\pm$0.67 & 95.33$\pm$0.66\\
0.25 & 93.99$\pm$0.16 & 0.21$\pm$0.31 & 92.90$\pm$0.35 & 96.94$\pm$0.23 & 89.30$\pm$1.58 & $-$0.46$\pm$1.67 & 86.73$\pm$0.96 & 90.28$\pm$1.67 & 95.06$\pm$1.22 & 1.08$\pm$0.22 & 95.22$\pm$0.59 & 96.08$\pm$1.17\\
0.5 & 94.12$\pm$0.03 & 0.17$\pm$0.15 & 93.08$\pm$0.11 & 96.96$\pm$0.18 & 89.30$\pm$1.08 & $-$0.15$\pm$1.41 & 87.04$\pm$0.93 & 90.43$\pm$0.96 & 94.99$\pm$1.37 & 0.89$\pm$0.51 & 95.03$\pm$0.72 & 96.11$\pm$1.23\\
\textbf{1} & \textbf{94.02$\pm$0.07} & \textbf{0.32$\pm$0.33} & \textbf{93.07$\pm$0.10} & \textbf{96.82$\pm$0.46} & \textbf{89.61$\pm$0.47} & \textbf{0.00$\pm$0.46} & \textbf{87.50$\pm$1.22} & \textbf{90.74$\pm$0.00} & \textbf{95.16$\pm$1.23} & \textbf{0.58$\pm$0.58} & \textbf{95.08$\pm$0.42} & \textbf{96.31$\pm$1.13}\\
2 & 93.94$\pm$0.24 & 0.22$\pm$0.24 & 92.90$\pm$0.43 & 96.82$\pm$0.39 & 89.30$\pm$0.71 & 0.15$\pm$0.27 & 87.19$\pm$1.49 & 90.28$\pm$0.46 & 95.21$\pm$0.72 & 0.31$\pm$0.19 & 94.86$\pm$0.32 & 96.36$\pm$0.84\\
4 & 93.94$\pm$0.28 & 0.10$\pm$0.15 & 92.79$\pm$0.53 & 96.92$\pm$0.40 & 89.20$\pm$1.35 & $-$0.15$\pm$1.17 & 86.73$\pm$1.17 & 90.28$\pm$1.22 & 95.06$\pm$0.83 & 0.36$\pm$0.38 & 94.64$\pm$0.35 & 96.28$\pm$0.92\\
\midrule
\multicolumn{13}{l}{\textit{SupCon weight $\lambda_{\mathrm{sc}}$ (all domains, including the first), $\lambda_{\mathrm h}=1$}}\\
0 & 91.85$\pm$0.57 & $-$0.38$\pm$0.07 & 88.29$\pm$0.81 & 97.40$\pm$0.38 & 84.47$\pm$2.59 & $-$4.63$\pm$2.58 & 80.25$\pm$4.20 & 85.34$\pm$1.93 & 92.27$\pm$0.56 & $-$0.06$\pm$0.17 & 88.08$\pm$0.79 & 94.06$\pm$0.41\\
0.1 & 94.05$\pm$0.24 & 0.18$\pm$0.20 & 93.00$\pm$0.59 & 96.92$\pm$0.41 & 89.30$\pm$1.28 & $-$0.46$\pm$1.39 & 86.88$\pm$0.96 & 90.59$\pm$0.96 & 95.04$\pm$0.80 & 0.36$\pm$0.32 & 94.58$\pm$0.29 & 96.19$\pm$0.91\\
0.25 & 94.00$\pm$0.21 & 0.18$\pm$0.25 & 92.94$\pm$0.38 & 96.85$\pm$0.41 & 89.51$\pm$1.11 & 0.00$\pm$0.00 & 87.35$\pm$1.87 & 90.43$\pm$0.96 & 95.21$\pm$0.72 & 0.25$\pm$0.17 & 94.83$\pm$0.29 & 96.39$\pm$0.82\\
\textbf{0.5} & \textbf{94.02$\pm$0.07} & \textbf{0.32$\pm$0.33} & \textbf{93.07$\pm$0.10} & \textbf{96.82$\pm$0.46} & \textbf{89.61$\pm$0.47} & \textbf{0.00$\pm$0.46} & \textbf{87.50$\pm$1.22} & \textbf{90.74$\pm$0.00} & \textbf{95.16$\pm$1.23} & \textbf{0.58$\pm$0.58} & \textbf{95.08$\pm$0.42} & \textbf{96.31$\pm$1.13}\\
1 & 94.13$\pm$0.06 & 0.28$\pm$0.09 & 93.14$\pm$0.13 & 96.96$\pm$0.22 & 89.30$\pm$1.08 & 0.00$\pm$1.67 & 87.19$\pm$0.96 & 90.43$\pm$0.96 & 95.09$\pm$1.28 & 0.86$\pm$0.29 & 95.11$\pm$0.77 & 96.19$\pm$1.06\\
2 & 93.94$\pm$0.14 & 0.21$\pm$0.19 & 92.82$\pm$0.28 & 96.96$\pm$0.22 & 88.99$\pm$1.28 & 0.15$\pm$1.41 & 86.73$\pm$0.96 & 89.97$\pm$1.17 & 95.11$\pm$1.22 & 1.06$\pm$0.29 & 95.25$\pm$0.52 & 96.14$\pm$1.13\\
\bottomrule
\end{tabular}
\end{table*}